\documentclass[11pt, letterpaper]{article}
\usepackage[margin=1in]{geometry}

\usepackage{microtype}
\usepackage[pdftex]{graphicx}
\usepackage{booktabs}
\usepackage{wrapfig}

\usepackage{amsmath}
\usepackage{amssymb}
\usepackage{mathtools}
\usepackage{amsthm}
\usepackage{enumitem}

\theoremstyle{plain}
\newtheorem{theorem}{Theorem}
\newtheorem{proposition}{Proposition}
\newtheorem{lemma}{Lemma}

\theoremstyle{definition}
\newtheorem{definition}{Definition}
\newtheorem{assumption}{Assumption}
\newtheorem{remark}{Remark}

\usepackage{algorithm} 
\usepackage{algorithmic}

\usepackage[dvipsnames,table]{xcolor}
\usepackage{booktabs} 
\usepackage{amsfonts}       
\usepackage{nicefrac}       
\usepackage{bbding}
\usepackage{longtable}
\usepackage{lipsum}
\usepackage{multirow}

\usepackage{tablefootnote}

\renewcommand{\CheckmarkBold}{{\color{ForestGreen}\Checkmark}}
\renewcommand{\XSolidBold}{{\color{BrickRed}\XSolid}}

\usepackage{subcaption}

\newcommand{\BX}{\mathbb{X}}
\newcommand{\KLs}{\mathsf{KL}}

\newcommand{\err}{\mathsf{err}}
\newcommand{\BE}{\mathbb{E}}
\newcommand{\BR}{\mathbb{R}}
\newcommand{\BN}{\mathbb{N}}
\newcommand{\base}{\mathrm{base}}
\newcommand{\opt}{\mathrm{opt}}
\newcommand{\Mean}{\mathrm{mean}}
\newcommand{\BI}{\mathbb{I}}
\newcommand{\Argmin}[1]{\underset{#1}{\mathrm{argmin}}}
\newcommand{\bl}{\pmb{\ell}}
\newcommand{\bbm}{\boldsymbol{m}}
\newcommand{\bcalG}{\boldsymbol{\calG}}
\newcommand{\FST}{FST}

\DeclareMathOperator*{\argmin}{\mathrm{argmin}}

\newcommand{\methodname}{\textup{\texttt{FairProjection}}}

\renewcommand{\tilde}{\widetilde}

\newcommand{\Reals}{\mathbb{R}}

\newcommand{\defined}{\triangleq}

\newcommand{\RNum}[1]{\uppercase\expandafter{\romannumeral #1\relax}}

\newcommand{\co}{\mathrm{conj}}

\newcommand{\simp}[1]{\mathbf{\Delta}_{#1}}

\newcommand{\empp}{\widehat{P}_X}

\newcommand{\calA}{\mathcal{A}}
\newcommand{\calB}{\mathcal{B}}
\newcommand{\calC}{\mathcal{C}}

\newcommand{\calG}{\mathcal{G}}
\newcommand{\calH}{\mathcal{H}}

\newcommand{\calR}{\mathcal{R}}
\newcommand{\calS}{\mathcal{S}}

\newcommand{\calX}{\mathcal{X}}
\newcommand{\calY}{\mathcal{Y}}

\newcommand{\bh}{\boldsymbol{h}}

\newcommand{\bG}{\boldsymbol{G}}

\newcommand{\bx}{\boldsymbol{x}}
\newcommand{\bz}{\boldsymbol{z}}
\newcommand{\by}{\boldsymbol{y}}

\newcommand{\bw}{\boldsymbol{w}}

\newcommand{\bu}{\boldsymbol{u}}
\newcommand{\bv}{\boldsymbol{v}}
\newcommand{\bs}{\boldsymbol{s}}
\newcommand{\ba}{\boldsymbol{a}}
\newcommand{\bp}{\boldsymbol{p}}
\newcommand{\bq}{\boldsymbol{q}}
\newcommand{\bb}{\boldsymbol{b}}

\newcommand{\bQ}{\boldsymbol{Q}}
\newcommand{\bU}{\boldsymbol{U}}

\newcommand{\bV}{\boldsymbol{V}}

\newcommand{\bA}{\boldsymbol{A}}
\newcommand{\bB}{\boldsymbol{B}}
\newcommand{\bW}{\boldsymbol{W}}

\newcommand{\bDelta}{\boldsymbol{\Delta}}
\newcommand{\be}{\boldsymbol{e}}
\newcommand{\bmu}{\pmb{\mu}}

\newcommand{\blambda}{\pmb{\lambda}}

\newcommand{\bI}{\boldsymbol{I}}
\newcommand{\bM}{\boldsymbol{M}}

\usepackage[utf8]{inputenc} 
\usepackage[T1]{fontenc}    
\usepackage{hyperref}       
\usepackage{url}            
\usepackage{booktabs}       
\usepackage{amsfonts}       
\usepackage{nicefrac}       
\usepackage{microtype}      

\usepackage[font=small,labelfont=bf]{caption}
\usepackage[numbers]{natbib}

\usepackage{hyperref}

\hypersetup{
    colorlinks,
    linkcolor={blue!90!black},
    citecolor={green!60!black},
    urlcolor={blue!80!black}
}

\title{Beyond Adult and COMPAS: Fairness in Multi-Class Prediction}

\author{
    Wael~Alghamdi,$^{1*}$~Hsiang~Hsu,$^{1}$\thanks{Equal contribution.  Correspondence to: Wael~Alghamdi (\texttt{alghamdi@g.harvard.edu}) and Flavio P.  Calmon (\texttt{flavio@seas.harvard.edu}). }~~Haewon~Jeong,$^{1*}$~Hao Wang,$^1$\\
    ~P.~Winston~Michalak,$^1$~Shahab~Asoodeh,$^2$ and~Flavio~P.~Calmon$^1$
}
\date{
$^1$John A. Paulson School of Engineering and Applied Sciences, Harvard University\\
$^2$Department of Computing and Software, McMaster University\\
}

\begin{document}

\maketitle

\begin{abstract}\noindent
We consider the problem of producing fair probabilistic classifiers for multi-class classification tasks. We formulate this problem in terms of ``projecting'' a pre-trained (and potentially unfair) classifier onto the set of models that satisfy target group-fairness requirements.  The new, projected model is given by post-processing the outputs of the pre-trained classifier by a multiplicative factor. We provide a parallelizable iterative algorithm for computing the projected classifier and derive both sample complexity and convergence guarantees. 
Comprehensive numerical comparisons with state-of-the-art benchmarks demonstrate that our approach maintains competitive performance in terms of accuracy-fairness trade-off curves, while achieving  favorable runtime on large datasets.  
We also  evaluate our method at scale on an open dataset with multiple classes, multiple intersectional protected groups, and over 1M samples.

\end{abstract}

\section{Introduction}
\label{sec:intro}

Machine learning (ML) algorithms are increasingly used to automate decisions that have significant social  consequences. This trend has led to a surge of research on designing and evaluating fairness interventions that prevent discrimination in ML models. When dealing with \emph{group fairness}, fairness interventions aim to ensure that a ML model does not discriminate against different groups determined, for example, by race, sex, and/or nationality. Extensive comparisons between  discrimination control methods can be found in~\cite{bellamy2018ai,friedler2019comparative,wei2021optimized}. As these studies demonstrate, there is still no ``best'' fairness intervention  for ML, with  the majority of existing approaches being tailored to either binary classification tasks, binary population groups, or both.   Moreover, discrimination control methods are often tested on overused datasets of modest size collected in either the US or Europe (e.g., UCI Adult~\citep{Lichman:2013} and COMPAS~\citep{angwin2016machine}).

Most fairness interventions\footnote{See Related Work and Table \ref{tab:benchmarks} for notable exceptions.} in ML focus on binary outcomes. In this case, the classification output is either positive or negative, and group-fairness metrics are tailored to  binary decisions~\cite{hardt2016equality}. 
While binary classification covers a  range of ML tasks of societal importance (e.g., whether to approve a loan, whether to admit a student), there are many cases where the predicted variable is not binary. For example, 
in education, grading algorithms  assign one out of several grades to students. In healthcare, predicted outcomes are frequently not binary (e.g., severity of disease). Even the original COMPAS algorithm---a timeworn case-study in fair ML---assigned  a score between 1 to 10 to each pre-trial defendant \citep{angwin2016machine}.

We introduce a theoretically-grounded discrimination control method that ensures group fairness in multi-class classification  for several, potentially overlapping population groups. We consider group fairness metrics that are natural multi-class extensions of their  binary classification counterparts, such as statistical parity~\citep{feldman2015certifying}, equalized odds~\citep{hardt2016equality},  and error rate imbalance~\citep{pleiss2017fairness,chouldechova2017fair}.  
When restricted to two predicted classes, our method performs competitively against  state-of-the-art fairness interventions tailored to binary classification tasks. Our fairness intervention is  model-agnostic (i.e., applicable to any model class) and scalable to datasets that are orders of magnitude larger than standard benchmarks found in the fair ML literature.

Our approach is based on an information-theoretic formulation called \emph{information projection}. We show that this formulation is particularly well-suited for ensuring fairness in probabilistic classifiers with multi-class outputs. Given a probability distribution $P$ and a convex set of distributions $\mathcal{P}$, the goal of information projection is to find the ``closest'' distribution to $P$ in $\mathcal{P}$. The study of information projection can be traced back to \cite{csiszar1975divergence}, which used  KL-divergence to measure ``distance'' between distributions. Since then, information projection has been extended to other divergence measures, such as $f$-divergences \cite{csiszar1995generalized} and R{\'e}nyi divergences \cite{kumar2016projection,kumar2015minimization}. 
Recently, \cite{ISIT} studied how to project a probabilistic classifier, viewed as a conditional distribution, onto the set of classifiers that satisfy target group-fairness requirements.  
Remarkably, the  projected classifier is obtained by multiplying (i.e., post-processing) the predictions of the original classifier by a factor that depends on the group-fairness constraints.

Prior work on information projection relies on a critical---and limiting---information-theoretic assumption:  the underlying probability distributions are \emph{known exactly}. This is infeasible in practical ML applications, where only a set of training samples from the underlying data distribution is available.  
We fill this gap by introducing an efficient procedure for computing the projected classifier with finite samples called \methodname. We establish theoretical guarantees for our algorithm in terms of convergence and sample complexity. Notably, our procedure is parallelizable (e.g., on a GPU). As a result, \methodname~scales to datasets with the number of samples comparable to the population of many US states ($>10^6$ samples). We provide a TensorFlow \cite{tensorflow2015} implementation of our algorithm at~\cite{FairProjection}. We apply \methodname~to post-process the outputs of probabilistic classifiers in order to ensure group fairness. 

We benchmark our post-processing approach against several state-of-the-art  fairness interventions selected based on the availability of reproducible code, and qualitatively compare it against many others. Our numerical results are among the most comprehensive comparison of post-processing fairness interventions to date. 
We present performance results on the HSLS (High School Longitudinal Study, used in~\cite{jeong2022aaai}), Adult ~\cite{Lichman:2013}, and COMPAS~\citep{angwin2016machine} datasets.

We  also evaluate \methodname~on a dataset derived from open and anonymized data from Brazil's national high school exam---the \emph{Exame Nacional do Ensino M\'edio} (ENEM)---with over 1M samples. We made use of this dataset due to the need for large-scale benchmarks for evaluating  fairness interventions in multi-class classification tasks. We also answer recent calls \cite{bao2021its,ding2021} for moving away from overused datasets such as Adult~\cite{Lichman:2013} and COMPAS~\citep{angwin2016machine}. We hope that the ENEM dataset encourages researchers in the field of fair ML to test their methods within broader contexts.\footnote{Since (to the best of our knowledge) the ENEM dataset has not been used in fair ML, we provide in Appendix \ref{appendix:datasheet} a datasheet for the ENEM dataset. The data can be found at \cite{enem2017} and code for pre-processing the data can be found at \href{https://github.com/HsiangHsu/Fair-Projection}{https://github.com/HsiangHsu/Fair-Projection}.}

In summary, our main contributions are: (i)~We introduce a  post-processing fairness intervention for multi-class (i.e., non-binary) classification problems that can account for multiple protected groups and is scalable to large datasets; (ii)~We derive finite-sample guarantees and convergence-rate results for our post-processing method. Importantly,  \methodname~makes information projection practical without requiring exact knowledge of probability distributions; (iii)~We demonstrate the favourable performance of our approach through comprehensive benchmarks against state-of-the-art fairness interventions; (iv)~We put forth a new large-scale  dataset (ENEM)  for benchmarking discrimination control methods in multi-class classification tasks; this dataset may encourage researchers in fair ML to evaluate their methods beyond Adult and COMPAS.

\begin{table}[!t]
\rowcolors{3}{MidnightBlue!10}{}
\centering
\resizebox{0.9\textwidth}{!}{\begin{tabular}{cccccccc}
\toprule
\multirow{2}{*}{\textbf{Method}} &  \multicolumn{6}{c}{\textbf{Feature}} \\
& Multiclass & Multigroup & Scores  &  Curve & Parallel  & Rate & Metric \\
\toprule
Reductions~\cite{agarwal2018reductions}        & \XSolidBold  & \CheckmarkBold  & \CheckmarkBold     & \CheckmarkBold  & \XSolidBold & \CheckmarkBold & SP, (M)EO   \\

Reject-option~\cite{kamiran2012reject}     & \XSolidBold  & \CheckmarkBold & \XSolidBold & \CheckmarkBold & \XSolidBold & \XSolidBold  & SP, (M)EO \\
EqOdds~\cite{hardt2016equality}             & \XSolidBold  & \CheckmarkBold  & \XSolidBold & \XSolidBold & \XSolidBold & \CheckmarkBold & EO  \\ 
LevEqOpp~\cite{chzhen2019leveraging}      & \XSolidBold  & \XSolidBold  & \XSolidBold    & \XSolidBold & \XSolidBold & \XSolidBold & FNR \\ 
CalEqOdds~\cite{pleiss2017fairness}      & \XSolidBold  & \XSolidBold  & \CheckmarkBold & \XSolidBold & \XSolidBold  & \CheckmarkBold   &  MEO  \\ 
FACT~\cite{kim2020fact}               &  \XSolidBold  &  \XSolidBold  & \XSolidBold   & \CheckmarkBold  &  \XSolidBold & \XSolidBold   &   SP, (M)EO \vspace{-0.5mm} \\ 
 Identifying\tablefootnote{\cite{jiang2020identifying} mention that their method can be applied to multi-class classification, but their reported benchmarks are only for binary classification tasks.}~\cite{jiang2020identifying}       &  $\text{\CheckmarkBold}^{\text{\XSolidBold}}$ & \CheckmarkBold & \CheckmarkBold & \CheckmarkBold & \XSolidBold  & \XSolidBold & SP, (M)EO \vspace{0.8mm} \\

\FST~\cite{wei2019optimized,wei2021optimized}     & \XSolidBold   & \CheckmarkBold & \CheckmarkBold &\CheckmarkBold & \XSolidBold  &   \CheckmarkBold  & SP, (M)EO  \\ 

Overlapping~\cite{yang2020fairness}  & \CheckmarkBold  & \CheckmarkBold  & \CheckmarkBold & \CheckmarkBold &  \XSolidBold & \XSolidBold &  SP, (M)EO \\ 

Adversarial~\cite{zhang2018mitigating}  & \CheckmarkBold & \CheckmarkBold & N/A\tablefootnote{\cite{zhang2018mitigating} is an in-processing method unlike other benchmarks in the table. It does not take a pre-trained classifier as an input.} & \CheckmarkBold & \CheckmarkBold & \XSolidBold & SP, (M)EO\\ 
\midrule
\methodname~(ours)  & \CheckmarkBold  & \CheckmarkBold  & \CheckmarkBold  & \CheckmarkBold & \CheckmarkBold  & \CheckmarkBold & SP, (M)EO \\ 
\bottomrule
\end{tabular}}
\vspace{2mm}
\caption{Comparison between benchmark methods. \textbf{Multiclass/multigroup}: implementation takes datasets with multiclass/multigroup labels; \textbf{Scores}: processes raw outputs of probabilistic classifiers; \textbf{Curve}: outputs fairness-accuracy tradeoff curves (instead of a single point); \textbf{Parallel}: parallel implementation (e.g., on GPU) is available; \textbf{Rate}: convergence rate or sample complexity guarantee is proved.
 \textbf{Metric}: applicable fairness metric, with SP$\leftrightarrow$Statistical Parity, EO$\leftrightarrow$Equalized Odds, MEO$\leftrightarrow$Mean EO. 
}
\label{tab:benchmarks}
\end{table}

\paragraph{Related work.} 
We summarize key differentiating factors from prior work in Table \ref{tab:benchmarks} and provide a more in-depth discussion in Appendix~\ref{appendix:literature}. The fairness interventions that are the most similar to ours are the FairScoreTransformer~\citep[\FST]{wei2019optimized,wei2021optimized} and the pre-processing method in~\cite{jiang2020identifying}.  
The \FST~and~\cite{jiang2020identifying}  
can be viewed as an instantiation of \methodname~restricted to binary classification and cross-entropy (for \FST) or KL-divergence (for \citep{jiang2020identifying}) as the $f$-divergence of choice.  Thus, our approach is a generalization of both methods  
to multiple $f$-divergences.  We also note that, unlike our method, \citep{jiang2020identifying} requires retraining a classifier multiple times.    

\cite{agarwal2018reductions}  
introduced a reductions approach for fair classification. When restricted to binary classification, the benchmarks in Section~\ref{sec:applications}  indicate that the reductions approach consistently achieves the most competitive fairness-accuracy trade-off compared to ours. 
The approach described here has two key differences from~\cite{agarwal2018reductions}: 
it is not restricted to  binary classification tasks and does not require refitting a classifier several times over the training dataset. These are also key differentiating points from~\cite{celis2019classification}, 
which presented a meta-algorithm for fair classification that accounts for multiple constraints and groups. 
The reductions approach was later significantly generalized in the  GroupFair method  by~\cite{yang2020fairness} 
to account for non-overlapping groups and multiple predicted classes. Unlike~\cite{yang2020fairness}, we do not require retraining classifiers. 

Several other recent efforts consider optimizing accuracy under group-fairness constraints. \cite{cotter2019optimization} 
proposed a ``proxy-Lagrangian'' formulation for incorporating non-differentiable rate constraints which may include group fairness constraints. We avoid non-differentiability issues by considering the probabilities (scores) at the output of the classifier instead of thresholded decisions. \cite{zafar2017fairness} 
introduced a fairness-constrained optimization  applicable to margin-based classifiers (our approach can be used on any probabilistic classifier). \cite{menon2018cost} and~\cite{corbett2017algorithmic} 
characterized fairness-accuracy trade-offs in binary classification tasks when the underlying distributions are known.

\paragraph{Notation.} Boldface Latin letters will always refer to vectors or matrices. The entries of a vector $\bz$ are denoted by $z_j$, and those of a matrix $\bG$ by $G_{i,j}$. The all-1 and all-0 vectors are denoted by $\mathbf{1}$ and $\mathbf{0}$. We set $[N]:=\{1,\cdots,N\}$ and $\BR_+\defined [0,\infty)$. The probability simplex over $[N]$ is denoted by $\simp{N}\defined \{\bp \in \BR_+^N ~ ; ~ \mathbf{1}^T\bp=1\}$, and $\simp{N}^+$ its (relative) interior. 
If $P$ is a Borel probability measure over $\BR^N$, $Z\sim P$ is a random variable, and $f:\BR^N \to \BR^K$ is Borel, then the expectation of $f(Z)$ is denoted by $ \BE[f(Z)] = \BE_P[f] = \BE_{Z\sim P}[f(Z)]$. We use the standard asymptotic notations $O,\Theta,$ and $\Omega$.

\section{Problem formulation and preliminaries} \label{sec:MP}

\paragraph{Classification tasks.} The essential objects in  classification are the input sample space $\calX$, the predicted classes $\calY$, and the classifiers. We fix two random variables $X$ and $Y$, taking values in sets $\calX $ and $\calY\defined [C]$. Here, $(X,Y)$ is a pair comprised of an input sample and corresponding  class label randomly drawn from $\calX\times\calY$ with distribution $P_{X,Y}$. A probabilistic classifier 
is a function $\bh:\calX \to \simp{C}$, where  $h_c(x)$ represents the probability of sample $x\in \calX$ falling in class $c\in \calY$. Thus, $\bh$ gives rise to a $\calY$-valued random variable $\widehat{Y}$ via the distribution $P_{\widehat{Y}|X=x}(c) \defined h_c(x)$. 

\paragraph{Group-fairness constraints.} 
Let $S$ be a group attribute (e.g., race and/or sex), taking values in $\calS \defined [A]$. We consider multi-class generalization of three commonly used group fairness criteria in Table~\ref{table:Fair}. 
As observed by existing works \citep[see, e.g.,][]{agarwal2018reductions,menon2018cost,celis2019classification,wei2019optimized,ISIT}, each of these fairness constraints\footnote{We remark that our framework can be applied to other fairness constraints, e.g., the ones in~\cite{wei2019optimized}.} 
can be written in the vector-inequality form $\BE_{P_X}[\bG\bh] \le \mathbf{0}$ for a closed-form matrix-valued function $\bG: \calX \to \BR^{K\times C}$. 
For instance, for statistical parity, the $\bG$ matrix evaluated at a fixed individual $x \in \calX$ has $K=2AC$ rows indexed by $(\delta,a,c')\in \{0,1\} \times [A]\times [C]$, where the $(\delta,a,c')$-th row is equal to $ \left( (-1)^\delta P_S(a)^{-1} \sum_{c\in [C]} P_{S|X=x,Y=c}(a) h_c^{\base}(x) - (\alpha + (-1)^\delta) \right)\be_{c'}$, 
with $\be_1,\cdots,\be_C$ denoting the standard basis for $\BR^C$. The expressions for the $\bG$ matrix corresponding to the other fairness metrics are given in Appendix~\ref{appendix:fair}. Note that $\bG$ depends on $P_{S|X,Y}$. If the group attribute $S$ is part of the input feature $X$, then $P_{S|X,Y}$ is simply replaced with an indicator function. Otherwise, we approximate this conditional distribution by training a probabilistic classifier.

\begin{table}[t]
\small
\centering
\newcommand*{\ScaleC}[1]{\scalebox{1}{#1}}
\renewcommand{\arraystretch}{1.5}
\resizebox{0.9\textwidth}{!}{
\begin{tabular}{lccc}
\toprule
\textbf{Fairness Criterion} & Statistical parity & Equalized odds & Overall accuracy equality   \\
\toprule
\textbf{Expression} &
\ScaleC{${\displaystyle \left|\frac{P_{\widehat{Y}|S=a}(c')}{P_{\widehat{Y}}(c')} -1 \right| \leq \alpha }$} & \ScaleC{${\displaystyle\left|\frac{P_{\widehat{Y}|Y=c, S=a}(c')}{P_{\widehat{Y}|Y=c}(c')}-1\right| \leq \alpha}$} & \ScaleC{${\displaystyle\left|\frac{P(\widehat{Y}=Y\mid S=a)}{P(\widehat{Y}=Y)}-1\right| \leq \alpha}$} 
\\ \bottomrule
\end{tabular}
}
\vspace{2mm}
\caption{Standard group fairness criteria; one fixes $\alpha>0$ and iterates over all $(a,c,c')\in [A] \times [C]^2$.\vspace{-.5cm}}
\label{table:Fair}
\end{table}

\paragraph{Goal.} Our goal is to design an efficient post-processing method that takes a pre-trained classifier $\bh^{\base}$ that may violate some target group-fairness criteria and finds a fair classifier that has the most similar outputs (i.e., closest utility performance) to that of $\bh^{\base}$. 

\paragraph{Fairness through information-projection.} We formulate fair post-processing problem as follows. For a fixed search space $\calH \subset \simp{C}^{\calX} \defined \left\{ \bh : \calX \to \simp{C} \right\}$, a loss function $\err:\simp{C}^{\calX} \times \simp{C}^{\calX} \to \BR$, and a base classifier $\bh^{\base} \in \simp{C}^{\calX}$, one seeks to solve:
\begin{equation} \label{eq:fairness math}
\begin{aligned}
    \underset{\bh\in \calH}{\text{minimize}} \ \err\left(\bh,\bh^{\base}\right) \quad
    \text{subject to} \ \BE_{P_X}\left[ \bG \bh \right] \le \mathbf{0}.
\end{aligned}
\end{equation}
The function $\err$ quantifies the ``closeness'' between the scores given by $\bh$ and $\bh^\base$ and we choose $f$-divergence to measure this:
\begin{align}
    \err\left(\bh,\bh^{\base}\right)  
    =D_f(\bh \| \bh^\base \mid P_X) 
    \defined \BE_{P_X}\left[ \sum_{c\in [C]} h^\base_c(X)f\left(\frac{h_c(X)}{h_c^\base(X)}\right) \right]-f(1), \label{eq:f-divergence}
\end{align}
where $f$ is a convex function over $(0,\infty)$. By varying different choices of $f$, we can obtain e.g., cross-entropy (CE, $f(t)=-\log t$) and KL-divergence ($f(t)=t\log t$). 
For a chosen $f$-divergence, the  optimization problem~\eqref{eq:fairness math} becomes a generalization of \emph{information projection}~\cite{csiszar1975divergence}. 

\paragraph{Preliminaries on information-projection.} In a recent work \cite{ISIT}, an optimal solution for the information projection formulation~\eqref{eq:fairness math} was theoretically characterized. We briefly describe this result next. 
Let\footnote{Here, $\calC(\calX,\simp{C})$ denotes the complete metric space of continuous functions from $\calX$ to $\simp{C}$, equipped with the sup-norm, i.e., $\|\bh\| \defined \sup_{x\in \calX} \|\bh(x)\|_1$.} $\mathcal{H} \defined  \left\{ \bh \in \calC(\calX, \simp{C}) ~ ; ~ \inf_{c,x} h_c(x) > 0 \right\}$ and we introduce the following definition and assumption.
\begin{definition} \label{def:Dfconj}
For $\bp \in \simp{C}$, let $D_f^\co(\, \cdot \, , \bp)$ denote the convex conjugate of $D_f(\, \cdot \, \| \, \bp)$:
\begin{equation} \label{eq:Dfconj definition}
    D_f^\co(\bv,\bp) \defined \sup_{\bq \in \simp{C}} \bv^T \bq - D_f(\bq \, \| \, \bp).
\end{equation}
\end{definition}

\begin{assumption} \label{assumption}
Assume that: \textbf{(i)}~$f\in \calC^2(\BR)$, $f(1)=0,$ $f'(0^+)=-\infty$, and $f''(t)>0$ for all $t>0$; \textbf{(ii)}~each $G_{k,c}$ is bounded, differentiable, and has bounded gradient; \textbf{(iii)}~$\bh^\base\in \calH$, and each $h^\base_c$ has bounded partial derivatives; and \textbf{(iv)}~there is an $\bh\in \calH$ such that $\BE_{P_X}[\bG\bh]<\mathbf{0}$. 
\end{assumption}

Now, the solution for \eqref{eq:fairness math} can be obtained by a simple ``tilting'' of the base classifier's output, as stated in the next theorem.
\begin{theorem}[\cite{ISIT}]
\label{thm:MP}
If $f,\bh^\base,$ and $\bG$ satisfy Assumption~\ref{assumption}, and $\calX = \BR^d$, then there is a unique solution $\bh^\opt$ for the optimization problem \eqref{eq:fairness math} for the $f$-divergence objective $\eqref{eq:f-divergence}$. 
Furthermore, $\bh^\opt$ is given by the tilt
\begin{equation} \label{eq:tilt}
    h_c^\opt(x)  = h_c^\base(x) \cdot \phi\left( v_c(x;\blambda^\star) + \gamma(x;\blambda^\star) \right), \qquad (x,c) \in \calX \times [C],
\end{equation}
where: \textup{\textbf{(i)}}~the function $\phi$ denotes the inverse of $f'$; \textup{\textbf{(ii)}}~the function $\bv:\calX \times \BR^K \to \BR^C$ is defined by $\bv(x;\blambda) \defined - \bG(x)^T \blambda$; \textup{\textbf{(iii)}}~the function $\gamma : \calX \times \BR^K \to \BR$ is characterized by the equation $\BE_{c\sim \bh^\base(x)}\left[ \phi\left( v_c(x;\blambda) + \gamma(x;\blambda) \right) \right] = 1$; and \textup{\textbf{(iv)}}~$\blambda^\star \in \BR^K$ is any solution to the convex problem
\begin{equation} \label{eq:MP dual D}
D^\star\triangleq    \min_{\blambda \in \BR_+^K} ~ \BE\left[ D_f^\co\left(\bv(X;\blambda), \bh^\base(X) \right) \right].
\end{equation}
\end{theorem}

If the underlying data distribution is known, Theorem~\ref{thm:MP} yields an expression for the projected classifier as a post-processing of the base classifier.   
However, in practice, we do not know the underlying distribution and have to approximate it from a finite number of i.i.d. samples. 
In Section~\ref{sec:problem}, we first describe how we approximate the solution given in Theorem~\ref{thm:MP} with finite samples. We then  propose a parallelizable algorithm to solve the approximation in  Section~\ref{sec:ADMM}.

\section{A finite-sample approximation of information projection} \label{sec:problem}

In practice, $P_X$ is unknown and only data points $\BX \defined \{X_i\}_{i\in [N]}\subset \calX$, drawn from $P_X$, are available. Thus, we propose the following fairness optimization problem. We search for a (multi-class) classifier $\bh : \BX \to \simp{C}$ that solves the following:
\begin{equation} \label{eq:MP alt direct}
\begin{aligned}
    \underset{\substack{\bh : \BX \to \simp{C} \\ \ba:\BX \to \BR^C, \boldsymbol{b} \in \BR^K}}{\textup{minimize}} \qquad & D_f\left(\bh \, \| \, \bh^{\base} \mid \empp\right) +\tau_1 \cdot \left( \BE_{X\sim \empp}\left[\|\ba(X)\|_2^2\right] + \|\bb\|_2^2 \right) \\
    \textup{subject to}  \quad \qquad & \BE_{\empp}\left[ \bG \cdot (\bh+\tau_2 \ba) \right] \le \tau_2 \boldsymbol{b},
\end{aligned}
\end{equation}
with $\empp$ being the empirical measure (e.g., obtained from a dataset), and $\tau_1,\tau_2 > 0$ prescribed constants. The terms $\ba$ and $\boldsymbol{b}$ are added to circumvent infeasibility issues and aid convergence of our numerical procedure. We show in the following theorem that there is a unique solution for~\eqref{eq:MP alt direct}, and that it is given by a tilt (i.e., multiplicative factor) of $\bh^\base$. The tilting parameter is the solution of a finite-dimensional strongly convex optimization problem.

\begin{theorem} \label{thm:hopt}
Suppose Assumption~\ref{assumption} holds, and set $\zeta \defined \tau_2^2/(2\tau_1)$. There exists a unique solution $\bh^{\opt,N}$ to~\eqref{eq:MP alt direct}, and it is given by the formula
\begin{equation} \label{eq:hopt formula}
    h_c^{\opt,N}(x) = h_c^{\base}(x) \cdot \phi\left( v_c(x;\blambda^\star_{\zeta,N}) + \gamma(x;\blambda^\star_{\zeta,N}) \right), \quad (x,c) \in \BX \times [C],
\end{equation}
with $\bv,\phi,\gamma$ as in Theorem~\ref{thm:MP}, 
and $\blambda_{\zeta,N}^\star \in \BR^K$ is the unique solution to the strongly convex problem
\begin{equation} \label{eq:MP dual D N}
D_{\zeta,N}^\star \triangleq    \min_{\blambda \in \BR_+^K} ~ \BE_{\empp}\left[ D_f^\co\left(\bv(X;\blambda), \bh^\base(X) \right) \right] + \frac{\zeta}{2} \left\| \bcalG_N^T \blambda \right\|_2^2
\end{equation}
where $\bcalG_N \defined  \left( \bG(X_1)/\sqrt{N},\cdots, \bG(X_N)/\sqrt{N},\bI_K \right)\in \BR^{K\times (NC+K)}$.
\end{theorem}
\begin{proof}
See Appendix~\ref{appendix:thm duality}.
\end{proof}

Theorem~\ref{thm:hopt} shows that: strong duality holds between the primal~\eqref{eq:MP alt direct} and (the negative of) the dual~\eqref{eq:MP dual D N}; there is a unique classifier $\bh^{\opt,N}$ minimizing our fairness formulation~\eqref{eq:MP alt direct}; there is a unique solution $\blambda_{\zeta,N}^\star$ to the dual~\eqref{eq:MP dual D}; and there is an explicit functional form of $\bh^{\opt,N}$ in terms of $\blambda_{\zeta,N}^\star$ in~\eqref{eq:hopt formula}. 

The key distinctions between our formulation and  Theorem~\ref{thm:MP} are that we use the empirical measure $\empp$ (e.g., produced using a dataset with i.i.d. samples), we have a \emph{strongly} convex dual problem in~\eqref{eq:MP dual D N} (in contrast to the convex program in~\eqref{eq:MP dual D}), and we prove strong duality in Theorem~\ref{thm:hopt} (whereas an analogous strong duality is absent from the results of~\cite{ISIT}). Hence, Theorem~\ref{thm:hopt} yields a \emph{practical} two-step procedure for solving the functional optimization in equation~\eqref{eq:MP alt direct}: (i)~compute the dual variables by solving the strongly convex optimization in~\eqref{eq:MP dual D N}; (ii)~tilt the base classifier by using the dual variables according to~\eqref{eq:hopt formula}. This process is applied on real-world datasets using \methodname~(see Algorithm~\ref{alg:MP}) in the next section.

The results of Theorem~\ref{thm:hopt} are tightly related to those in information projection shown in Theorem~\ref{thm:MP}, which corresponds to the case $\tau_1=\tau_2=\zeta=0$ and $P_X$ in place of $\empp$. We show in Theorem~\ref{thm:admm convergence kl} in the next section that the choice $\zeta\propto N^{-1/2}$ yields a sense in which our numerically obtained (from \methodname) tilting parameters $\blambda$ work well for the population problem~\eqref{eq:MP dual D}.

\begin{remark} \label{rem:assumption}
In practice, Assumption~\ref{assumption} is not a limiting factor for Theorem~\ref{thm:hopt} and  \methodname. This is because: we are considering here a finite-set domain so continuity is automatic; we can perturb $\bh^\base$ by negligible noise to push it away from the simplex boundary; and the uniform classifier is strictly feasible. Nevertheless, Assumption~\ref{assumption}  simplifies the derivation of our theoretical results.
\end{remark}

\section{Fair projection and theoretical guarantees}
\label{sec:ADMM}

We introduce a parallelizable  algorithm, \methodname, that solves~\eqref{eq:MP alt direct} using $N$ i.i.d. data points. We prove that its utility converges to $D^\star$ (see~\eqref{eq:MP dual D}) in the population limit and establish both sample-complexity and convergence rate guarantees. Applying \methodname~to the group-fairness intervention problem in \eqref{eq:fairness math} yields the optimal parameters in \eqref{eq:hopt formula} for post-processing (i.e., tilting) the output of a multi-class classifier in order to satisfy target fairness constraints.

\begin{algorithm}[!b]
   \caption{\textbf{:} \methodname~for  solving~\eqref{eq:MP dual D N}.}
   \label{alg:MP}
\begin{algorithmic}
\small
   \STATE {\bfseries Input:} divergence $f$, predictions $\{\bp_i \defined \bh^{\base}(X_i) \}_{i\in [N]}$, constraints $\{ \bG_i \defined \bG(X_i) \}_{i\in[N]}$, regularizer $\zeta$, ADMM penalty $\rho$, and initializers  
   $\blambda$ and $(\bw_i)_{i\in [N]}$.
   \STATE {\bfseries Output:} $h_c^{\opt,N}(x) \defined h_c^\base(x)\cdot \phi(\gamma(x;\blambda)+v_c(x;\blambda))$.
   \vspace{2mm}
   \STATE $\bQ \leftarrow \frac{\zeta}{2}\bI + \frac{\rho}{2N} \sum_{i\in [N]} \bG_i\bG_i^T$ 
   \vspace{1mm}
   \FOR{$t=1,2,\cdots,t'$}
   \STATE $\ba_i \leftarrow \bw_i + \rho \bG_i^T\blambda \hfill~ i\in [N]$ 
   \vspace{0.5mm}
    \STATE $\bv_i\leftarrow \Argmin{\bv\in \BR^C} ~ D_f^\co(\bv,\bp_i) + \frac{\rho+\zeta}{2} \|\bv\|_2^2 + \ba_i^T \bv \hfill~ i\in [N]$
   \vspace{0.5mm}
   \STATE $\bq \leftarrow \frac{1}{N} {\displaystyle \sum_{i\in [N]}} \bG_i \cdot \left( \bw_i + \bv_i \right) $
   \vspace{0.3mm}
   \STATE $\blambda \leftarrow \Argmin{\bl \in \BR_+^K} ~ \bl^T \bQ \bl + \bq^{T} \bl$
   \STATE $\bw_i \leftarrow \bw_i + \rho \cdot \left( \bv_i + \bG_i^T\blambda \right) \hfill i\in [N]$
    \ENDFOR
\end{algorithmic}
\end{algorithm}

The \methodname~algorithm uses ADMM~\cite{Boyd_ADMM} to solve the convex program in \eqref{eq:MP dual D N}. Recall that it suffices to optimize \eqref{eq:MP dual D N} for computing~\eqref{eq:MP alt direct} as proved in Theorem~\ref{thm:hopt}. Algorithm~\ref{alg:MP} presents the steps of \methodname, and its detailed derivation is given in Appendix~\ref{appendix:ADMM iterations}. A salient feature of \methodname~is its \emph{parallelizability}. Each step that is done for $i$ varying over~$[N]$ can be executed for each $i$ separately and in parallel. In particular, this applies to the most computationally intensive step, the $\bv_i$-update step. We discuss next how the $\bv_i$-update step is carried out.

\paragraph{Inner iterations.} One approach to carry out the inner iteration in Algorithm~\ref{alg:MP} that updates $\bv_i$ is to study the vanishing of the gradient of  $\bv \mapsto D_f^\co(\bv,\bp_i) + \xi \|\bv\|_2^2 + \ba_i^T \bv$ (where $\xi = (\rho+\zeta)/2$ and $\ba_i \in \BR^C$ is some vector). In the KL-divergence case, $D_{\KLs}^\co$ is given by a log-sum-exp function, so its gradient is given by a softmax function, and equating the gradient to zero becomes a fixed-point equation. We give an iterative routine to solve this fixed point equation in Appendix~\ref{sec:KL vupdate}, whose proof of convergence is aided by our proof that the softmax function is $\frac12$-Lipschitz in Appendix~\ref{appendix:softmax}. Beyond the KL-divergence case, setting the gradient to zero does not seem to be an analytically tractable problem. Nevertheless, we may reduce the vector minimization in the $\bv_i$-step to a tractable $1$-dimensional root-finding problem, as the following result aids in showing.

\begin{lemma} \label{lem:dfconj 1d}
For $\bp\in \simp{C}^+$, $\ba \in \BR^C$, and $\xi>0$, if $f$ satisfies Assumption~\ref{assumption}, we have that
\begin{align}
    \min_{\bv \in \BR^C} D_f^\co(\bv,\bp) + \xi \|\bv\|_2^2 + \ba^T\bv = - \sup_{\theta \in \BR} - \theta +\hspace{-2mm} \sum_{c\in [C]} \min_{q_c\ge 0} ~ p_cf\left( \frac{q_c}{p_c} \right) + \frac{(a_c+q_c)^2}{4\xi} + \theta q_c.
\end{align}
\end{lemma}
\begin{proof}
See Appendix~\ref{sec:lemma1d proof}.
\end{proof}

We note that the $\bv_i$-update steps for both KL and CE (provided in detail in Appendix~\ref{sec:CE vupdate}) give, as a byproduct, the implicitly defined function~$\gamma(x;\blambda)$ (see the statements of Theorems~\ref{thm:MP}--\ref{thm:hopt}).

\paragraph{Convergence guarantees.} Our proposed algorithm, \methodname, enjoys the following convergence guarantees. First, the output after the $t$-th iteration $\blambda_{\zeta,N}^{(t)}$ converges exponentially fast to $\blambda_{\zeta,N}^\star$ (see~\eqref{eq:MP dual D N}).

\begin{theorem} \label{thm:FairProjection}
If Assumption~\ref{assumption} holds, the \methodname~algorithm for KL-divergence converges in $t'=O(\log N)$ steps, and runs in time $O(N\log N)$, to the unique solution $\blambda_{\zeta,N}^\star$ of~\eqref{eq:MP dual D N}. Further, if $\blambda_{\zeta,N}^{(t)}$ and $\bh^{(t)}$ are the $t$-th iteration outputs of \methodname, then $\|\blambda_{\zeta,N}^{(t)}-\blambda_{\zeta,N}^\star\|_2 = O(e^{-t})$ and $\bh^{(t)}(x)=\bh^{\opt,N}(x)\cdot \left( 1 + O(e^{-t}) \right)$ uniformly in $x\in \BX$ as $t\to \infty$.
\end{theorem}
\begin{proof}
See Appendix~\ref{appendix:rate}.
\end{proof}

Second, the parameter $\blambda_{N^{-1/2},N}^{(\log N)}$ obtainable from \methodname~performs well for the \emph{population} problem for information projection~\eqref{eq:MP dual D}.

\begin{theorem} \label{thm:admm convergence kl}
Suppose Assumption~\ref{assumption} holds, let $\calX=\BR^d$, and consider the KL-divergence case. 
Then, choosing $\zeta = \Theta\left( N^{-1/2} \right)$ and $t=\Omega(\log N)$ we obtain for any $\delta \in (0,1)$  that (see~\eqref{eq:MP dual D})
\begin{equation}
    \Pr\left\{ \BE_X\left[ D_{\KLs}^\co\left( \bv\left( X; \blambda_{\zeta,N}^{(t)} \right), \bh^\base(X) \right) \right] >  D^\star + O\left( \frac{1}{\sqrt{N}} \right) \right\} \le \delta.
\end{equation}
\end{theorem}
\begin{proof}
See Appendix~\ref{appendix:population}.
\end{proof}
\begin{remark}
Though Theorems~\ref{thm:FairProjection}--\ref{thm:admm convergence kl} are shown for the KL-divergence, the proof directly extends to general $f$-divergences satisfying Assumption~\ref{assumption}. In fact,  Lipschitz continuity of the gradient of $D_{\KLs}^\co$ is the only specific property that we apply to derive the KL-divergence case. For a general $f$-divergence, Lipschitz continuity of $\nabla D_f^\co$ may be derived as follows. Combining Lemmas~\ref{lem:qconj formula}--\ref{lem:Dfconj gradient} reveals the formula $\nabla_{\bv} D_f^\co(\bv,\bp) = \left( p_c \cdot \phi\left( \gamma(\bv) + v_c \right) \right)_{c\in [C]}$, where $\phi=(f')^{-1}$ and $\gamma(\bv)$ is uniquely defined by $\BE_{c\sim \bp}\left[ \phi(\gamma(\bv)+v_c) \right]=1$, with $\bp \in \simp{C}^+$ fixed. Since $\phi'=1/(f''\circ\phi)$, we have that $\phi$ is locally Lipschitz. From the proof of Theorem~5 in~\cite{ISIT_extended}, we know that $\bv \mapsto \gamma(\bv)$ is locally Lipschitz. Thus, $\bv \mapsto \nabla_{\bv} D_f^\co(\bv,\bp)$ is locally Lipschitz. Further, $\blambda \mapsto \nabla_{\bv} D_f^\co(\bv(x;\blambda),\bp)$ is then also locally Lipschitz. Note that we may restrict $\blambda$ \emph{a priori} to be within some finite ball (see Lemma~\ref{lem:lambda}). Thus, if, e.g., $X$ is compactly-supported, we would obtain the desired Lipschitzness properties of the gradient of $D_f^\co$, and the proofs of Theorems~\ref{thm:FairProjection}--\ref{thm:admm convergence kl} carry through for $D_f$ in place of $D_{\KLs}$.
\end{remark}

\paragraph{Benefit of parallelization.} The parallelizability of \methodname~provides significant speedup. In Appendix~\ref{appendix:runtime}, we provide an ablation study comparing the speedup due to parallelization. For the ENEM dataset (discussed next section), parallelization yields a $15$-fold reduction in runtime. 
In addition to the parallel advantage of \methodname, its inherent mathematical approach is more advantageous than gradient-based solutions. When numerically solving the dual problem~\eqref{eq:MP dual D N} (or any close variant) via gradient methods, the gradient of $D_f^\co$ (the convex conjugate of an $f$-divergence) must be computed. However, this gradient is tractable in only a very limited number of relevant instances of $f$-divergences. \methodname~tackles this intractability through having its subroutines be informed by Lemma~\ref{lem:dfconj 1d} and the discussion preceding it.

\section{Numerical benchmarks} \label{sec:applications}

We present empirical results and show that \texttt{FairProjection} has competitive performance both in terms of runtime and fairness-accuracy trade-off curves compared to  benchmarks---most notably \citep{agarwal2018reductions}, which requires retraining. Extensive additional benchmarks and experiment details are reported in Appendix~\ref{appendix:exps}.

\paragraph{Setup.} 
We consider three base classifiers (\texttt{Base}): gradient boosting (GBM), logistic regression (LR), and random forest (RF), implemented by Scikit-learn~\cite{pedregosa2011scikit}.
For \texttt{FairProjection} (the constrained optimization in \eqref{eq:MP alt direct}), we use cross-entropy (\texttt{FairProjection}\texttt{-CE}) and KL-divergence (\texttt{FairProjection}\texttt{-KL}) as the loss function.\footnote{We focus on  \texttt{FairProjection}\texttt{-CE} and random forest here; results for \texttt{FairProjection}\texttt{-KL} and other models are in Appendix~\ref{appendix:exps}.} 
We consider two fairness constraints: mean equalized odds (MEO) and statistical parity (SP) (cf. Table~\ref{table:Fair}). 
Particularly, to measure MEO, we define:
\begin{equation}\label{eq:multi-meo}
    \textsf{MEO} = \max_{i \in \mathcal{Y}} \max_{s_1, s_2 \in \mathcal{S}} (| \textsf{TPR}_i(s_1) - \textsf{TPR}_i(s_2)| + | \textsf{FPR}_i(s_1) - \textsf{FPR}_i(s_2)|) /2
\end{equation}
where
$\textsf{TPR}_i(s) = P(\widehat{Y} = i | Y = i, S=s)$, and $\textsf{FPR}_i(s) = P(\widehat{Y} = i | Y \neq i, S=s)$. 
The definition of statistical parity is provided in Appendix~\ref{appendix:acc-fairness-multi}.
All values reported in this section are from the test set with 70/30 train-test split. 
When benchmarking against methods tailored to binary classification, we restrict our results to both binary $Y$ and $S$ since, unlike \texttt{FairProjection}, competing methods cannot necessarily handle multi-class predictions and multiple groups.

\paragraph{Datasets.} 
We evaluate \texttt{FairProjection} and all benchmarks on four datasets.
We use two datasets in the education domain: the high-school longitudinal  study (HSLS) dataset~\cite{ingels2011high,jeong2022aaai} and a novel dataset we introduce here called ENEM~\cite{enem2017} (details in Appendix \ref{appendix:datasets}). 
The ENEM dataset contains Brazilian college entrance exam scores along with student demographic information and socio-economic questionnaire answers (e.g., if they own a computer). After pre-processing, the dataset contains $\sim$1.4 million samples with 139 features. Race was used as the group attribute $S$, and  Humanities exam score is used as the label $Y$. The score can be quantized into an arbitrary number of classes. For binary experiments, we quantize $Y$ into two classes, and for multi-class, we quantize it to 5 classes. The race feature $S$ has 5 categories, but we binarize it into White and Asian ($S=1$) and others ($S=0$). 
We call the entire ENEM dataset ENEM-1.4M. We also created smaller versions of the dataset with 50k samples: ENEM-50k-2C (binary classes) and ENEM-50k-5C (5 classes).\footnote{A   datasheet (see \cite{gebru2021datasheets})  for ENEM is given in Appendix~\ref{appendix:datasheet}.}  
For completeness, we report results on 
UCI Adult~\citep{Lichman:2013} and COMPAS~\citep{angwin2016machine}.

\paragraph{Benchmarks.} We compare our method with five existing fair learning algorithms: \texttt{Reduction} \citep{agarwal2018reductions}, reject-option classifier \citep[\texttt{Rejection}]{kamiran2012reject}, equalized-odds~\citep[\texttt{EqOdds}]{hardt2016equality}, calibrated equalized-odds~\cite[\texttt{CalEqOdds}]{pleiss2017fairness}, and leveraging equal opportunity~\cite[\texttt{LevEqOpp}]{chzhen2019leveraging}.\footnote{https://github.com/lucaoneto/NIPS2019\_Fairness.} 
The choice of benchmarks is based on the availability of reproducible codes.
For the first four baselines, we use IBM AIF360 library~\citep{bellamy2018ai}.
For \texttt{Reduction} and \texttt{Rejection}, we vary the tolerance to achieve different operation points on the fairness-accuracy trade-off curves. 
As \texttt{EqOdds}, \texttt{CalEqOdds} and \texttt{LevEqOpp} only allow hard equality constraint on equalized odds, they each produce a single point on the plot (see Fig.~\ref{fig:binary-benchmarks}).
We include the group attribute as a feature in the training set following the same benchmark procedure described in \cite{agarwal2018reductions, wei2021optimized} for a consistent comparison. 
Additional comparisons to~\cite{kim2020fact} are given in Appendix~\ref{appendix:acc-fairness-binary}.

\begin{figure}[t!]
\centering
\includegraphics[width=\textwidth]{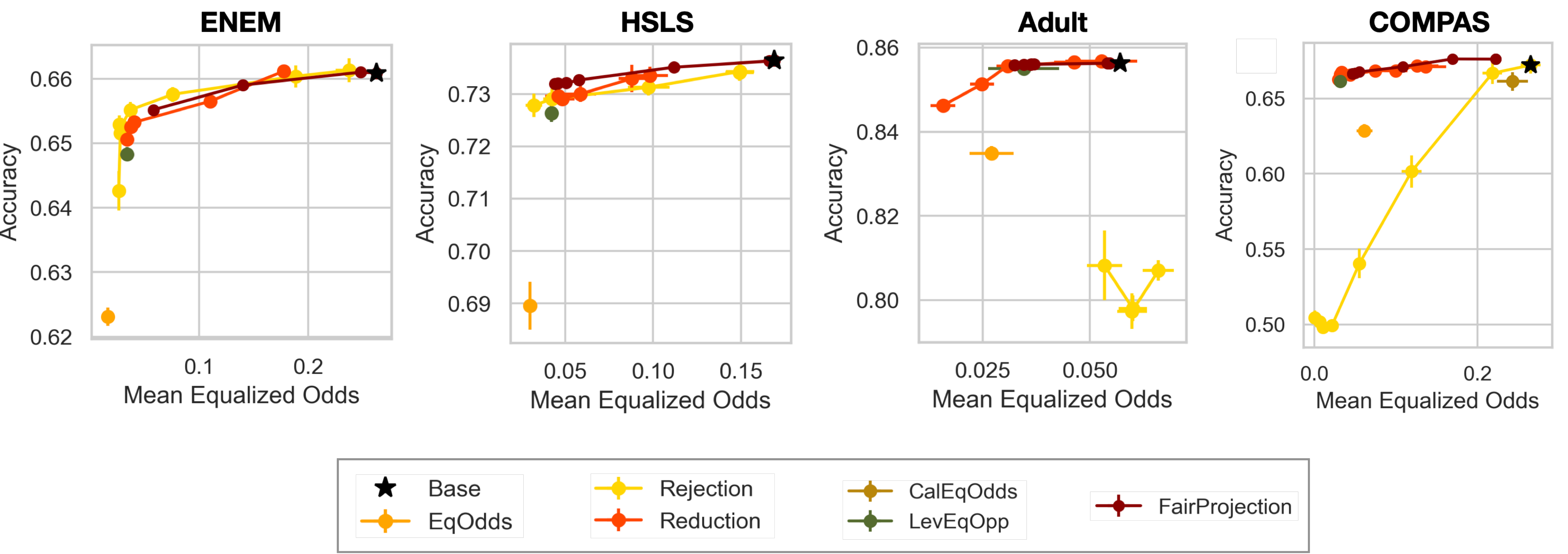}
\caption{Fairness-accuracy trade-off comparisons between \texttt{FairProjection} and five baselines on ENEM-50k-2C, HSLS, Adult and COMPAS datasets. For all methods, we used random forest as a base classifier.}
\label{fig:binary-benchmarks}
\end{figure}

\paragraph{Binary classification results.} \label{exp:af}
We compare \texttt{FairProjection}~with  benchmarks tailored to binary classification in terms of the MEO-accuracy trade-off on the ENEM-50k-2C, HSLS, Adult, and COMPAS datasets in Fig.~\ref{fig:binary-benchmarks}. 
Each point is obtained by averaging 10 runs with different train-test splits.  
\texttt{FairProjection}\texttt{-CE} curves were obtained by varying $\alpha$ values (cf.~Table~\ref{table:Fair}).
When $\alpha = 1.0$, the outputs of \texttt{FairProjection}\texttt{-CE} are equivalent to the base classifier RF.

We observe that \texttt{FairProjection}\texttt{-CE} and \texttt{Reduction} have the overall best and most consistent performances. On ENEM-50k-2C and HSLS datasets, although \texttt{EqOdds} achieves the best fairness, that fairness comes at the cost of $4\%$ accuracy drop (additively). 
The other four methods, on the other hand, produce comparatively good fairness with an accuracy loss of $<1\%$. 
In particular, \texttt{FairProjection}\texttt{-CE} has the smallest accuracy drop whilst improving MEO from 0.17 to 0.04 on HSLS.  
\texttt{CalEqOdds} requires strict calibration requirements and yields inconsistent performance when these requirements are not met.  
On ENEM-50k-2C and HSLS, \texttt{LevEqOpp} achieves comparable MEO with a slight accuracy drop, and on COMPAS, \texttt{LevEqOpp} performs equally well as \texttt{FairProjection}\texttt{-CE} and \texttt{Reduction}. 
Note that with high fairness constraints (i.e., small tolerance), the accuracy of \texttt{Rejection}  deteriorates.

\paragraph{Multi-Class results.}
We illustrate how \texttt{FairProjection} performs on multi-class prediction using  ENEM-50k-5C. 
In Figure~\ref{fig:multi}, we plot fairness-accuracy trade-off of \texttt{FairProjection}\texttt{-CE} with logistic regression  and adversarial debiasing~\cite[\texttt{Adversarial}]{zhang2018mitigating}. As their base classifiers are different (\texttt{Adversarial} is a GAN-based method), we plot accuracy difference compared to the base classifier instead of plotting the absolute value of accuracy.\footnote{Base accuracy for \texttt{FairProjection}$\ = 0.336$, \texttt{Adversarial}$\ =0.307$. Random guessing accuracy$\ =0.2$.}  \texttt{FairProjection} reduces MEO significantly with very small loss in accuracy. While \texttt{Adversarial} is also able to reduce MEO with negligible accuracy drop, it does not reduce the MEO as much as \texttt{FairProjection}. We show more extensive results with multi-group and multi-class ($|\calY| =5, = |\calS| = 5$) in Appendix~\ref{appendix:acc-fairness-multi}.

\begin{figure}
    \centering
    \includegraphics[width=0.35\columnwidth]{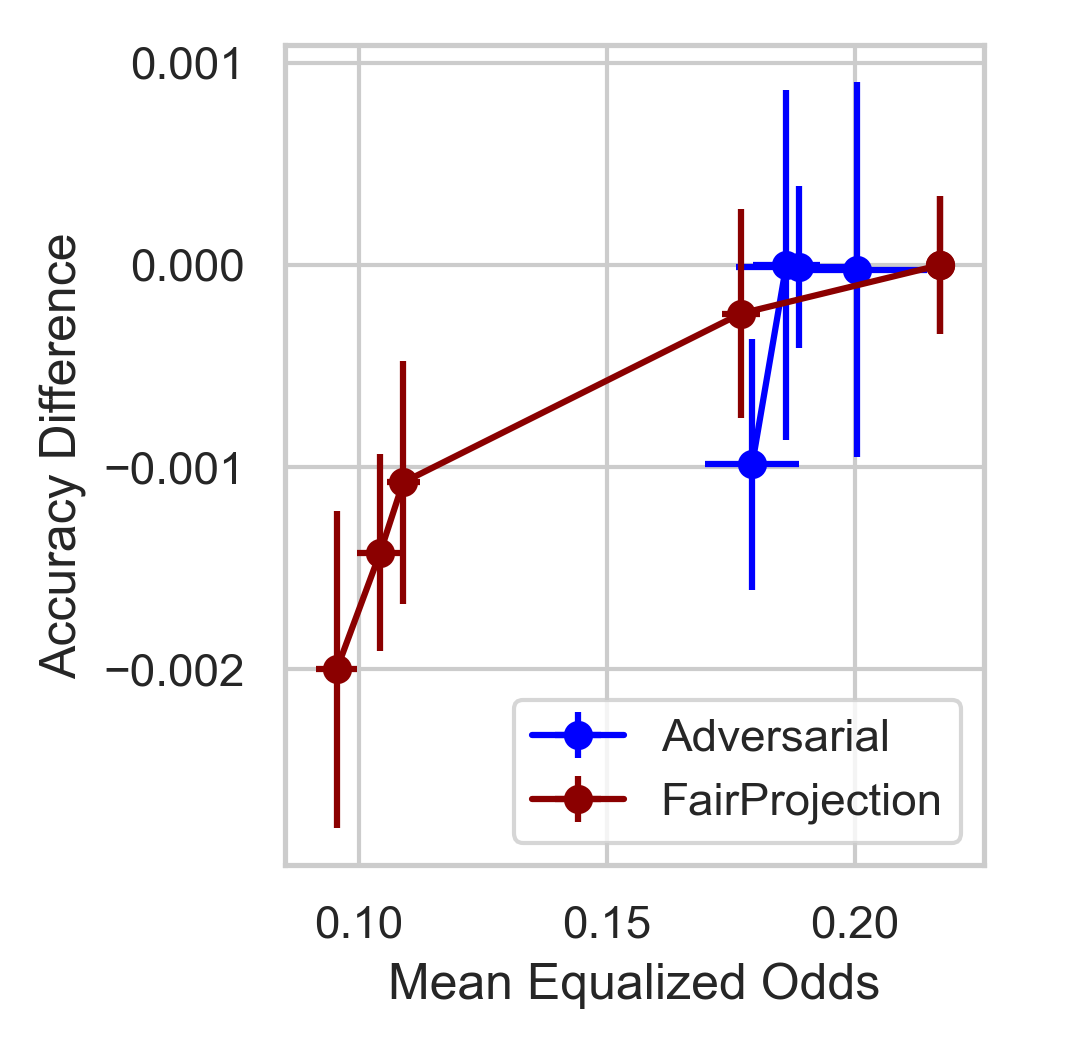}
    \caption{Fairness-accuracy trade-off for multi-class prediction on ENEM-50k-5C. FairProjection is
    \texttt{FairProjection}\texttt{-CE} with LR base classifier.}
    \label{fig:multi}
\end{figure}

\paragraph{Runtime comparisons.}\label{exp:runtime} 
In Table~\ref{table:runtime}, we record the runtime of \texttt{FairProjection-CE} and\texttt{-KL} with the five benchmarks on ENEM-1.4M-2C. 
These experiments were run on a machine with AMD Ryzen 2990WX 64-thread 32-Core CPU and NVIDIA TITAN Xp  12-GB GPU. 
For consistency, we used the same fairness metric (MEO, $\alpha = 0.01$), base classifier (GBM), and train/test split, and each number is the average of 2 repeated experiments.
\texttt{EqOdds}, \texttt{LevEqOpp}, and \texttt{CalEqOdds} are faster than \texttt{FairProjection} since they are optimized to produce one trade-off point (cf.~Fig.~\ref{fig:binary-benchmarks}).
Compared to baselines that  produce  full fairness-accuracy trade-off curves (i.e., \texttt{Reduction} and \texttt{Rejection}), \texttt{FairProjection} has the fastest runtime. 
Note that the 11.6 mins reported here for \texttt{FairProjection-KL} includes the time to fit the base classifiers. If base classifiers are pre-trained, the runtime of \texttt{FairProjection-KL} is \textbf{1.63 mins}. Also, the non-parallel implementation of \texttt{FairProjection-KL} takes 25.3 mins---parallelization attains 15$\times$ speedup (detailed results in Appendix~\ref{appendix:runtime}).

\begin{table}[t]
\footnotesize
\centering
\begin{tabular}{cccccccc}
\toprule
\multirow{2}{*}{\textbf{Method}} & 
\textit{\textbf{Reduction}} & \textbf{Rejection} & EqOdds & LevEqOpp & CalEqOdds & \multicolumn{2}{c}{\textit{\textbf{FairProjection (ours)}}}  \\
& \cite{agarwal2018reductions}  &  \cite{kamiran2012reject} & \cite{hardt2016equality} & \cite{chzhen2019leveraging} & \cite{pleiss2017fairness}  & CE & KL \\
\toprule
\textbf{Runtime} & 223.6  & 16.9  & 5.9 & 7.9 & 5.3 & 11.3 & 11.6 \\
\bottomrule
\end{tabular}
\vspace{2mm}
\caption{Execution time of \texttt{FairProjection} on the ENEM-1.4M-2C compared with five baseline methods (time shown in minutes). Methods in \textbf{bold} are capable of producing the full fairness-accuracy trade-off curves. Methods that are \textit{italicized} have a uniformly superior performance.}
\label{table:runtime}
\end{table}

\section{Final remarks and limitations}

We introduce a theoretically-grounded and versatile fairness intervention method, \methodname, and showcase its favorable performance in extensive experiments. We encourage the reader to peruse our theoretical result in Appendix~\ref{appendix:theory} and extensive additional numerical benchmarks in Appendix~\ref{appendix:exps}.  \methodname~is able to correct bias for multigroup/multiclass datasets, and it enjoys a fast runtime thanks to its parallelizability. We also evaluate our method on the ENEM dataset (see Appendix~\ref{appendix:datasheet} for a detailed description of the dataset). Our benchmarks are a step forward in moving away from the overused COMPAS and UCI Adult datasets.

We only consider group-fairness, and it would be interesting to try to incorporate other fairness notions (e.g., individual fairness~\cite{dwork2012fairness}) into our formulation. We assume that $\bh^\base$ is a pre-trained accurate (and potentially unfair) classifier; one future research direction is understanding how the accuracy of $\bh^\base$ influences the performance of the projected classifier.   
Finally, the performance of \methodname~is inherently constrained by data availability. Performance may degrade with intersectional increases of the number of groups, the number of labels, and the number of fairness constraints.

\clearpage
\bibliography{main}

\newcommand{\etalchar}[1]{$^{#1}$}
\begin{thebibliography}{AAW{\etalchar{+}}20b}

\bibitem[AAB{\etalchar{+}}15]{tensorflow2015}
Mart\'{i}n Abadi, Ashish Agarwal, Paul Barham, Eugene Brevdo, Zhifeng Chen,
  Craig Citro, Greg~S. Corrado, Andy Davis, Jeffrey Dean, Matthieu Devin,
  Sanjay Ghemawat, Ian Goodfellow, Andrew Harp, Geoffrey Irving, Michael Isard,
  Yangqing Jia, Rafal Jozefowicz, Lukasz Kaiser, Manjunath Kudlur, Josh
  Levenberg, Dandelion Man\'{e}, Rajat Monga, Sherry Moore, Derek Murray, Chris
  Olah, Mike Schuster, Jonathon Shlens, Benoit Steiner, Ilya Sutskever, Kunal
  Talwar, Paul Tucker, Vincent Vanhoucke, Vijay Vasudevan, Fernanda Vi\'{e}gas,
  Oriol Vinyals, Pete Warden, Martin Wattenberg, Martin Wicke, Yuan Yu, and
  Xiaoqiang Zheng.
\newblock {TensorFlow}: Large-scale machine learning on heterogeneous systems,
  2015.
\newblock Software available from tensorflow.org.

\bibitem[AAV19]{aghaei2019learning}
Sina Aghaei, Mohammad~Javad Azizi, and Phebe Vayanos.
\newblock Learning optimal and fair decision trees for non-discriminative
  decision-making.
\newblock In {\em Proceedings of the AAAI Conference on Artificial
  Intelligence}, volume~33, pages 1418--1426, 2019.

\bibitem[AAW{\etalchar{+}}20a]{ISIT}
Wael Alghamdi, Shahab Asoodeh, Hao Wang, Flavio~P. Calmon, Dennis Wei, and
  Karthikeyan~Natesan Ramamurthy.
\newblock Model projection: Theory and applications to fair machine learning.
\newblock In {\em 2020 IEEE International Symposium on Information Theory
  (ISIT)}, pages 2711--2716, 2020.

\bibitem[AAW{\etalchar{+}}20b]{ISIT_extended}
Wael Alghamdi, Shahab Asoodeh, Hao Wang, Flavio~P. Calmon, Dennis Wei, and
  Karthikeyan~Natesan Ramamurthy.
\newblock Model projection: Theory and applications to fair machine learning.
\newblock \url{https://github.com/WaelAlghamdi/ModelProjection}, 2020.

\bibitem[ABD{\etalchar{+}}18]{agarwal2018reductions}
Alekh Agarwal, Alina Beygelzimer, Miroslav Dud{\'\i}k, John Langford, and Hanna
  Wallach.
\newblock A reductions approach to fair classification.
\newblock In {\em International Conference on Machine Learning}, pages 60--69.
  PMLR, 2018.

\bibitem[AHJ{\etalchar{+}}22]{FairProjection}
Wael Alghamdi, Hsiang Hsu, Haewon Jeong, Hao Wang, P.~Winston Michalak, Shahab
  Asoodeh, and Flavio du~Pin Calmon.
\newblock {FairProjection}.
\newblock \url{https://github.com/HsiangHsu/Fair-Projection}, 2022.

\bibitem[ALMK16]{angwin2016machine}
Julia Angwin, Jeff Larson, Surya Mattu, and Lauren Kirchner.
\newblock Machine bias.
\newblock {\em ProPublica}, 2016.

\bibitem[BDH{\etalchar{+}}18]{bellamy2018ai}
Rachel~KE Bellamy, Kuntal Dey, Michael Hind, Samuel~C Hoffman, Stephanie Houde,
  Kalapriya Kannan, Pranay Lohia, Jacquelyn Martino, Sameep Mehta, Aleksandra
  Mojsilovic, et~al.
\newblock Ai fairness 360: An extensible toolkit for detecting, understanding,
  and mitigating unwanted algorithmic bias.
\newblock {\em arXiv preprint arXiv:1810.01943}, 2018.

\bibitem[BDNP19]{balcan2019envy}
Maria-Florina~F Balcan, Travis Dick, Ritesh Noothigattu, and Ariel~D Procaccia.
\newblock Envy-free classification.
\newblock {\em Advances in Neural Information Processing Systems}, 32, 2019.

\bibitem[BFG87]{borwein1987differentiability}
Jonathan~M Borwein, SP~Fitzpatrick, and JR~Giles.
\newblock The differentiability of real functions on normed linear space using
  generalized subgradients.
\newblock {\em Journal of mathematical analysis and applications},
  128(2):512--534, 1987.

\bibitem[BPC{\etalchar{+}}11]{Boyd_ADMM}
Stephen Boyd, Neal Parikh, Eric Chu, Borja Peleato, and Jonathan Eckstein.
\newblock Distributed optimization and statistical learning via the alternating
  direction method of multipliers.
\newblock {\em Found. Trends Mach. Learn.}, 3(1):1–122, jan 2011.

\bibitem[BZZ{\etalchar{+}}21]{bao2021its}
Michelle Bao, Angela Zhou, Samantha~A Zottola, Brian Brubach, Sarah Desmarais,
  Aaron~Seth Horowitz, Kristian Lum, and Suresh Venkatasubramanian.
\newblock It's {COMPAS}licated: The messy relationship between {RAI} datasets
  and algorithmic fairness benchmarks.
\newblock In {\em Thirty-fifth Conference on Neural Information Processing
  Systems Datasets and Benchmarks Track (Round 1)}, 2021.

\bibitem[CDH{\etalchar{+}}19]{chzhen2019leveraging}
Evgenii Chzhen, Christophe Denis, Mohamed Hebiri, Luca Oneto, and Massimiliano
  Pontil.
\newblock Leveraging labeled and unlabeled data for consistent fair binary
  classification.
\newblock {\em Advances in Neural Information Processing Systems}, 32, 2019.

\bibitem[CDPF{\etalchar{+}}17]{corbett2017algorithmic}
Sam Corbett-Davies, Emma Pierson, Avi Feller, Sharad Goel, and Aziz Huq.
\newblock Algorithmic decision making and the cost of fairness.
\newblock In {\em Proceedings of the 23rd ACM SIGKDD International Conference
  on Knowledge Discovery and Data Mining}, pages 797--806, 2017.

\bibitem[CHKV19]{celis2019classification}
L~Elisa Celis, Lingxiao Huang, Vijay Keswani, and Nisheeth~K Vishnoi.
\newblock Classification with fairness constraints: A meta-algorithm with
  provable guarantees.
\newblock In {\em Proceedings of the conference on fairness, accountability,
  and transparency}, pages 319--328, 2019.

\bibitem[Cho17]{chouldechova2017fair}
Alexandra Chouldechova.
\newblock Fair prediction with disparate impact: A study of bias in recidivism
  prediction instruments.
\newblock {\em Big data}, 5(2):153--163, 2017.

\bibitem[CHS20]{cho2020fair}
Jaewoong Cho, Gyeongjo Hwang, and Changho Suh.
\newblock A fair classifier using kernel density estimation.
\newblock {\em Advances in Neural Information Processing Systems},
  33:15088--15099, 2020.

\bibitem[CJG{\etalchar{+}}19]{cotter2019optimization}
Andrew Cotter, Heinrich Jiang, Maya~R Gupta, Serena Wang, Taman Narayan,
  Seungil You, and Karthik Sridharan.
\newblock Optimization with non-differentiable constraints with applications to
  fairness, recall, churn, and other goals.
\newblock {\em J. Mach. Learn. Res.}, 20(172):1--59, 2019.

\bibitem[CKV20]{celis2020data}
L~Elisa Celis, Vijay Keswani, and Nisheeth Vishnoi.
\newblock Data preprocessing to mitigate bias: A maximum entropy based
  approach.
\newblock In {\em International Conference on Machine Learning}, pages
  1349--1359. PMLR, 2020.

\bibitem[Com]{CC-License}
Creative Commons.
\newblock Creative commons attribution-noderivs 3.0 unported license.
\newblock \url{https://creativecommons.org/licenses/by-nd/3.0/deed.en}.
\newblock 05/25/2022.

\bibitem[Csi75]{csiszar1975divergence}
Imre Csisz{\'a}r.
\newblock I-divergence geometry of probability distributions and minimization
  problems.
\newblock {\em The annals of probability}, pages 146--158, 1975.

\bibitem[Csi95]{csiszar1995generalized}
Imre Csisz{\'a}r.
\newblock Generalized projections for non-negative functions.
\newblock In {\em Proceedings of 1995 IEEE International Symposium on
  Information Theory}, page~6. IEEE, 1995.

\bibitem[CST17]{Chen2017}
Liang Chen, Defeng Sun, and Kim-Chuan Toh.
\newblock A note on the convergence of admm for linearly constrained convex
  optimization problems.
\newblock {\em Comput. Optim. Appl.}, 66(2):327–343, mar 2017.

\bibitem[DEHH21]{denis2021fairness}
Christophe Denis, Romuald Elie, Mohamed Hebiri, and Fran{\c{c}}ois Hu.
\newblock Fairness guarantee in multi-class classification.
\newblock {\em arXiv preprint arXiv:2109.13642}, 2021.

\bibitem[DHMS21]{ding2021}
Frances Ding, Moritz Hardt, John Miller, and Ludwig Schmidt.
\newblock Retiring adult: New datasets for fair machine learning.
\newblock In M.~Ranzato, A.~Beygelzimer, Y.~Dauphin, P.S. Liang, and J.~Wortman
  Vaughan, editors, {\em Advances in Neural Information Processing Systems},
  volume~34, pages 6478--6490. Curran Associates, Inc., 2021.

\bibitem[DHP{\etalchar{+}}12]{dwork2012fairness}
Cynthia Dwork, Moritz Hardt, Toniann Pitassi, Omer Reingold, and Richard Zemel.
\newblock Fairness through awareness.
\newblock In {\em Proceedings of the 3rd innovations in theoretical computer
  science conference}, pages 214--226, 2012.

\bibitem[DOBD{\etalchar{+}}18]{donini2018empirical}
Michele Donini, Luca Oneto, Shai Ben-David, John Shawe-Taylor, and Massimiliano
  Pontil.
\newblock Empirical risk minimization under fairness constraints.
\newblock {\em arXiv preprint arXiv:1802.08626}, 2018.

\bibitem[DV75]{donsker1975asymptotic}
Monroe~D Donsker and SR~Srinivasa Varadhan.
\newblock Asymptotic evaluation of certain markov process expectations for
  large time, i.
\newblock {\em Communications on Pure and Applied Mathematics}, 28(1):1--47,
  1975.

\bibitem[DY16]{deng2016}
Wei Deng and Wotao Yin.
\newblock On the global and linear convergence of the generalized alternating
  direction method of multipliers.
\newblock {\em Journal of Scientific Computing}, 66:889--916, 2016.

\bibitem[ET99]{EkelandIvar1999Caav}
Ivar Ekeland and Roger T\'{e}mam.
\newblock {\em Convex analysis and variational problems}.
\newblock Classics in Applied Mathematics. Society for Industrial and Applied
  Mathematics, 1999.

\bibitem[FFM{\etalchar{+}}15]{feldman2015certifying}
Michael Feldman, Sorelle~A Friedler, John Moeller, Carlos Scheidegger, and
  Suresh Venkatasubramanian.
\newblock Certifying and removing disparate impact.
\newblock In {\em proceedings of the 21th ACM SIGKDD international conference
  on knowledge discovery and data mining}, pages 259--268, 2015.

\bibitem[FSV{\etalchar{+}}19]{friedler2019comparative}
Sorelle~A Friedler, Carlos Scheidegger, Suresh Venkatasubramanian, Sonam
  Choudhary, Evan~P Hamilton, and Derek Roth.
\newblock A comparative study of fairness-enhancing interventions in machine
  learning.
\newblock In {\em Proceedings of the conference on fairness, accountability,
  and transparency}, pages 329--338, 2019.

\bibitem[GMV{\etalchar{+}}21]{gebru2021datasheets}
Timnit Gebru, Jamie Morgenstern, Briana Vecchione, Jennifer~Wortman Vaughan,
  Hanna Wallach, Hal~Daum{\'e} Iii, and Kate Crawford.
\newblock Datasheets for datasets.
\newblock {\em Communications of the ACM}, 64(12):86--92, 2021.

\bibitem[GP17]{Gao2017}
Bolin Gao and Lacra Pavel.
\newblock On the properties of the softmax function with application in game
  theory and reinforcement learning.
\newblock {\em arXiv preprint arXiv:1704.00805}, 2017.

\bibitem[HPS16]{hardt2016equality}
Moritz Hardt, Eric Price, and Nati Srebro.
\newblock Equality of opportunity in supervised learning.
\newblock {\em Advances in neural information processing systems},
  29:3315--3323, 2016.

\bibitem[HR19]{hajek2019statistical}
Bruce Hajek and Maxim Raginsky.
\newblock Statistical learning theory.
\newblock {\em Lecture Notes}, 387, 2019.

\bibitem[{INE}20]{enem2017}
{INEP}.
\newblock Instituto nacional de estudos e pesquisas educaionais an\'isio
  teixeira, microdados do {ENEM}.
\newblock
  https://www.gov.br/inep/pt-br/acesso-a-informacao/dados-abertos/microdados/enem,
  2020.
\newblock Accessed: 2022-05-23.

\bibitem[IPH{\etalchar{+}}11]{ingels2011high}
Steven~J Ingels, Daniel~J Pratt, Deborah~R Herget, Laura~J Burns, Jill~A Dever,
  Randolph Ottem, James~E Rogers, Ying Jin, and Steve Leinwand.
\newblock High school longitudinal study of 2009 (hsls: 09): Base-year data
  file documentation. nces 2011-328.
\newblock {\em National Center for Education Statistics}, 2011.

\bibitem[JN20]{jiang2020identifying}
Heinrich Jiang and Ofir Nachum.
\newblock Identifying and correcting label bias in machine learning.
\newblock In {\em International Conference on Artificial Intelligence and
  Statistics}, pages 702--712. PMLR, 2020.

\bibitem[JSW22]{jang2022group}
Taeuk Jang, Pengyi Shi, and Xiaoqian Wang.
\newblock Group-aware threshold adaptation for fair classification.
\newblock {\em Proceedings of the AAAI Conference on Artificial Intelligence},
  2022.

\bibitem[JWC22]{jeong2022aaai}
Haewon Jeong, Hao Wang, and Flavio Calmon.
\newblock Fairness without imputation: A decision tree approach for fair
  prediction with missing values.
\newblock In {\em Proceedings of the AAAI Conference on Artificial
  Intelligence}, 2022.

\bibitem[KCT20]{kim2020fact}
Joon~Sik Kim, Jiahao Chen, and Ameet Talwalkar.
\newblock {FACT}: A diagnostic for group fairness trade-offs.
\newblock In {\em International Conference on Machine Learning}, pages
  5264--5274. PMLR, 2020.

\bibitem[KJW{\etalchar{+}}21]{krishnaswamy2021fair}
Anilesh Krishnaswamy, Zhihao Jiang, Kangning Wang, Yu~Cheng, and Kamesh
  Munagala.
\newblock Fair for all: Best-effort fairness guarantees for classification.
\newblock In {\em International Conference on Artificial Intelligence and
  Statistics}, pages 3259--3267. PMLR, 2021.

\bibitem[KKZ12]{kamiran2012reject}
F.~{Kamiran}, A.~{Karim}, and X.~{Zhang}.
\newblock Decision theory for discrimination-aware classification.
\newblock In {\em 2012 IEEE 12th International Conference on Data Mining},
  pages 924--929, Dec 2012.

\bibitem[KS15]{kumar2015minimization}
M~Ashok Kumar and Rajesh Sundaresan.
\newblock Minimization problems based on relative $\alpha $-entropy i: Forward
  projection.
\newblock {\em IEEE Transactions on Information Theory}, 61(9):5063--5080,
  2015.

\bibitem[KS16]{kumar2016projection}
M~Ashok Kumar and Igal Sason.
\newblock Projection theorems for the r{\'e}nyi divergence on $\alpha$-convex
  sets.
\newblock {\em IEEE Transactions on Information Theory}, 62(9):4924--4935,
  2016.

\bibitem[Lic13]{Lichman:2013}
M.~Lichman.
\newblock {UCI} machine learning repository, 2013.

\bibitem[MW18]{menon2018cost}
Aditya~Krishna Menon and Robert~C Williamson.
\newblock The cost of fairness in binary classification.
\newblock In {\em Conference on Fairness, Accountability and Transparency},
  pages 107--118. PMLR, 2018.

\bibitem[Nes04]{Nesterov_book}
Y.~Nesterov.
\newblock {\em Introductory Lectures on Convex Optimization}.
\newblock Springer, Boston, MA, 2004.

\bibitem[PRW{\etalchar{+}}17]{pleiss2017fairness}
Geoff Pleiss, Manish Raghavan, Felix Wu, Jon Kleinberg, and Kilian~Q
  Weinberger.
\newblock On fairness and calibration.
\newblock {\em arXiv preprint arXiv:1709.02012}, 2017.

\bibitem[PVG{\etalchar{+}}11]{pedregosa2011scikit}
Fabian Pedregosa, Ga{\"e}l Varoquaux, Alexandre Gramfort, Vincent Michel,
  Bertrand Thirion, Olivier Grisel, Mathieu Blondel, Peter Prettenhofer, Ron
  Weiss, Vincent Dubourg, et~al.
\newblock Scikit-learn: Machine learning in python.
\newblock {\em the Journal of machine Learning research}, 12:2825--2830, 2011.

\bibitem[Roc09]{RockafellarR.Tyrrell2009Va}
R.~Tyrrell Rockafellar.
\newblock {\em Variational analysis}.
\newblock Grundlehren der mathematischen Wissenschaften ; 317. Springer, Berlin
  ; Heidelberg, 1st ed. 1998. edition, 2009.

\bibitem[WRC20]{wei2019optimized}
Dennis Wei, Karthikeyan~Natesan Ramamurthy, and Flavio~P Calmon.
\newblock Optimized score transformation for fair classification.
\newblock In {\em 23rd International Conference on Artificial Intelligence and
  Statistics}, 2020.

\bibitem[WRC21]{wei2021optimized}
Dennis Wei, Karthikeyan~Natesan Ramamurthy, and Flavio~P Calmon.
\newblock Optimized score transformation for consistent fair classification.
\newblock {\em Journal of Machine Learning Research}, 22(258):1--78, 2021.

\bibitem[YCK20]{yang2020fairness}
Forest Yang, Mouhamadou Cisse, and Oluwasanmi~O Koyejo.
\newblock Fairness with overlapping groups; a probabilistic perspective.
\newblock {\em Advances in Neural Information Processing Systems}, 33, 2020.

\bibitem[YX20]{ye2020unbiased}
Qing Ye and Weijun Xie.
\newblock Unbiased subdata selection for fair classification: A unified
  framework and scalable algorithms.
\newblock {\em arXiv preprint arXiv:2012.12356}, 2020.

\bibitem[ZLM18]{zhang2018mitigating}
Brian~Hu Zhang, Blake Lemoine, and Margaret Mitchell.
\newblock Mitigating unwanted biases with adversarial learning.
\newblock In {\em Proceedings of the 2018 AAAI/ACM Conference on AI, Ethics,
  and Society}, pages 335--340, 2018.

\bibitem[ZVRG17]{zafar2017fairness}
Muhammad~Bilal Zafar, Isabel Valera, Manuel~Gomez Rogriguez, and Krishna~P
  Gummadi.
\newblock Fairness constraints: Mechanisms for fair classification.
\newblock In {\em Artificial Intelligence and Statistics}, pages 962--970,
  2017.

\end{thebibliography}
\bibliographystyle{alpha}

\clearpage
\appendix
\section*{Appendix}

This appendix is divided into three parts: Appendix~\ref{appendix:theory}: Proofs of theoretical results; Appendix~\ref{appendix:exps}: More details on the experimental setup, additional quantitative experiments, and more qualitative comparisons with related work; and Appendix~\ref{appendix:datasheet}: A datasheet for ENEM (2020) dataset.

\section{Proofs of theoretical results} \label{appendix:theory}

The theoretical details of our work are included in this appendix. We prove the strong duality stated in Theorem~\ref{thm:hopt} in Appendix~\ref{appendix:thm duality}. Algorithm~\ref{alg:MP} is derived in Appendix~\ref{appendix:ADMM iterations}. The inner iterations of Algorithm~\ref{alg:MP} are further developed in Appendices~\ref{appendix:vupdate}--\ref{appendix:softmax}. The convergence rate results in Theorems~\ref{thm:FairProjection}--\ref{thm:admm convergence kl} are proved in Appendices~\ref{appendix:rate}--\ref{appendix:population}. Explicit formulas for the $\bG$ matrix induced by the fairness metrics in Table~\ref{table:Fair} are given in Appendix~\ref{appendix:fair}.

\subsection{Proof of Theorem~\ref{thm:hopt}: strong duality} \label{appendix:thm duality}

We use the following minimax theorem, which is a generalization of Sion's minimax theorem.

\begin{theorem}[{\cite[Chapter VI, Prop. 2.2]{EkelandIvar1999Caav}}] \label{thm:minimax}
Let $V$ and $Z$ be two reflexive Banach spaces, and fix two convex, closed, and non-empty subsets $\calA \subset V$ and $\calB\subset Z$. Let $L:\calA\times \calB \to \BR$ be a function such that for each $u\in \calA$ the function $p\mapsto L(u,p)$ is concave and upper semicontinuous, and for each $p\in \calB$ the function $u\mapsto L(u,p)$ is convex and lower semicontinuous. Suppose that there exist points $u_0 \in \calA$ and $p_0 \in \calB$ such that $\lim_{p \in \calB, \|p\|\to \infty} L(u_0,p) = -\infty$ and $\lim_{u\in \calA, \|u\|\to \infty} L(u,p_0) = \infty$. Then, $L$ has at least one saddle-point $(\overline{u},\overline{p})$, and
\begin{equation} \label{eq:minimax}
    L(\overline{u},\overline{p}) = \min_{u\in \calA} \ \sup_{p\in \calB} \ L(u,p) = \max_{p\in \calB} \ \inf_{u\in \calA} \ L(u,p).
\end{equation}
In particular, in~\eqref{eq:minimax}, there exists a minimizer in $\calA$ of the outer minimization, and a maximizer in $\calB$ of the outer maximization.
\end{theorem}

Denote $\bh_i \defined \bh(X_i)$, $\bp_i \defined \bh^{\base}(X_i)$, $\ba_i \defined \ba(X_i)$, and $\bG_i\defined \bG(X_i)$, and let the matrix $\bcalG_N \defined  \left( \bG_1/\sqrt{N},\cdots, \bG_N/\sqrt{N},\bI_K \right)\in \BR^{K\times (NC+K)}$ be as in the theorem statement. We may rewrite the optimization~\eqref{eq:MP alt direct} as
\begin{equation} \label{eq:opt samples}
\begin{aligned}
    \underset{\substack{(\bh_i,\ba_i,\bb)\in \simp{C}\times \BR^C\times \BR^K, i\in [N] }}{\textup{minimize}} \qquad & \frac1N \sum_{i\in [N]} D_f\left( \bh_i \| \bp_i \right) +\tau_1 \cdot \left( \|\ba_i\|_2^2 + \|\bb\|_2^2 \right)  \\
    \textup{subject to}   ~ \quad \qquad \qquad & \frac{1}{N}\sum_{i\in [N]} \bG_i \bh_i + \tau_2 \cdot \left( \bG_i \ba_i - \bb \right) \le \mathbf{0}.
\end{aligned}
\end{equation}
We define $f$ at $0$ by the right limit $f(0)\defined f(0+)$. Assume for now that $f(0+)<\infty$, and we will explain at the end of this proof how to treat the case $f(0+)=\infty$. For the optimization problem~\eqref{eq:opt samples}, the Lagrangian $L:\simp{C}^N\times \BR^{NC}\times \BR^K \times \BR_{+}^K\to \BR$ is given by
\begin{equation}
\begin{aligned}
    L\left( \left( \bh_i \right)_{i\in [N]}, \left( \ba_i \right)_{i\in [N]}, \bb, \blambda \right) \defined \frac1N \sum_{i\in [N]} D_f\left( \bh_i \| \bp_i \right) &+\tau_1 \left( \|\ba_i\|_2^2 + \|\bb\|_2^2 \right) \\
    &+ \blambda^T \left( \bG_i \bh_i + \tau_2 \left( \bG_i \ba_i - \bb \right) \right).
\end{aligned}
\end{equation}
With $\bv(x;\blambda)\defined -\bG(x)^T\blambda$ as in the theorem statement, and denoting $\bv_i \defined \bv(X_i;\blambda)=-\bG_i^T\blambda$, we may rewrite the Lagrangian as
\begin{equation}
\begin{aligned}
    L\left( \left( \bh_i \right)_{i\in [N]}, \left( \ba_i \right)_{i\in [N]}, \bb, \blambda \right) = \frac1N  \sum_{i\in [N]}  D_f\left( \bh_i \| \bp_i \right) &- \bv_i^T\bh_i +\tau_1 \|\ba_i\|_2^2 - \tau_2 \bv_i^T \ba_i \\
    &+ \tau_1\|\bb\|_2^2 - \tau_2 \blambda^T\bb. 
\end{aligned}
\end{equation}
The optimization problem~\eqref{eq:opt samples} can be written as
\begin{equation}
    \inf_{(\bh_i,\ba_i,\bb)\in \simp{C}\times \BR^C\times \BR^K, i\in [N] } \ \sup_{\blambda \in \BR_+^K} \  L\left( \left( \bh_i \right)_{i\in [N]}, \left( \ba_i \right)_{i\in [N]}, \bb, \blambda \right).
\end{equation}
We check that the Lagrangian $L$ satisfies the conditions in Theorem~\ref{thm:minimax}. First, any Euclidean space $\BR^M$ (for $M\in \BN$) is a reflexive Banach space since it is finite-dimensional. In addition, the convex nonempty sets $\simp{C}^N \times \BR^{NC}\times \BR^K$ and $\BR_+^K$ are closed in their respective ambient Euclidean spaces. By continuity and convexity of $f$, and linearity of $L$ in $\blambda$, we have that $L$ satisfies all the convexity, concavity, and semicontinuity conditions in Theorem~\ref{thm:minimax}. Further, fixing any $\bh_i \in \simp{C}$, $i\in [N]$, and letting $\ba_i = \mathbf{0}$, $i\in [N]$, and $\bb = \frac{1}{\tau_2}\left( \mathbf{1} + \frac{1}{N} \sum_{i\in [N]} \bG_i \bh_i \right)$, we would get that
\begin{equation}
    L\left( \left( \bh_i \right)_{i\in [N]}, \left( \ba_i \right)_{i\in [N]}, \bb, \blambda \right) = -\blambda^T \mathbf{1} + \frac{1}{N} \sum_{i\in [N]} D_f\left( \bh_i \| \bp_i \right) + \tau_1 \|\bb\|_2^2 \to - \infty \quad \text{as} \quad \|\blambda\|_2 \to \infty.
\end{equation}
In addition, choosing $\blambda = \mathbf{0}$, we have the Lagrangian
\begin{equation}
\begin{aligned}
    L\left( \left( \bh_i \right)_{i\in [N]}, \left( \ba_i \right)_{i\in [N]}, \bb, \blambda \right) = \frac1N  \sum_{i\in [N]}  D_f\left( \bh_i \| \bp_i \right) & +\tau_1 \|\ba_i\|_2^2 + \tau_1\|\bb\|_2^2  \to \infty
\end{aligned}
\end{equation}
as $\|\bb\|_2 + \sum_{i\in [N]} \|\bh_i\|_2 + \|\ba_i\|_2 \to \infty$. Thus, we may apply the minimax result in Theorem~\ref{thm:minimax} to obtain the existence of a saddle-point of $L$ and that 
\begin{equation} \label{eq:minimax fair}
\begin{aligned}
    \min_{(\bh_i,\ba_i,\bb)\in \simp{C}\times \BR^C\times \BR^K, i\in [N]} \ &\sup_{\blambda \in \BR_+^K} \ L\left( \left( \bh_i \right)_{i\in [N]}, \left( \ba_i \right)_{i\in [N]}, \bb, \blambda \right) \\
    &  = \max_{\blambda \in \BR_+^K} \ \inf_{(\bh_i,\ba_i,\bb)\in \simp{C}\times \BR^C\times \BR^K, i\in [N]} \ L\left( \left( \bh_i \right)_{i\in [N]}, \left( \ba_i \right)_{i\in [N]}, \bb, \blambda \right).
\end{aligned}
\end{equation}
In particular, there exists a minimizer $(\bh_i^{\opt,N},\ba_i^{\opt,N},\bb^{\opt,N})\in \simp{C}\times \BR^C\times \BR^K, i\in [N]$, of the outer minimization in the left-hand side in~\eqref{eq:minimax fair}, and a maximizer $\blambda^\star \in \BR_+^K$ of the outer maximization in the right-hand side of~\eqref{eq:minimax fair}. By strict convexity of the objective function in~\eqref{eq:opt samples} (and convexity of the feasibility set), we obtain that the minimizer $(\bh_i^{\opt,N},\ba_i^{\opt,N},\bb^{\opt,N})\in \simp{C}\times \BR^C\times \BR^K, i\in [N]$, is unique. We show next that the optimizer $\blambda^\star$ is unique too, which we will denote by $\blambda_{\zeta,N}^\star$ as in the theorem statement. We also show that, for each fixed $\blambda\in \BR_+^K$, there is a unique minimizer $(\bh_i^{\blambda},\ba_i^{\blambda},\bb^{\blambda})\in \simp{C}\times \BR^C\times \BR^K, i\in [N]$, of the \emph{inner} minimization in the right-hand side of~\eqref{eq:minimax fair}; by strict convexity of $f$, this would imply that $\bh_i^{\opt,N} = \bh_i^{\blambda^\star_{\zeta,N}}$. 

Now, fix $\blambda \in \BR_+^K$, and consider the inner minimization in~\eqref{eq:minimax fair}. We have that
\begin{align}
    &\inf_{(\bh_i,\ba_i,\bb)\in \simp{C}\times \BR^C\times \BR^K, i\in [N]} \ L\left( \left( \bh_i \right)_{i\in [N]}, \left( \ba_i \right)_{i\in [N]}, \bb, \blambda \right) \nonumber \\
    &=\hspace{-4pt} \inf_{(\bh_i,\ba_i,\bb)\in \simp{C}\times \BR^C\times \BR^K, i\in [N]} \hspace{-2pt} \frac1N  \hspace{-3pt} \sum_{i\in [N]} \hspace{-3pt} D_f\left( \bh_i \| \bp_i \right) - \bv_i^T\bh_i +\tau_1 \|\ba_i\|_2^2 - \tau_2 \bv_i^T \ba_i + \tau_1\|\bb\|_2^2 - \tau_2 \blambda^T\bb \\
    &= \frac1N \sum_{i\in [N]} \inf_{\substack{\bh_i \in \simp{C} }}  D_f(\bh_i\|\bp_i) - \bv_i^T \bh_i + \inf_{\ba_i \in \BR^C} \tau_1 \|\ba_i\|_2^2 - \tau_2 \bv_i^T \ba_i + \inf_{\bb\in \BR^K} \tau_1\|\bb\|_2^2 - \tau_2 \blambda^T\bb \\
    &= \frac{1}{N} \sum_{i\in [N]} - D_f^\co(\bv_i, \bp_i)  - \frac{1}{2}\zeta \|\bv_i\|_2^2 - \frac{1}{2} \zeta \|\blambda\|_2^2 \\
    &= -\frac{\zeta}{2} \left\| \bcalG_N^T \blambda \right\|_2^2 - \frac{1}{N} \sum_{i\in [N]}  D_f^\co(\bv_i, \bp_i) \label{eq:lambda strictly convex}
\end{align}
where $\zeta \defined \tau_2^2/(2\tau_1)$. Here, the minimizers are $\ba_i^{\blambda} \defined \frac{\tau_2}{2\tau_1} \bv_i$ and $\bb_i^{\blambda} \defined \frac{\tau_2}{2\tau_1} \blambda$, and $\bh_i^{\blambda}$ is the unique probability vector in $\simp{C}$ for which $D_f^\co(\bv_i,\bp_i) = D_f(\bh_i^{\blambda}\|\bp_i)-\bv_i^T \bh_i^{\blambda}$; the existence and uniqueness of $\bh_i^{\blambda}$ is guaranteed since $\bq \mapsto D_f(\bq\| \bp_i)-\bv_i^T\bq$ is lower semicontinuous and strictly convex, and $\simp{C}$ is compact. Rewriting it in the form~\eqref{eq:lambda strictly convex}, the function
\begin{equation}
    \blambda \mapsto \inf_{(\bh_i,\ba_i,\bb)\in \simp{C}\times \BR^C\times \BR^K, i\in [N]} \ L\left( \left( \bh_i \right)_{i\in [N]}, \left( \ba_i \right)_{i\in [N]}, \bb, \blambda \right)
\end{equation}
can be seen to be strictly concave. Indeed, the function $\blambda \mapsto \left\| \bcalG_N^T \blambda \right\|_2^2$ is strictly convex. Also, each function $\blambda \mapsto D_f^\co(\bv_i,\bp_i)$ is convex as it is a pointwise supremum of linear functions: recalling that $\bv_i = -\bG_i^T\blambda$, we have the formula
\begin{equation} \label{eq:dfconj opt}
    D_f^\co(\bv_i,\bp_i) =  \sup_{\bq\in \simp{C}} -\bq^T \bG_i^T\blambda - D_f(\bq \, \| \, \bp_i).
\end{equation}
Hence, the outer maximizer $\blambda^\star$ in~\eqref{eq:minimax fair} is indeed unique, which we denote by $\blambda_{\zeta,N}^\star$. Note that $\blambda_{\zeta,N}^\star$ is the unique solution to the \emph{minimization}~\eqref{eq:MP dual D N}, i.e.,
\begin{equation} 
\blambda_{\zeta,N}^\star =   \Argmin{\blambda \in \BR_+^K} ~   \frac{1}{N} \sum_{i\in [N]}  D_f^\co(\bv_i, \bp_i) + \frac{\zeta}{2} \left\| \bcalG_N^T \blambda \right\|_2^2,
\end{equation}
as stated by the theorem.

Since $\bh^{\opt,N}=\bh^{\blambda_{\zeta,N}^\star}$, the following formula for $\bh^{\blambda}$ (for a general $\blambda \in \BR_+^K$) yields the desired functional form~\eqref{eq:hopt formula} for $\bh^{\opt,N}$ in terms of $\blambda_{\zeta,N}^\star$.

\begin{lemma}[{\cite[Lemma~4]{ISIT_extended}}] \label{lem:qconj formula}
Let $f:[0,\infty)\to \BR \cup\{\infty\}$ be a strictly convex function that is continuously differentiable over $(0,\infty)$ and satisfying $f(0)=f(0+)$, $f(1)=0$, and $f'(0+)=-\infty$. Let $\phi$ denote the inverse of $f'$. Fix $\bp\in \simp{C}^+$ and $\bv\in \BR^C$. Then, the unique minimizer of $\bq \mapsto D_f(\bq \| \bp) - \bv^T \bq$ over $\bq \in \simp{C}$ is given by $q^\star_c = p_c\cdot \phi(\gamma+v_c)$, $c\in [C]$, where $\gamma \in \BR$ is the unique number satisfying $\BE_{c\sim \bp}[\phi(\gamma+v_c)]=1$.
\end{lemma}

From Lemma~\ref{lem:qconj formula}, and using $\bv(x;\blambda_{\zeta,N}^\star) = - \bG(x)^T\blambda_{\zeta,N}^\star$ and $\phi=(f')^{-1}$, we get that there exists a uniquely defined function $\gamma:\BX \times \BR^K \to \BR$ for which
\begin{equation}
    \BE_{c\sim \bh^{\base}(x)}\left[ \phi\left( \gamma(x;\blambda_{\zeta,N}^\star)+v_c(x;\blambda_{\zeta,N}^\star) \right) \right] = 1
\end{equation}
for every $x\in \BX$. For this $\gamma$, we know from Lemma~\ref{lem:qconj formula} that 
\begin{equation}
    h_c^{\blambda_{\zeta,N}^\star}(x) = h_c^{\base}(x) \cdot \phi\left( \gamma(x;\blambda_{\zeta,N}^\star)+v_c(x;\blambda_{\zeta,N}^\star) \right)
\end{equation}
for every $c\in [C]$ and $x\in \BX$. Since $\bh^{\opt,N}=\bh^{\blambda_{\zeta,N}^\star}$, we obtain formula~\eqref{eq:hopt formula} for $\bh^{\opt,N}$ in terms of $\blambda_{\zeta,N}^\star$, and the proof of Theorem~\ref{thm:hopt} is complete in the case $f(0+)<\infty$.

Finally, we note how the case $f(0+)=\infty$ is treated, so assume $f(0)=f(0+)=\infty$. The only difference in this case is that the Lagrangian $L$ might attain the value $\infty$, whereas we need it to be $\BR$-valued to apply the minimax result in Theorem~\ref{thm:minimax}. Nevertheless, the only way $L$ can be infinite is if some classifier $\bh_i$ has an entry equal to $0$, in which case the objective function in~\eqref{eq:MP alt direct} (or~\eqref{eq:opt samples}) will also be infinite, so such a classifier can be thrown out without affecting the optimization problem. More precisely, we still have strict convexity and lower semicontinuity of the objective function in~\eqref{eq:opt samples}. Thus, there is a unique minimizer $\bh^{\opt,N}$ of~\eqref{eq:opt samples}. For this optimizer, there must be an $\varepsilon_1>0$ such that $\bh^{\opt,N}(x)\ge \varepsilon_1 \mathbf{1}$ for \emph{every} $x\in \BX$. Thus, the optimization problem~\eqref{eq:opt samples} remains unchanged if $\simp{C}$ is restricted to classifiers bounded away from $0$ by $\varepsilon_1$. Moreover, by the same reasoning, the optimization problem~\eqref{eq:dfconj opt} for finding $D_f^\co$ also remains unchanged if $\simp{C}$ is replaced by the set of classifiers bounded away from $0$ by some $\varepsilon_2>0$ that is \emph{independent} of the $X_i$. Hence, choosing $\varepsilon = \min(\varepsilon_1,\varepsilon_2)>0$, and replacing $\simp{C}$ by $\tilde{\bDelta}_C\defined \{\bq \in \simp{C} \, ; \, \bq \ge \varepsilon \mathbf{1}\}$ in the above proof, we attain the same results for the case $f(0+)=\infty$. 

\begin{remark}
In addition to our fairness problem formulation~\eqref{eq:MP alt direct} being different from that in \cite{ISIT_extended}, we note that our proof techniques are distinct. Indeed, the proofs in~\cite{ISIT_extended} develop several techniques since they are based only on Sion's minimax theorem, precisely because a generalized minimax result such as Theorem~\ref{thm:minimax} is inapplicable in the setup of~\cite{ISIT_extended}. The reason behind this inapplicability is that the ambient Banach space $\calC(\calX,\BR^C)$ is \emph{not} reflexive when $\calX$ is infinite, e.g., when $\calX=\BR^d$ as is assumed in~\cite{ISIT_extended}, whereas it is reflexive in our case as we consider a finite set of samples $\BX\subset \calX$.
\end{remark}

\subsection{Algorithm~\ref{alg:MP}: derivation of the ADMM iterations} \label{appendix:ADMM iterations}

ADMM is applicable to problems taking the form
\begin{equation} \label{eq:ADMM general}
    \begin{aligned}
    \underset{(\bV,\blambda)\in \BR^{V}\times \BR^K}{\text{minimize}} \ \quad & F(\bV) + \psi(\blambda) \\
    \text{subject to} ~~~ \quad & \bA \bV + \bB \blambda = \bbm,
    \end{aligned}
\end{equation}
where $F:\BR^V \to \BR\cup \{\infty\}$ and $\psi:\BR^K\to \BR\cup\{\infty\}$ are closed proper convex functions, and $\bA\in \BR^{U\times V}$, $\bB\in \BR^{U\times K}$, and $\bbm\in \BR^{U}$ are fixed.

We rewrite the convex problem~\eqref{eq:MP dual D N} into the ADMM form~\eqref{eq:ADMM general} as follows. With the samples $X_1,\cdots,X_N \overset{\text{i.i.d.}}{\sim} P_X$ fixed, we denote the following fixed vectors and matrices: for each $i\in [N]$, set
\begin{align}
    \bp_i &\defined \bh^\base(X_i) \in \simp{C}^+ = \{ \bq \in \simp{C} ~ ; ~ \bq > \mathbf{0} \}, \label{eq:p} \\
    \bG_i &\defined \bG(X_i) \in \BR^{K\times C}.
\end{align}
We introduce a variable $\bV \defined (\bv_i)_{i\in [N]} \in \BR^{NC}$ (with components $\bv_i\in \BR^C$), and consider the objective functions
\begin{align}
    F(\bV) &\defined \frac{1}{N} \sum_{i\in [N]}  D_f^\co\left( \bv_i , \bp_i \right) + \frac{\zeta}{2}\left\| \bV \right\|_2^2 , \label{eq:F} \\
    \psi(\blambda) &\defined \BI_{\BR_+^K}(\blambda) + \frac{\zeta}{2} \left\| \blambda \right\|_2^2.
\end{align}
Then, setting\footnote{The prefactor $1/\sqrt{N}$ is unnecessary since $\bbm = \mathbf{0}$, but we introduce it to simplify the ensuing expressions.}
\begin{equation} \label{eq:ADMM constraints}
    \bA = \frac{1}{\sqrt{N}}\bI_{NC}, \enskip \bB = \frac{1}{\sqrt{N}}(\bG_i)_{i\in [N]}^T, \text{ and } \enskip \bbm = \mathbf{0}_{NC},
\end{equation}
our finite-sample problem~\eqref{eq:MP dual D N} takes the ADMM form~\eqref{eq:ADMM general}. 

In addition, this reparametrization allows us to parallelize the ADMM iterations, which we briefly review next. One starts with forming the augmented Lagrangian for problem~\eqref{eq:ADMM general}, $L_\rho:\BR^V \times \BR^K \times \BR^U \to \BR\cup \{\infty\}$, where $\rho>0$ is a fixed \emph{penalty parameter} and $\bU\in \BR^U$ denotes a \emph{dual variable}, by
\begin{equation}
    \begin{aligned}
    L_\rho(\bV,\blambda,\bU) \defined F(\bV) &+ \psi(\blambda) + \bU^T \left( \bA \bV + \bB \blambda - \bbm \right)  + \frac{\rho}{2}\left\| \bA \bV + \bB \blambda - \bbm \right\|_2^2.
    \end{aligned}
\end{equation}
The ADMM iterations then repeatedly update the triplet  after the $t$-th iteration $(\bV^{(t)},\blambda^{(t)},\bU^{(t)})$ into a triplet $(\bV^{(t+1)},\blambda^{(t+1)},\bU^{(t+1)})$ that is given by
\begin{align}
    \bV^{(t+1)} &\in \Argmin{\bV \in \BR^V} ~ L_\rho(\bV,\blambda^{(t)},\bU^{(t)}), \label{eq:admm V} \\
    \blambda^{(t+1)} &\in \Argmin{\blambda \in \BR^V} ~ L_\rho(\bV^{(t+1)},\blambda,\bU^{(t)}), \label{eq:admm lambda} \\
    \bU^{(t+1)} &= \bU^{(t)} + \rho \cdot \left( \bA\bV^{(t+1)} + \bB \blambda^{(t+1)} \right). \label{eq:admm U}
\end{align}
We next instantiate the ADMM iterations to our problem, and we note that we will consider the scaled dual variable $\bW=\sqrt{N}\bU$.

 In our case, the augmented Lagrangian splits into non-interacting components along the $\bv_i$. This splitting allows parallelizability of the $\bV$-update step, which is the most computationally intensive step. Consider a conforming decomposition $\bU= (\bu_i)_{i\in [N]}$ for $\bu_i\in \BR^C$, and let $\bW=\sqrt{N}\bU$. With some algebra, one can show that the ADMM iterations for the ADMM problem specified by~\eqref{eq:F}--\eqref{eq:ADMM constraints} are expressible by\footnote{Note also that in these specific ADMM iterations, unlike in the general ADMM iterations, we write ``$=\mathrm{argmin}$'' as opposed to ``$\in\mathrm{argmin}$'' since strict convexity and coercivity guarantee that a unique minimizer exists (see~\cite{Chen2017} for a case where $\argmin{}$ is empty). Also, we write here $\bq^{(t)T}$ instead of $\left( \bq^{(t)} \right)^T$ for readability.}
\begin{align}
    \bv_i^{(t+1)} &=  \Argmin{\bv \in \BR^C}   ~ D_f^\co\left( \bv,\bp_i \right) + \calR_i^{(t)}(\bv), \hspace{2cm} i\in [N],\label{eq:V Update}\\
    \blambda^{(t+1)} &= \Argmin{\blambda \in \BR_+^K} ~ \blambda^T \bQ \blambda + \bq^{(t)T} \blambda,  \\
    \bw_i^{(t+1)} &= \bw_i^{(t)} + \rho \cdot \left( \bv_i^{(t+1)} + \bG_i^T\blambda^{(t+1)} \right), \hspace{1.9cm} i\in [N], \label{eq:W Update}
\end{align}
where $\calR_i^{(t)}:\BR^C\to \BR$ is the quadratic form
\begin{equation} \label{eq:R}
    \calR_i^{(t)}(\bv) \defined \frac{\rho+\zeta}{2} \|\bv\|_2^2 + \left( \bw_i^{(t)} + \rho \bG_i^T\blambda^{(t)} \right)^T \bv,
\end{equation}
and the fixed matrix $\bQ \in \BR^{K\times K}$ and vectors $\bq^{(t)}\in \BR^K$ are given by
\begin{align}
    \bQ &\defined \frac{\zeta}{2}\bI_K + \frac{\rho}{2N} \sum_{i\in [N]} \bG_i\bG_i^T, \label{eq:Q} \\
    \bq^{(t)} &\defined \frac{1}{N} \sum_{i\in [N]} \bG_i \cdot \left( \bw_i^{(t)} + \bv_i^{(t+1)} \right).
\end{align}
Note that both the first~\eqref{eq:V Update} and last~\eqref{eq:W Update} steps can be carried out for each sample $i\in [N]$ in parallel.

\subsection{The inner iterations: minimizing the convex conjugate of \texorpdfstring{$f$}{f}-divergence} \label{appendix:vupdate}

Only updating the primal-variable $\bv_i$ in Algorithm~\ref{alg:MP}, i.e., solving
\begin{equation} \label{eq:Dfconj min}
    \min_{\bv \in \BR^C} \ D_f^\co(\bv,\bp) + \xi \|\bv\|_2^2 + \ba^T \bv
\end{equation}
for fixed $(\bp,\xi,\ba) \in \simp{C}^+ \times (0,\infty) \times \BR^C$, is a nonstandard task. We propose in this section two approaches to execute this step, which aim at re-expressing the required minimization as either a fixed-point or a root-finding problem. In more detail, if one has access to an explicit formula for the gradient of $D_f^\co$, then one can transform~\eqref{eq:Dfconj min} into a fixed-point equation. This case applies for the KL-divergence, for which $\nabla D_{\KLs}^\co$ is the softmax function (Appendix~\ref{sec:KL vupdate}). Furthermore, for the convergence of the fixed-point iterations, we derive an improved Lipschitz constant for the softmax function in Appendix~\ref{appendix:softmax}. On the other hand, if one does not have a tractable formula for $\nabla D_f^\co$, we propose the reduction provided in Lemma~\ref{lem:dfconj 1d}, whose proof is provided in Appendix~\ref{sec:lemma1d proof}. We specialize the reduction provided by Lemma~\ref{lem:dfconj 1d} to the cross-entropy case in Appendix~\ref{sec:CE vupdate}. Finally, we include in Appendix~\ref{appendix:Dfconj gradient} a general formula for $\nabla D_f^\co$ that can be used for the $\bv_i$-update step for a general $f$-divergence, and we also utilize it in Appendices~\ref{appendix:rate}--\ref{appendix:population} to prove the convergence rate of Algorithm~\ref{alg:MP} stated in Theorems~\ref{thm:FairProjection}--\ref{thm:admm convergence kl}.

\subsubsection{Primal update for KL-divergence} \label{sec:KL vupdate}

Consider the case when the $f$-divergence of choice is the KL-divergence, i.e., $f(t)=t\log t$. Then, the convex conjugate $D_f^\co$ is given by the log-sum-exp function \cite{donsker1975asymptotic}, namely, for $(\bp,\bv)\in \simp{C}^+\times \BR^C$ we have
\begin{equation}
    D_f^\co(\bv,\bp) = \log \sum_{c\in [C]} p_ce^{v_c}.
\end{equation}
Thus, the first step in a given ADMM iteration, as in~\eqref{eq:V Update} (see also the beginning of the for-loop in Algorithm~\ref{alg:MP}), amounts to solving
\begin{equation} \label{eq:V1}
    \min_{\bv \in \BR^C} ~ \log \sum_{c\in [C]} p_ce^{v_c}  + \xi \|\bv\|_2^2 + \ba^T \bv
\end{equation}
for $\xi \defined \frac{\rho+\zeta}{2} > 0$ and some fixed vectors $(\bp,\ba)\in \simp{C}^+\times \BR^C$; see~\eqref{eq:p},~\eqref{eq:V Update} and~\eqref{eq:R} for explicit expressions. The problem~\eqref{eq:V1} is strictly convex. Further, we may recast this problem, via introducing the variable $\bz \in \BR^C$ by $z_c \defined v_c+\log p_c$, as
\begin{equation}
    \min_{\bz\in \BR^C} ~ \log \sum_{c\in [C]} e^{z_c} + \xi\|\bz\|_2^2 + \bb^T\bz,
\end{equation}
where $b_c = a_c -2\xi\log p_c$ is fixed. To solve this latter problem, it suffices to find a zero of the gradient, which is given by
\begin{equation}
    \nabla_{\bz} \left( \log \sum_{c\in [C]} e^{z_c} + \xi\|\bz\|_2^2 + \bb^T\bz \right) = \sigma(\bz) + 2 \xi \bz + \bb
\end{equation}
where $\sigma:\BR^C \to \simp{C}^+$ denotes the softmax function $\sigma(\bz) \defined \left( \frac{e^{z_{c'}}}{\sum_{c\in [C]} e^{z_c}} \right)_{c'\in [C]}$. Thus, we arrive at the fixed-point problem $\theta(\bz) = \bz$ for the function
\begin{equation} \label{eq:theta}
    \theta(\bz) \defined - \frac{1}{2\xi}\left( \sigma(\bz) + \bb \right).
\end{equation}
We solve $\theta(\bz)=\bz$ using a fixed-point-iteration method, i.e., with some initial $\bz_0$, we iteratively compute the compositions $\theta^{(m)}(\bz_0)$ for $m\in \BN$. This procedure is summarized in Algorithm~\ref{alg:V Update KL}. 

The exponentially-fast convergence of Algorithm~\ref{alg:V Update KL} is guaranteed in view of Lipschitzness of $\theta$ as defined in~\eqref{eq:theta}. Indeed, it is known that the softmax function is $1$-Lipschitz (see, e.g.,~\citep[Prop. 4]{Gao2017}); we improve this Lipschitz constant to $1/2$. This improvement yields a better guarantee on the convergence speed of \methodname. Indeed, as a lower value of the ADMM penalty $\rho$ correlates with a faster convergence, lowering the Lipschitz constant of the softmax function allows us to speed up \methodname~by choosing $\rho>\frac12 - \zeta $ instead of $\rho > 1-\zeta$.

\begin{algorithm}[t]
    \caption{\textbf{:} $\Argmin{\bv\in \BR^C} ~ D_{\mathsf{KL}}^\co(\bv,\bp) + \xi\|\bv\|_2^2 + \ba^T\bv$}
  \label{alg:V Update KL}
\begin{algorithmic}
    \STATE {\bfseries Input:} $\xi>0$, $\bp\in \simp{C}^+$, $\ba,\bv\in \BR^C$.
    \STATE $z_c \leftarrow v_c + \log p_c \hfill c\in [C]$
    \STATE $b_c \leftarrow a_c-2\xi \log p_c \hfill c\in [C]$
  \REPEAT
  \STATE $\bz \leftarrow -\frac{1}{2\xi}\left( \sigma(\bz)+\bb \right)$
  \UNTIL{convergence}
  \vspace{1mm}
  \STATE {\bfseries Output:} $v_c \defined z_c - \log p_c \hfill c\in [C]$
\end{algorithmic}
\end{algorithm}

\subsubsection{Proof of Lemma~\ref{lem:dfconj 1d}: primal update for general \texorpdfstring{$f$}{f}-divergences} \label{sec:lemma1d proof}

The lemma follows by the following sequence of steps:
\begin{align}
    \min_{\bv \in \BR^C} D_f^\co(\bv,\bp) + \xi \|\bv\|_2^2 + \ba^T\bv &\overset{\text{(I)}}{=} \min_{\bv\in \BR^C} \max_{\bq \in \simp{C}} \bq^T\bv - D_f(\bq \, \| \, \bp) + \ba^T\bv + \xi\|\bv\|_2^2 \\
    &\overset{\text{(II)}}{=}  \max_{\bq \in \simp{C}} \min_{\bv\in \BR^C} \bq^T\bv - D_f(\bq \, \| \, \bp) + \ba^T\bv + \xi\|\bv\|_2^2 \\
    &\overset{\text{(III)}}{=} \max_{\bq \in \simp{C}} -D_f(\bq\, \| \, \bp) - \frac{1}{4\xi}\|\ba + \bq\|_2^2 \\
    &= - \min_{\bq \in \simp{C}} D_f(\bq\, \| \, \bp) + \frac{1}{4\xi}\|\ba + \bq\|_2^2 \\
    &= - \min_{\bq \in \BR_+^C} \sup_{\theta \in \BR} D_f(\bq\, \| \, \bp) + \frac{1}{4\xi}\|\ba + \bq\|_2^2 + \theta\cdot \left(\mathbf{1}^T\bq - 1 \right) \\
    &\overset{\text{(IV)}}{=} -\sup_{\theta \in \BR} \min_{\bq \in \BR_+^C}  D_f(\bq\, \| \, \bp) + \frac{1}{4\xi}\|\ba + \bq\|_2^2 + \theta\cdot \left(\mathbf{1}^T\bq - 1 \right) \\
    &\overset{\text{(V)}}{=} - \sup_{\theta \in \BR} - \theta + \sum_{c\in [C]} \min_{q_c\ge 0} ~ p_cf\left( \frac{q_c}{p_c} \right) + \frac{1}{4\xi}(a_c+q_c)^2 + \theta q_c,
\end{align}
where (I)~holds by definition of $D_f^\co$ (see~\eqref{eq:Dfconj definition}), (II)~by Sion's minimax theorem, (III)~since the inner minimization occurs at $\bv=-\frac{1}{2\xi}(\bq+\ba)$, (IV)~by generalized minimax theorems~\citep[see, e.g.,~Chapter~VI,~Proposition~2.2 in][]{EkelandIvar1999Caav} (restated as Theorem~\ref{thm:minimax} herein for convenience), and (V)~by separability. 

\subsubsection{Primal update for cross-entropy} \label{sec:CE vupdate}

In the cross-entropy (CE) case, i.e., $f(t)=-\log t$, instead of using an explicit formula for $D_f^\co$ (which would yield unwieldy expressions), we utilize the reduction shown in Lemma~\ref{lem:dfconj 1d}. Thus, we have the equality
\begin{equation}
    \min_{\bv \in \BR^C} D_f^\co(\bv,\bp) + \xi \|\bv\|_2^2 + \ba^T\bv = - \sup_{\theta \in \BR} - \theta + \sum_{c\in [C]} \min_{q_c\ge 0} ~ p_cf\left( \frac{q_c}{p_c} \right) + \frac{1}{4\xi}(a_c+q_c)^2 + \theta q_c. \label{eq:CE dual}
\end{equation}
As per~\eqref{eq:CE dual}, we focus next on solving the inner single-variable minimization
\begin{equation}
    \min_{q\ge 0} -p\log q + \frac{1}{4\xi}(a+q)^2 + \theta q.
\end{equation}
It is easily seen that the solution to this minimization is the unique point making the objective's derivative vanish, i.e., it is $q^\star \in (0,\infty)$ for which
\begin{equation}
    -\frac{p}{q^\star} + \frac{q^\star}{2\xi} + \theta + \frac{a}{2\xi} = 0.
\end{equation}
This is easily solvable as a quadratic, yielding
\begin{equation}
    q^\star = \sqrt{\left( \theta \xi + \frac{a}{2} \right)^2 + 2p\xi} - \left( \theta \xi + \frac{a}{2} \right).
\end{equation}

Therefore, solving~\eqref{eq:CE dual} amounts to finding the constant $\theta\in \BR$ that yields a probability vector $\bq\in \simp{C}$, where
\begin{equation}
    q_c \defined  \sqrt{\left( \theta \xi + \frac{a_c}{2} \right)^2 + 2p_c\xi} - \left( \theta \xi + \frac{a_c}{2} \right).
\end{equation}
Consider the function
\begin{equation}
    g(z) \defined -1 + \sum_{c\in [C]} \sqrt{\left( z + \frac{a_c}{2} \right)^2 + 2p_c\xi} - \left( z + \frac{a_c}{2} \right),
\end{equation}
so we simply are looking for a root of $g$ (then set $\theta = z/\xi$ and $\bv = -\frac{1}{2\xi}(\bq+\ba)$). This can be efficiently accomplished via Newton's method. Namely, we compute
\begin{equation}
    g'(z) = -C + \sum_{c\in [C]} \frac{2z+a_c}{\sqrt{\left( z + \frac{a_c}{2} \right)^2 + 2p_c\xi}},
\end{equation}
then, starting from $z^{(0)}$, we form the sequence
\begin{equation}
    z^{(t+1)} \defined z^{(t)} - \frac{g\left( z^{(t)} \right)}{g'\left( z^{(t)} \right)}.
\end{equation}
This procedure is summarized in Algorithm~\ref{alg:V Update CE}.

\begin{algorithm}[t]
    \caption{\textbf{:} $\Argmin{\bv\in \BR^C} ~ D_{\mathsf{CE}}^\co(\bv,\bp) + \xi \|\bv\|_2^2 + \ba^T \bv$}
  \label{alg:V Update CE}
\begin{algorithmic}
    \STATE {\bfseries Input:} $\xi>0$, $z\in \BR$, $\bp\in \simp{C}^+$, $\ba\in \BR^C$.
    \vspace{1mm}
  \REPEAT
  \STATE ${\displaystyle g(z) \leftarrow -1 + \hspace{-2.5pt} \sum_{c\in [C]}\hspace{-2.5pt} \sqrt{\left( z + \frac{a_c}{2} \right)^2 + 2p_c\xi} - \left( z + \frac{a_c}{2} \right)}$ 
  \STATE ${\displaystyle g'(z) \leftarrow  -C +  \sum_{c\in [C]} \frac{2z+a_c}{\sqrt{\left( z + \frac{a_c}{2} \right)^2 + 2p_c\xi}}}$
  \STATE ${\displaystyle z \leftarrow z - \frac{g(z)}{g'(z)}}$
  \vspace{2mm}
  \UNTIL{convergence}
  \STATE {\bfseries Output:} ${\displaystyle v_c \defined \frac{1}{2\xi}\left( z - \frac{a_c}{2} - \sqrt{\left( z + \frac{a_c}{2} \right)^2 + 2p_c\xi} \right)}$
\end{algorithmic}
\end{algorithm}

\subsubsection{On the gradient of the convex conjugate of \texorpdfstring{$f$}{f}-divergence} \label{appendix:Dfconj gradient}

The following general result on the differentiability of $D_f^\co$ can be used to carry out the $\bv_i$-update step for a general $f$-divergence, and it will also be useful in Appendices~\ref{appendix:rate}--\ref{appendix:population} for proving the convergence rate of Algorithm~\ref{alg:MP} as stated in Theorems~\ref{thm:FairProjection}--\ref{thm:admm convergence kl}.

\begin{lemma} \label{lem:Dfconj gradient}
Suppose $f:(0,\infty)\to \BR$ is strictly convex. For any fixed $\bp \in \simp{C}^+$, the function $\bv\mapsto D_f^\co(\bv,\bp)$ is differentiable, and its gradient is given by
\begin{equation} \label{eq:Dfconj gradient}
    \nabla_{\bv} D_f^\co(\bv,\bp) = \bq_f^\co(\bv,\bp) \in \simp{C},
\end{equation}
where
\begin{equation} \label{eq:qconj def}
    \bq_f^\co(\bv,\bp) \defined \Argmin{\bq\in \simp{C}} ~ D_f(\bq \, \| \, \bp) - \bv^T\bq. 
\end{equation}
\end{lemma}
\begin{proof}
From~\citep[Proposition~11.3]{RockafellarR.Tyrrell2009Va}, since $\bq \mapsto D_f(\bq \, \| \, \bp)$ is a lower semicontinuous proper convex function, the subgradient of its convex conjugate $\bv \mapsto D_f^\co(\bv,\bp)$ is given by
\begin{equation} \label{eq:rhs}
    \partial_{\bv} D_f^\co(\bv,\bp) = \Argmin{\bq\in \simp{C}} ~ D_f(\bq \, \| \, \bp) - \bv^T\bq.
\end{equation}
Recall also that a function is differentiable at a point if and only if its subgradient there consists of a singleton \cite{borwein1987differentiability}. Thus, it only remains to show that the right-hand side in~\eqref{eq:rhs} is a singleton. For this, we note that $\bq \mapsto  D_f(\bq \, \| \, \bp) - \bv^T\bq$ is lower semicontinuous and strictly convex, and $\simp{C}$ is compact.
\end{proof}

\subsection{\texorpdfstring{\nicefrac{1}{2}}{1/2}-Lipschitzness of the Softmax Function} \label{appendix:softmax}

As stated in Section~\ref{sec:ADMM} and Appendix~\ref{sec:KL vupdate}, the convergence speed of the inner iteration (the $\bv_i$ update step) of \methodname~can be guaranteed to be faster if the Lipschitz constant of the softmax function is lowered from $1$ (which is proved in~\citep[Prop. 4]{Gao2017}). By Lipschitzness here, we mean $\ell_2$-norm Lipschitzness. We prove the following proposition in this appendix.

\begin{proposition} \label{prop:softmax}
For any $n\in \BN$, the softmax function $\sigma(\bz) \defined \left( \frac{e^{z_j}}{\sum_{i=1}^n e^{z_i}} \right)_{j\in [n]}$ is $\frac{1}{2}$-Lipschitz.
\end{proposition}

We will need the following result.

\begin{lemma}[Theorem 2.1.6 in~\cite{Nesterov_book}]
A twice continuously differentiable function $f:\BR^n\to \BR$ is convex and has an $L$-Lipschitz continuous gradient if and only if its Hessian is positive semidefinite with maximal eigenvalue at most $L$.
\end{lemma}

Since the softmax function is the gradient of the log-sum-exp function, and since the spectral norm is upper bounded by the Frobenius norm, it suffices to upper bound the Frobenius norm of the Jacobian of $\sigma$ by $1/2$. Suppose that $\sigma$ is operating on $n$ symbols. Consider the sum of powers functions $s_k(\bx) \defined \sum_{i\in [n]} x_i^k$ for $\bx \in \BR^n$. For any $\bv\in \BR^n$, denoting $\bx = \sigma(\bv)$, the square of the Frobenius norm of the Jacobian of $\sigma$ at $\bv$ is given by
\begin{equation}
    w(\bx) \defined s_2(\bx)^2 + s_2(\bx) - 2s_3(\bx).
\end{equation}
We show that $w(\bx)\le \frac{1}{4}$ for any $n\in \BN$ and $\bx \in \simp{n}$.

The approach we take is via reduction to the case $n \le 3$, which one can directly verify. Namely, assuming, without loss of generality, that $x_1\le x_2 \le \cdots \le x_n$, we show that if $x_1+x_2\le 1/2$ then $w(\by)\ge w(\bx)$ where $\by\in \simp{n-1}$ is given by $\by = (x_1+x_2,x_3,\cdots,x_n)$. Note that if $n\ge 4$ then we must have $x_1+x_2\le 1/2$, because $x_1+x_2 \le x_3+x_4$ and $x_1+x_2+x_3+x_4\le 1$. Thus, we will have reduced the problem from an $n\ge 4$ to $n-1$, which iteratively reduces the problem to $n\le 3$. Fix $n\ge 4$.

Denote $\bz=(x_3,\cdots,x_n)$. A direct computation yields that
\begin{equation}
\begin{aligned}
    w(\by)-w(\bx) &= 2x_1x_2\cdot \left( 2s_2(\bz) + g(x_1,x_2) \right)
\end{aligned}
\end{equation}
with the quadratic
\begin{equation}
    g(a,b) \defined 2a^2+2b^2 + 2ab - 3a-3b + 1.
\end{equation}
By assumption, $x_i\ge \max(x_1,x_2)$ for each $i\ge 3$, so $2s_2(\bz) \ge (n-2)x_1^2 + (n-2)x_2^2 \ge x_1^2 + x_2^2$. Then,
\begin{equation}
    w(\by)-w(\bx) \ge 2x_1x_2\cdot h(x_1,x_2)
\end{equation}
with
\begin{equation}
    h(a,b) \defined 3a^2+3b^2 + 2ab - 3a-3b + 1.
\end{equation}
Now, we show that $h$ is nonnegative for every $a,b\ge 0$ with $a+b\le 1/2$. With $c=a+b$, we may write
\begin{equation}
    h(a,b) = 3c^2 - (3+4a)c+4a^2+1.
\end{equation}
This quadratic in $c$ has its vertex at $c_{\min} = (3+4a)/6$. As $a\ge 0$, $c_{\min}\ge 1/2$. As $a+b \le 1/2$, we see that the minimum of $h$ is attained for $c=1/2$. Substituting $b=1/2-a$, we obtain
\begin{equation}
    h(a,b) = \left(2a - \frac12 \right)^2,
\end{equation}
which is nonnegative, as desired.

\subsection{Convergence rate of Algorithm~\ref{alg:MP}: proof of Theorem~\ref{thm:FairProjection}} \label{appendix:rate}

We recall a general result on the R-linear convergence rate for ADMM, which corresponds to case 1 in scenario 1 in~\cite{deng2016}; see Tables 1 and 2 therein. Recall that a sequence $\{z^{(t)}\}_{t\in \BN}$ is said to converge R-linearly to $z^\star$ if there is a constant $\eta \in (0,1)$ and a sequence $\{\beta^{(t)}\}_{t\in \BN}$ such that $\|z^{(t)}-z^\star\|\le \beta^{(t)}$ and $\sup_t \left( \beta^{(t+1)}/\beta^{(t)} \right) \le \eta$. In particular, one has exponentially small errors:
\begin{equation}
    \|z^{(t)}-z^\star\| \le \beta^{(0)}\cdot \eta^t.
\end{equation}
The following theorem is used in our proof of Theorem~\ref{thm:FairProjection}.

\begin{theorem}[\cite{deng2016}] \label{thm:ADMM linear}
Suppose that problem~\eqref{eq:ADMM general} has a saddle point, $F$ is strongly convex and differentiable with Lipschitz-continuous gradient, $\bA$ has full row-rank, and $\bB$ has full column-rank. Then, the ADMM iterations~\eqref{eq:admm V}--\eqref{eq:admm U} converge R-linearly to a global optimizer.
\end{theorem}

In Appendix~\ref{appendix:ADMM iterations}, we show that the dual~\eqref{eq:MP dual D N} of our fairness optimization problem~\eqref{eq:MP alt direct} can be written in the ADMM general form~\eqref{eq:ADMM general} with the choices
\begin{align}
    F(\bV) &= \frac{1}{N} \sum_{i\in [N]}  D_f^\co\left( \bv_i , \bp_i \right) + \frac{\zeta}{2}\left\| \bV \right\|_2^2
\end{align}
and
\begin{equation} 
    \bA = \frac{1}{\sqrt{N}}\bI_{NC}, \enskip \bB = \frac{1}{\sqrt{N}}(\bG_i)_{i\in [N]}^T. 
\end{equation}
Recall from Theorem~\ref{thm:hopt} (see also the proof in Appendix~\ref{appendix:thm duality}) that our problem~\eqref{eq:MP dual D N} has a saddle point. Further, the function $F:\BR^{NC}\to \BR$ is $\zeta$-strongly convex and differentiable. Indeed, each $\bv\mapsto D_f^\co(\bv,\bp_i)$ is convex, and the term $\frac{\zeta}{2}\|\bV\|_2^2$ is $\zeta$-strongly convex, so $F$ is $\zeta$-strongly convex too. In addition, by the formula for $\nabla D_f^\co$ in Lemma~\ref{lem:Dfconj gradient}, the gradient of $F$ is
\begin{equation} \label{eq:F gradient}
    \nabla F(\bV) = \bq_f^\co(\bV) + \zeta \bV,
\end{equation}
where
\begin{equation}
    \bq_f^\co(\bV) \defined \left( \bq_f^\co(\bv_i,\bp_i) \right)_{i\in [N]},
\end{equation}
with $\bq_f^\co(\bv_i)$ as defined in~\eqref{eq:qconj def}.

In the KL-divergence case, i.e., $f(t)=t\log t$, the gradient of $D_f^\co$ is given by the softmax function (see Appendix~\ref{sec:KL vupdate})
\begin{equation}
    \bq_f^\co(\bv,\bp) = \sigma\left( \bv + \log \bp \right) = \left( \frac{p_ce^{v_c}}{\sum_{c'\in [C]} p_{c'}e^{v_{c'}}} \right)_{c\in [C]}.
\end{equation}
Therefore, we have that
\begin{equation}
    \nabla F (\bV) = \left( \sigma\left(\bv_i+\log \bp_i \right) \right)_{i\in [N]} + \zeta \bV.
\end{equation}
By Proposition~\ref{prop:softmax}, the softmax function $\sigma$ is $\frac12$-Lipschitz. Hence, $\nabla F$ is $\left( \frac12 + \zeta \right)$-Lipschitz. 

Adding negligible noise, if necessary, to $\bB = N^{-1/2}\cdot \left( \bG_i \right)_{i\in [N]}^T \in \BR^{NC\times K}$ to make it have full column-rank, the general ADMM convergence rate in Theorem~\ref{thm:ADMM linear} yields that there is a constant $r > 0$ such that
\begin{equation} \label{eq:lambda exponential}
    \left\| \blambda_{\zeta,N}^{(t)} - \blambda_{\zeta,N}^\star \right\|_2 \le \beta \cdot e^{-rt}
\end{equation}
where $\beta \defined \left\| \blambda_{\zeta,N}^{(0)} - \blambda_{\zeta,N}^\star \right\|_2$. (Although Theorem~\ref{thm:ADMM linear} guarantees exponentially-fast convergence of $\blambda_{\zeta,N}^{(t)}$ to \emph{a} global optimizer, recall that $\blambda_{\zeta,N}^\star$ is the \emph{unique} optimizer of~\eqref{eq:MP dual D N}, as Theorem~\ref{thm:hopt} shows.) This exponentially-fast convergence shows that the choice of $t' \propto \log N$ iterations in Algorithm~\ref{alg:MP} yields arbitrary accuracy approximation of $\blambda_{\zeta,N}^\star$ by $\blambda_{\zeta,N}^{(t')}$. 

Finally, it remains to bound the distance between $\bh^{\opt,N}$ and the output classifier $\bh^{(t)}$ after the $t$-th iteration of Algorithm~\ref{alg:MP}. Note that $\phi(u) = (f')^{-1}(u) = e^{u-1}$, so $\gamma$ may be obtained explicitly, and equation~\eqref{eq:hopt formula} becomes
\begin{equation}
    h^{\opt,N}_{c'}(x) = \frac{h^{\base}_{c'}(x) \cdot e^{v_{c'}(x;\blambda_{\zeta,N}^{\star})}}{\sum_{c\in [C]} h^{\base}_c(x) \cdot e^{v_c(x;\blambda_{\zeta,N}^{\star})}}.
\end{equation}
Thus, using $\blambda^{(t)} \defined \blambda_{\zeta,N}^{(t)}$ in place of $\blambda_{\zeta,N}^\star$, we obtain that the $t$-th classifier obtained by Algorithm~\ref{alg:MP} is
\begin{equation}
    h^{(t)}_{c'}(x) = \frac{h^{\base}_{c'}(x) \cdot e^{v_{c'}(x;\blambda^{(t)})}}{\sum_{c\in [C]} h^{\base}_c(x) \cdot e^{v_c(x;\blambda^{(t)})}}.
\end{equation}
Therefore, we have the ratios
\begin{equation} \label{eq:h ratio}
    \frac{h_{c'}^{(t)}(x)}{h_{c'}^{\opt,N}(x)} = \frac{\sum_{c\in [C]} h_c^{\base}(x) e^{v_c(x;\blambda_{\zeta,N}^\star)}}{\sum_{c\in [C]} h_c^{\base}(x) e^{v_c(x;\blambda^{(t)})}} \cdot \exp\left( v_{c'}(x;\blambda^{(t)}) - v_{c'}(x;\blambda_{\zeta,N}^\star) \right).
\end{equation}
By definition of $\bv$, $\bv(x;\blambda) = - \bG(x)^T\blambda$. Thus, we obtain from~\eqref{eq:lambda exponential} and boundedness of $\bG$ that
\begin{equation} \label{eq:v bound}
    \left\| \bv(x;\blambda^{(t)}) - \bv(x;\blambda_{\zeta,N}^\star) \right\|_\infty = O(e^{-t}),
\end{equation}
where the implicit constant is independent of $x$. Applying~\eqref{eq:v bound} in~\eqref{eq:h ratio}, and noting that $e^{e^{-t}}=1+O(e^{-t})$ as $t\to \infty$, we conclude that 
\begin{equation}
    \bh^{(t)}(x) = \bh^{\opt,N}(x) \cdot \left( 1 + O(e^{-t}) \right) 
\end{equation}
uniformly in $x$.

\subsection{Convergence rate to the population problem: proof of Theorem~\ref{thm:admm convergence kl}} \label{appendix:population}

The proof is divided into several lemmas. We note first that, in the course of the proof of Theorems~1 and~2 in~\cite{ISIT}, it was shown that at least one minimizer $\blambda^\star$ of~\eqref{eq:MP dual D} exists. Further, any such minimizer satisfies the following bound. Denote the constraint function by $\bmu(\bh) \defined \BE_{P_X}[\bG \bh]$. Throughout this proof, we set $\calX \defined \BR^d$.

\begin{lemma} \label{lem:lambda}
Suppose Assumption~\ref{assumption} holds, and fix a strictly feasible classifier $\bh\in \calH$, i.e., $\bmu(\bh)<\mathbf{0}$. Every minimizer $\blambda^\star\in \BR_+^K$ of~\eqref{eq:MP dual D} must satisfy the inequality
\begin{equation}
    \left\| \blambda^\star \right\|_1 \le \lambda_{\max} \defined \frac{D_f\left( \bh \, \| \, \bh^\base \mid P_X \right)}{{\displaystyle \min_{k\in [K]}} - \mu_k(\bh)}.
\end{equation}
\end{lemma}
We note that for the fairness metrics specified in Table~\ref{table:Fair}, one valid choice of a strictly feasible $\bh$ (i.e., one for which $\bmu(\bh)<\mathbf{0}$) is the uniform classifier $\bh(x) \equiv \frac{1}{C}\mathbf{1}$. In any case, we have that $\lambda_{\max}<\infty$ since both $\bh$ and $\bh^\base$ are assumed to belong to $\calH$ and $f$ is continuous over $(0,\infty)$; e.g., one bound on $\lambda_{\max}$ is $\lambda_{\max} \le \max_{m\le t \le M} f(t)/ \min_{k\in [K]} - \mu_k(\bh)$ 
where $m=\inf_{c,x}h_c(x)$ and $M=1/\inf_{x,c} h_c^\base(x)$. We will also need the following constants for the convergence analysis:
\begin{align} 
    g_{\Mean} &\defined \BE\left[\left\| \bG(X) \right\|_2^2 \right], \label{eq:def gmean} \\
    g_{\max} &\defined \sup_{x\in \calX} \left\| \bG(x) \right\|_2^2 \label{eq:def gmax}.
\end{align}
Clearly, $g_{\Mean}\le g_{\max}$. By the boundedness of $\bG$ in the second item in Assumption~\ref{assumption}, $g_{\max}$ is finite.

\begin{remark}
Although the results in this paper are stated to hold under Assumption~\ref{assumption}, we note that those conditions do not essentially impose any restriction on carrying our \methodname~algorithm. Indeed, we focus in this paper on the CE and KL cases, for which $f$ satisfies the imposed conditions. We also note that only boundedness of $\bG$ is required for Theorem~\ref{thm:hopt}, which is true for the fairness metrics in Table~\ref{table:Fair} in non-degenerate cases (e.g., no empty groups). The condition on $\bh^\base$ being bounded away from zero can be made to hold by perturbing it if necessary with negligible noise. The condition on $\bh^\base$ being continuous is automatically satisfied if its domain is a finite set (as is the case for Theorem~\ref{thm:hopt}). Finally, the strict feasibility condition is verified by the uniform classifier.
\end{remark}

Now, consider a form of $\ell_2$ regularization of~\eqref{eq:MP dual D}:
\begin{equation} \label{eq:MP dual regular}
    \min_{\blambda \in \BR_+^K} ~ \BE\left[ D_f^\co\left(\bv(X;\blambda), \bh^\base(X) \right) +\frac{\zeta}{2} \left\| \tilde{\bG}(X)^T \blambda \right\|_2^2 \right]
\end{equation}
where $\tilde{\bG}(x) \defined \left( \bG(x), \bI_K \right) \in \BR^{K\times (K+C)}$. We show now that there is a unique minimizer $\blambda_\zeta^\star$ of~\eqref{eq:MP dual regular}. 

\begin{lemma} \label{lem:lambda zeta}
Under Assumption~\ref{assumption}, there exists a unique minimizer $\blambda_\zeta^\star$ of the regularized problem~\eqref{eq:MP dual regular}.
\end{lemma}
\begin{proof}
Denote the function $A:\BR_+^K \to \BR$ by
\begin{equation}
    A(\blambda) \defined \BE\left[ D_f^\co\left(\bv(X;\blambda), \bh^\base(X) \right) +\frac{\zeta}{2} \left\| \tilde{\bG}(X)^T \blambda \right\|_2^2 \right].
\end{equation}
That the range of $A$ falls within $\BR$ follows by Assumption~\ref{assumption}, since then the function $x\mapsto D_f^\co(\bv(x;\blambda),\bh^\base(x))$ is $P_X$-integrable. We will show that $A$ is lower semicontinuous and $\zeta$-strongly convex. 

By Lemma~\ref{lem:Dfconj gradient}, $\bv\mapsto D_f^\co(\bv,\bp)$ is differentiable for any fixed $\bp \in \simp{C}^+$, implying that it is also continuous. Thus, $\blambda \mapsto D_f^\co(\bv(x;\blambda),\bh^\base(x))$ is continuous for each $x\in \calX$. Hence, by Fatou's lemma and boundedness of $\bG$, $A$ is lower semicontinuous. 

Next, to show strong convexity, we note that $\blambda \mapsto D_f^\co(\bv(x;\blambda),\bh^\base(x))$ is convex for each $x\in \calX$. Indeed, this function is the supremum of affine functions. Further, the regularization term is $\zeta$-strongly convex, as its Hessian is given by
\begin{equation}
    \zeta \cdot \left( \BE\left[ \tilde{\bG}(X)\tilde{\bG}(X)^T \right] + \bI \right),
\end{equation}
which is positive definite with minimal eigenvalue at least $\zeta$.

Now, for each fixed $\theta>0$, consider the compact set $\Lambda_\theta \defined \{\blambda \in \BR_+^K ~ ; ~ \|\blambda\|_2^2 \le \theta \}$. By what we have shown thus far, there is a unique minimizer $\blambda_\theta$ of $A$ over $\Lambda_\theta$. By strong convexity, if $A$ has a global minimizer then it is unique. We will show that $\blambda_\theta$ is a global minimizer of $A$, where $\theta = 2(A(\mathbf{0})-D^\star)/\zeta$. Suppose that $\mathbf{0}$ is not a global minimzer. Fix $\blambda\in \BR_+^K$ such that $A(\mathbf{0})>A(\blambda)$. Then,
\begin{equation}
    A(\mathbf{0}) > A(\blambda) \ge D^\star + \frac{\zeta}{2} \left( \BE\left[ \left\| \bG(X)^T\blambda \right\|_2^2 \right] + \left\| \blambda \right\|_2^2 \right) \ge D^\star + \frac{\zeta}{2} \left\| \blambda \right\|_2^2.
\end{equation}
Thus, $\|\blambda\|_2^2<\theta$. This implies that $\blambda_\theta$ is a global minimizer of $A$, hence it is the unique global minimizer of $A$. The proof of the lemma is thus complete.
\end{proof}

The following bound shows that $\blambda_{\zeta}^\star$ is within $O(\zeta)$ of achieving $D^\star$ (see~\eqref{eq:MP dual D}).
\begin{lemma} \label{lem:bound zeta}
Suppose Assumption~\ref{assumption} holds, fix $\zeta\ge 0$, and denote the unique solution and the optimal objective value of~\eqref{eq:MP dual regular} by $\blambda_\zeta^\star$ and $D_\zeta^\star$, respectively. We have the bounds
\begin{align}
    \BE\left[ D_f^\co\left(\bv(X;\blambda_\zeta^\star), \bh^\base(X) \right) \right] \le D_\zeta^\star  \le D^\star + \theta_{\textup{reg}}\cdot \zeta, \label{eq:bound D zeta}
\end{align}
where we define the constant $\theta_{\textup{reg}}\defined \lambda_{\max}^2\cdot (1+g_{\Mean})/2$. 
\end{lemma}
\begin{proof}
The first bound is trivial. Using Lemma~\ref{lem:lambda}, we may fix a $\blambda^\star \in \BR_+^K$ with $\|\blambda^\star\|_1\le \lambda_{\max}$ such that $\blambda^\star$ achieves $D^\star$. By definition of $D_\zeta^\star$,
\begin{align*}
    D_\zeta^\star \le \BE\left[ D_f^\co\left(\bv(X;\blambda^\star), \bh^\base(X) \right) + \frac{\zeta}{2} \left\| \tilde{\bG}(X)^T \blambda^\star \right\|_2^2\right] \le  D^\star + \theta_{\text{reg}}\cdot \zeta,
\end{align*}
where the last inequality follows since for the $2$-matrix norm, $\|\bM \blambda\|_2 \le \|\bM\|_2 \|\blambda\|_2$ and $\|\bM^T\|_2 = \|\bM\|_2$.
\end{proof}

Next, we derive a sample-complexity bound for the finite-sample problem~\eqref{eq:MP dual D N} via generalizing the proofs of Theorem~3 in~\cite{ISIT} and Theorem~13.2 in~\cite{hajek2019statistical}.

\begin{lemma} \label{lem:sample complexity}
Suppose Assumption~\ref{assumption} holds, and let $\lambda_{\max}$ and $g_{\max}$ be as defined in Lemma~\ref{lem:lambda} and equation~\eqref{eq:def gmax}. For any $\delta \in (0,1)$, with $\blambda_{\zeta,N}^\star$ denoting the unique solution to~\eqref{eq:MP dual D N}, it holds with probability at least $1-\delta$ that
\begin{equation}
\begin{aligned}
    \BE_X\left[ D_f^\co\left(\bv(X;\blambda_{\zeta,N}^\star), \bh^\base(X) \right) \right]  \le D^\star_\zeta + \frac{2g_{\max} \cdot \left( 1 + \zeta \cdot \lambda_{\max} \right)^2}{\delta \zeta N}.
\end{aligned}
\end{equation}
\end{lemma}
\begin{proof}
Let $\Lambda \defined \{ \blambda \in \BR_+^K ~ ; ~ \|\blambda\|_1 \le \lambda_{\max}\}$, and consider the function $\ell:\Lambda \times \calX \to \BR$ defined by
\begin{equation}
    \ell(\blambda,x) \defined D_f^\co\left( \bv(x;\blambda) , \bh^\base(x) \right) + \frac{\zeta}{2} \left\| \tilde{\bG}(x)^T\blambda \right\|_2^2.
\end{equation}
Note that the regularized problem~\eqref{eq:MP dual regular} can be written as
\begin{equation}
    D_\zeta^\star \defined \min_{\blambda \in \BR_+^K } ~ \BE\left[ \ell(\blambda,X) \right],
\end{equation}
and the finite-sample version of it~\eqref{eq:MP dual D N} can also be written as
\begin{equation}
    D_{\zeta,N}^\star \defined \min_{\blambda \in \BR_+^K } ~ \frac{1}{N} \sum_{i\in [N]} \ell(\blambda,X_i).
\end{equation}

We show first that, for each fixed $x\in \calX$, the function $\blambda \mapsto \ell(\blambda,x)$ is $\zeta$-strongly convex over $\Lambda$. The gradient of the regularization term is $\zeta \tilde{\bG}(x)^T\blambda$, and its Hessian is given by
\begin{equation}
    \nabla^2_{\blambda} ~ \frac{\zeta}{2} \left\| \tilde{\bG}(x)^T\blambda \right\|_2^2 = \zeta \bG(x)\bG(x)^T + \zeta \bI_K.
\end{equation}
Further, the function $\blambda \mapsto D_f^\co(\bv(x;\blambda),\bh^\base(x))$ is convex as it is a pointwise supremum of linear functions. Indeed, for any $\bp \in \simp{C}$, recalling that $\bv(x;\blambda) = - \bG(x)^T\blambda$, we have the formula
\begin{equation}
    D_f^\co(\bv(x;\blambda),\bp) =  \sup_{\bq\in \simp{C}} -\bq^T \bG(x)^T\blambda - D_f(\bq \, \| \, \bp).
\end{equation}

Next, we show Lipschitzness of $\blambda \mapsto \ell(\blambda,x)$. For any fixed $\bv \in \BR^C$ and $\bp \in \simp{C}^+$, we have the gradient (see Lemma~\ref{lem:Dfconj gradient})
\begin{equation}
    \nabla_{\bv} D_f^\co(\bv,\bp) = \bq^\co(\bv) \in \simp{C},
\end{equation}
where
\begin{equation}
    \bq^\co(\bv) \defined \Argmin{\bq\in \simp{C}} ~ D_f(\bq \, \| \, \bp) - \bv^T\bq. 
\end{equation}
Thus, we have the gradient
\begin{equation}
    \nabla_{\blambda} D_f^\co\left(\bv(x;\blambda), \bh^\base(x) \right) = -\bG(x)\bq^\co\left(\bv(x;\blambda)\right).
\end{equation}
Hence, the gradient of $\blambda \mapsto \ell(\blambda,x)$ is
\begin{equation}
    \nabla_{\blambda} \ell(\blambda,x) = -\bG(x)\bq^\co\left(\bv(x;\blambda)\right) + \zeta \tilde{\bG}(x)^T\blambda,
\end{equation}
which therefore satisfies the bound
\begin{equation}
    \left\| \nabla_{\blambda} \ell(\blambda,x) \right\|_2 \le \|\bG(x)\|_2\left( 1 + \zeta\cdot \lambda_{\max} \right).
\end{equation}
Therefore, each $\blambda \mapsto \ell(\blambda,x)$ is $A$-Lipschitz with
\begin{equation}
    A = \left( 1 + \zeta\cdot \lambda_{\max} \right) \cdot \sup_{x\in \calX} \|\bG(x)\|_2.
\end{equation}
Thus, by Theorem~13.1 in~\cite{hajek2019statistical}, with probability $1-\delta$ we have the bound
\begin{equation}
    \BE_{X}\left[ \ell\left(\blambda_{\zeta,N}^\star,X \right) \right] \le D_\zeta^\star + \frac{2A^2}{\delta \zeta N}.
\end{equation}
With probability one, we have the bound
\begin{equation}
\begin{aligned}
    \BE_X\left[ D_f^\co\left( \bv\left(X;\blambda_{\zeta,N}^\star\right),\bh^\base(X) \right) \right]  \le \BE_{X}\left[ \ell\left(\blambda_{\zeta,N}^\star,X \right) \right].
\end{aligned}
\end{equation}
This completes the proof of the lemma. 
\end{proof}

Now, we are ready to finish the proof of Theorem~\ref{thm:admm convergence kl} by specializing the above lemmas to the KL-divergence case. So, we set $f(t)=t\log t$ for the rest of the proof. By Lemmas~\ref{lem:bound zeta}--\ref{lem:sample complexity}, we have with probability $1-\delta$
\begin{equation}
    \BE_X\left[ D_f^\co\left(\bv(X;\blambda_{\zeta,N}^\star), \bh^\base(X) \right) \right]  \le D^\star + \theta_{\mathrm{reg}}\cdot \zeta + \frac{2g_{\max} \cdot \left( 1 + \zeta \cdot \lambda_{\max} \right)^2}{\delta \zeta N}.
\end{equation}
Thus, by Lipschitzness (Proposition~\ref{prop:softmax}) and~\eqref{eq:lambda exponential} 
\begin{equation} \label{eq:D bound}
    \BE_X\left[ D_f^\co\left(\bv(X;\blambda_{\zeta,N}^{(t)}), \bh^\base(X) \right) \right]  \le D^\star + \frac{1}{2} \sqrt{g_{\Mean}}\beta e^{-rt}+ \theta_{\mathrm{reg}}\cdot \zeta + \frac{2g_{\max} \cdot \left( 1 + \zeta \cdot \lambda_{\max} \right)^2}{\delta \zeta N}.
\end{equation}
Here, we are choosing the constants $\beta$ and $r$ independently of $N$, as can be guaranteed from Corollary~3.1 and Theorem~3.4 in~\cite{deng2016}. 

Choose $\zeta = \Theta(N^{-1/2})$. Collecting the constants in~\eqref{eq:D bound}, we obtain that
\begin{equation}
    \BE_X\left[ D_f^\co\left(\bv(X;\blambda_{\zeta,N}^{(t)}), \bh^\base(X) \right) \right]  \le D^\star + \frac{1}{2} \sqrt{g_{\Mean}}\beta e^{-rt}+ \frac{\ell}{\delta \sqrt{N}}
\end{equation}
for some constant $\ell$ that is completely determined by $\theta_{\mathrm{reg}}$, $g_{\max}$, and $\lambda_{\max}$. This bound can be further upper bounded by $D^\star + O(N^{-1/2})$ by choosing $t\ge \frac{1}{2r} \log N$, thereby completing the proof of the theorem.

\subsection{More details on group-fair classifiers} 
\label{appendix:fair}

We include, for completeness, how the group-fairness metrics in Table~\ref{table:Fair} linearize, i.e., written in the form:
\begin{align}
    \bmu(\bh) \defined \BE_{P_X}\left[ \bG(X)\bh(X) \right] \le \mathbf{0}.
\end{align}
We assume that we have in hand a well-calibrated classifier that approximates $P_{Y,S|X}$, i.e., that predicts both group membership $S$ and the true label $Y$ from input variables $X$. This classifier can be directly marginalized into the following models:
\begin{itemize}
    \item a label classifier $\bh^{\base}: \mathcal{X}\to \simp{C}$ that predicts true label from input variables, 
    \begin{equation}\label{Classifier_y}
         \bh^{\base}(x) \defined (P_{Y|X=x}(1),\cdots,P_{Y|X=x}(C)) ~~~ \text{for} ~~ x\in \calX,
    \end{equation}
    
    \item a group membership classifier $\bs: \mathcal{X}\times \mathcal{Y} \to \simp{A}$ that uses input and output variables to predict group membership,
    \begin{equation}\label{Classifier_s}
        \bs(x,y) \defined (P_{S|X,Y}(1\mid x,y),\cdots,P_{S|X,Y}(A\mid x,y)) ~~~ \text{for} ~~ (x,y)\in \calX\times\calY,
    \end{equation}
\end{itemize}

We let $\be_1,\cdots,\be_C$ denote the standard basis vectors of $\Reals^C$. We suppose that the support of the group attribute $S$ is $\calS \defined [A]$.

\paragraph{Statistical parity.} This fairness metric measures whether the predicted outcome $\widehat{Y}$ is independent of the group attribute $S$. 
For statistical parity, the $\bG(x)$ matrix has rows:
\begin{align*}
   \left((-1)^{\delta} \frac{\sum_{c=1}^C s_a(x,c)h^{\base}_c(x)}{P_S(a)} - \left(\alpha + (-1)^{\delta}\right)\right) \be_{c'}.
\end{align*}
There are $K=2AC$ rows since $(\delta,a, c')\in \{0,1\}\times[A]\times [C]$.

\paragraph{Equalized odds.} This fairness metric requires the predicted outcome $\widehat{Y}$ and the group attribute $S$ to be independent conditioned on the true label $Y$. When the classification task is binary, the equalized odds becomes the equality of false positive rate and false negative rate over all groups. For equalized odds, the $\bG(x)$ matrix has rows:
\begin{align*}
    \left((-1)^{\delta}\frac{s_{a'}(x,c) h_c^{\base}(x)}{P_{S|Y=c}(a')} - \left(\alpha + (-1)^{\delta}\right)h^{\base}_c(x) \right)\be_{c'}.
\end{align*}
There are $K=2AC^2$ rows.

\paragraph{Overall accuracy equality.} This fairness metric requires the accuracy of the predictive model to be the same across all group groups. The $\bG(x)$ matrix has rows:
\begin{equation*}
    (-1)^{\delta}\frac{s_a(x,\cdot) \odot \bh^{\base}(x)}{P_S(a)} - \left(\alpha + (-1)^{\delta} \right) \cdot \bh^{\base}(x),
\end{equation*}
where $\odot$ represents the element-wise product. There are $K=2A$ rows.

\section{Additional experiments and more details on the experimental setup} \label{appendix:exps}

\subsection{Numerical Benchmark Details} \label{appendix:datasets}
\subsubsection{Datasets}
The HSLS dataset is collected from 23,000+ participants across 944 high schools in the USA, and it includes thousands of features such as student demographic information, school information, and students' academic performance across several years. 
We preprocessed the dataset (e.g., dropping rows with a significant number of missing entries, performing k-NN imputation, normalization), and the number of samples reduced to 14,509. 

The ENEM dataset, collected from the 2020 Brazilian high school national exam and made available by the Brazilian Government \cite{enem2017}, is comprised of student demographic information, socio-economic questionnaire answers (e.g., parents education level, if they own a computer) and exam scores. 
We preprocess the dataset by removing missing values, repeated exam takers, and students taking the exam before graduation (``treineiros'') and obtain $\sim$1.4 million samples with 138 features.

\subsubsection{Hyperparameters}
For logistic regression and gradient boosting, we use the default parameters given by Scikit-learn. For random forest, we set the number of trees and the minimum number of samples per leaf to 10. For all classifiers, we fixed the random state to 42.
When running \methodname~(cf.~Algorithm~\ref{alg:MP}), we set the hyperparameters $\zeta = 1/\sqrt{N}$ (see Theorem~\ref{thm:admm convergence kl}) and $\rho = 2$ (see Appendix~\ref{sec:KL vupdate}), where $N$ is the number of samples.

\subsubsection{Benchmark Methods} \label{subsub:benchmark}
For binary classification, we compare with six different benchmark methods:
\begin{itemize}
    \item EqOdds~\cite{hardt2016equality}: We use AIF360 implementation of \textsf{EqOddsPostprocessing} and we use 50\% of the test set as a validation set, i.e., 70\% training set, 15\% validation set, 15\% test set.
    \item CalEqOdds~\cite{pleiss2017fairness}: We use AIF360 implementation of \textsf{CalibratedEqOddsPostprocessing} and we use 50\% of the test set as a validation set, i.e., 70\% training set, 15\% validation set, 15\% test set.
    \item Reduction~\cite{agarwal2018reductions}: We use AIF360 implementation of \textsf{ExponentiatedGradientReduction}, and we use 10 different epsilon values as follows: $[0.001, 0.005, 0.01, 0.02, 0.05, 0.1, 0.2, 0.5, 1, 2]$. We used \textsf{EqualizedOdds} constraint for MEO experiments and \textsf{DemographicParity} for statistical parity experiments.
    
    \item Rejection~\cite{kamiran2012reject}: We use AIF360 implementation of \textsf{RejectOptionClassification}. We use the default parameters except \textsf{metric\_{ub}} and \textsf{metric\_lb}, namely, \textsf{low\_class\_thresh}$\ =0.01$, \textsf{high\_class\_thresh}$\ =0.99$, \textsf{num\_class\_thresh}$\ =100$, \textsf{num\_ROC\_margin}$\ =50$. We set the values \textsf{metric\_ub}$\ =\epsilon$ and \textsf{metric\_lb}$\ =-\epsilon$ to obtain  trade-off curves. Epsilon values we used are: $[0.001, 0.005, 0.01, 0.02, 0.05, 0.1, 0.2, 0.5, 1, 2]$.
    
    \item FACT~\cite{kim2020fact}: We used the code provided on the \href{https://github.com/wnstlr/FACT}{Github repo}. We did not include the results in the main text as we found that: 
    \begin{enumerate}[label=(\roman*)]
        \item This method is not directly comparable because they find post-processing parameters on the entire test set and apply them on the test set. This is different from all other methods we are comparing including our method, which use training set or a separate validation set to fit the post-processing mechanism. For this reason, FACT often has a point that lies above all other curves on the accuracy-fairness plot. However, this is not a fair comparison. We include the results of FACT in the COMPAS plots for the sake of demonstration.
        \item We found the results produced by this method inconsistent. Partial reason is due to the problem of finding mixing rates---probability of flipping $\widehat{Y}=1$ to $0$ (i.e., $P(\widetilde{Y} = 0| \widehat{Y}=1)$) and vice-versa---which have to be between 0 and 1. But there are cases where these values lie outside $[0,1]$, which leads to erroneous and inconsistent results.
    \end{enumerate}
    For the results we present in the COMPAS plots, we used 20 epsilon values from 1 to $10^{-4}$, equidistant in log space. We used 10 different train/test splits as we do in all other experiments. If certain splits does not produce a feasible solution, we drop those results. If none of the 10 splits produce a feasible solution, we drop the epsilon value. At the end, we had 19 epsilon values.
    
    \item LevEqOpp~\cite{chzhen2019leveraging}: We used the code provided in the \href{https://github.com/lucaoneto/NIPS2019_Fairness}{Github repo}, originally programmed in R. We converted it into Python, and verified that the Python version achieved similar accuracy/fairness performance to their R version on UCI Adult dataset. We follow the same hyperparameters setup in \cite{chzhen2019leveraging}.
    
\end{itemize}

For multi-class comparison, we compare with Adversarial~\cite{zhang2018mitigating}. In theory, the adversarial debiasing method is applicable to multi-class labels and groups, but its AIF360 implementation works only for binary labels and binary groups. We adapted their implementation to work on multi-class labels by changing the last layer of the classifier model from one-neuron sigmoid activation to multi-neuron soft-max activation. We varied \textsf{adversary\_loss\_weight} to obtain a trade-off curve, values taken from $[0.001, 0.01, 0.1, 0.2, 0.35, 0.5, 0.75]$. For all other parameters, we used the default values: \textsf{num\_epochs}$\ =50$, \textsf{batch\_size}$\ =128$, \textsf{classifier\_num\_hidden\_units}$\ =200$. 

There are some methods that are relevant to our work but we could not benchmark in our experiments due to the lack of publicly available codes, including \cite{wei2021optimized}, \cite{menon2018cost}, \cite{jang2022group}.

\subsection{Additional experiments on runtime of FairProjection} \label{appendix:runtime}
We preform an ablation study on the runtime to illustrate that the parallelizability of {\methodname} can significantly reduce the runtime, especially when the dataset contains hundreds of thousands of samples. 
We report the runtime of \texttt{FairProjection-KL} on ENEM with 2 classes, 2 groups, and with different sizes. 
In Table~\ref{table:parallel-non-runtime}, we observe that when the number of samples exceeds 200k, parallelization leads to 10.1$\times$ to 15.5$\times$ speedup of the runtime.

\begin{table}[!ht]
\scriptsize
\centering
\begin{tabular}{ccccccc}
\toprule
\multirow{2}{*}{\textbf{Method}} &  \multicolumn{6}{c}{\textbf{\# of Samples (in thousands)}} \\
& 20 & 50 & 100 & 200 & 500 & $\sim$1400 \\
\toprule
Non-Parallel & 0.37$\pm$0.00 & 0.87$\pm$0.01 & 1.72$\pm$0.01 & 3.53$\pm$0.01 & 9.09$\pm$0.01 & 25.26$\pm$0.02 \\
Parallel (GPU) & 0.18$\pm$0.00 & 0.22$\pm$0.01 & 0.25$\pm$0.01 & 0.32$\pm$0.01 & 0.64$\pm$0.01 & 1.63$\pm$0.05  \\
Speedup & 2.00$\times$ & 3.92$\times$ & 7.21$\times$ & 10.97$\times$ & 14.23$\times$ & 15.46$\times$ \\
\bottomrule
\end{tabular}
\vspace{5mm}\caption{Execution time of parallel (on GPU) and non-parallel (on CPU) versions of the \texttt{FairProjection-KL} ADMM algorithm on the ENEM datasets with different sizes (time shown in minutes) with gradient boosting base classifiers.}
\label{table:parallel-non-runtime}
\end{table}

\subsection{Omitted Experimental Results on Accuracy-Fairness Trade-off} \label{appendix:exp-trade-off}

\subsubsection{Accuracy-fairness trade-off in binary classification}\label{appendix:acc-fairness-binary}
We include the results of benchmark methods and Fair Projection on 4 datasets (HSLS, ENEM-50k, Adult, and COMPAS) and 3 base classifiers (Logistic regression, Random forest, and GBM) in Figures~\ref{fig:hsls-eo-all}-\ref{fig:adult-sp-all}. For  equalized odds experiments, we have six benchmark methods (\texttt{EqOdds}, \texttt{Rejection}, \texttt{Reduction}, \texttt{CalEqOdds}, \texttt{FACT}, \texttt{LevEqOpp}). For statistical parity experiments, we have \texttt{Rejection} and \texttt{Reduction}. We plot Fair Projection with both cross entropy and KL divergence. 

When a method performs significantly worse than others, we did not plot its results. We did not include \texttt{Rejection} in the Adult plots as it did not produce consistent and reliable results on this dataset. \texttt{CalEqOdds} is included only in COMPAS as its performance was significantly worse and the point was too far away from other curves in all other datasets. \texttt{FACT} is also included only in the COMPAS plots and the reasons for this are explained in Appendix~\ref{subsub:benchmark}.

We observe that Fair Projection performs consistently well in all four datasets. \texttt{FairProjection-CE} and \texttt{FairProjection-KL} have similar performance (i.e., overlapping curves) in most cases. The performance of Fair Projection is often comparable with \texttt{Reduction}. \texttt{Rejection} has competitive performance in ENEM-50k and HSLS, but its performance falters in COMPAS and Adult. \texttt{EqOdds} produces a point with very low MEO but with a substantial loss in accuracy. \texttt{LevEqOpp} also yields a point with low MEO but with a much smaller accuracy drop. Even though \texttt{LevEqOpp}  only optimizes for FNR difference between two groups, it performs surprisingly well in terms of MEO in all four datasets. However, we note that \texttt{LevEqOpp} can only produce a point, not a curve, and it does not enjoy the generality of Fair Projection as it is specifically designed for binary-class, binary-group predictions and minimizing Equalized Opportunity difference.

\subsubsection{Accuracy-fairness trade-off in multi-class/multi-group classification}\label{appendix:acc-fairness-multi}
In the main text, we showed the performance of \texttt{FairProjection-CE} on multi-class prediction with 5 classes and 2 groups (see Figure~\ref{fig:multi}). We include results under a few different multi-class settings here. First, we show results on ENEM-50k-5-5 which has 5 classes and 5 groups in Figure~\ref{fig:enem-50k-5-5-meo} and~\ref{fig:enem-50k-5-5-sp} . We obtain 5 groups by not binarizing the race feature.  Then, we show results on binary classification with 5 groups in Figure~\ref{fig:enem-50k-2-5-meo} and~\ref{fig:enem-50k-2-5-sp}. Finally, we include the extended version of Figure~\ref{fig:multi} that include both \texttt{FairProjection-CE} and \texttt{FairProjection-KL} in Figure~\ref{fig:enem-multi-adv}.

To measure multi-class performance, we extend the definition of mean equalized odds (MEO) and statistical parity as follows: 
\begin{align}\label{eq:multi-metrics}
    \textsf{MEO} &= \max_{i \in \mathcal{Y}} \max_{s_1, s_2 \in \mathcal{S}} (| \textsf{TPR}_i(s_1) - \textsf{TPR}_i(s_2)| + | \textsf{FPR}_i(s_1) - \textsf{FPR}_i(s_2)|) /2 \\ 
    \textsf{Statistical Parity} &= \max_{i \in \mathcal{Y}} \max_{s_1, s_2 \in \mathcal{S}} | \textsf{Rate}_i(s_1) - \textsf{Rate}_i(s_2)|
\end{align}
where we denote 
$\textsf{TPR}_i(s) = P(\widehat{Y} = i \mid Y = i, S=s)$,  $\textsf{FPR}_i(s) = P(\widehat{Y} = i \mid Y \neq i, S=s)$, and $\textsf{Rate}_i(s) = P(\widehat{Y} = i \mid S=s)$. 

In all experiments, \methodname~reduces MEO and statistical parity significantly (e.g., 0.22 to 0.14) with a negligible  sacrifice in accuracy.

\clearpage

\begin{figure}[t!]
\begin{center}
\centerline{\includegraphics[width=\columnwidth]{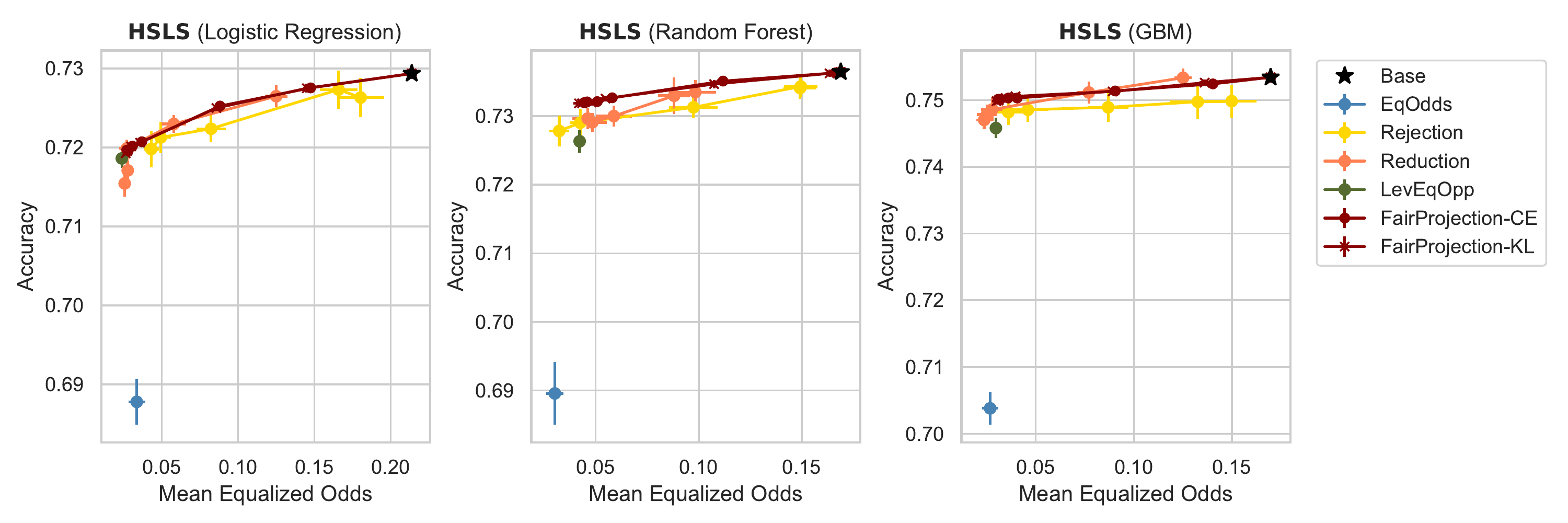}} \vspace{-.2in}
\caption{Accuracy-fairness curves of FairProjection and benchmark methods on the HSLS dataset with 3 different models (Logistic regression, Random forest, GBM). The fairness constraint is MEO.  \vspace{-.4in}}
\label{fig:hsls-eo-all}
\end{center}
\end{figure}

\begin{figure}[t!]
\begin{center}
\centerline{\includegraphics[width=\columnwidth]{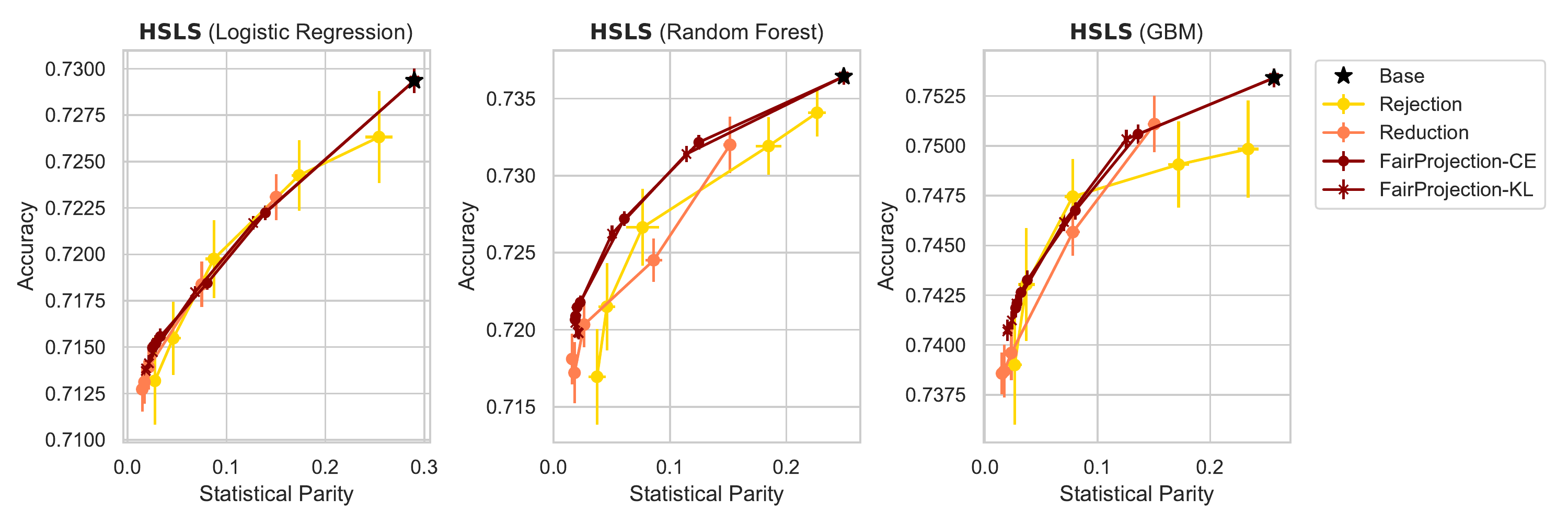}} \vspace{-.2in}
\caption{Accuracy-fairness curves of FairProjection and benchmark methods on the HSLS dataset with 3 different models (Logistic regression, Random forest, GBM). The fairness constraint is statistical parity.  \vspace{-.4in}}
\label{fig:hsls-sp-all}
\end{center}
\end{figure}

\begin{figure}[t!]
\begin{center}
\centerline{\includegraphics[width=\columnwidth]{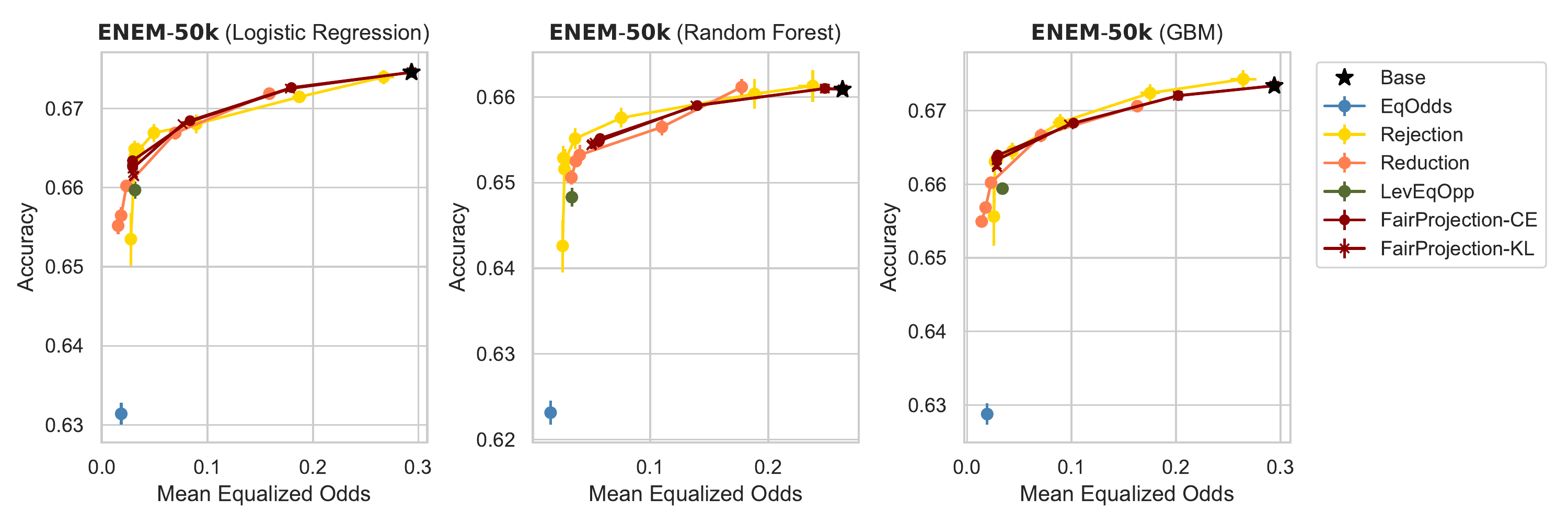}} \vspace{-.2in}
\caption{ Accuracy-fairness curves of FairProjection and benchmark methods on the ENEM-50k-2C dataset with 3 different models (Logistic regression, Random forest, GBM). The fairness constraint is MEO.}
\label{fig:enem-eo-all}
\end{center}
\end{figure}

\begin{figure}[t!]
\begin{center} 
\centerline{\includegraphics[width=\columnwidth]{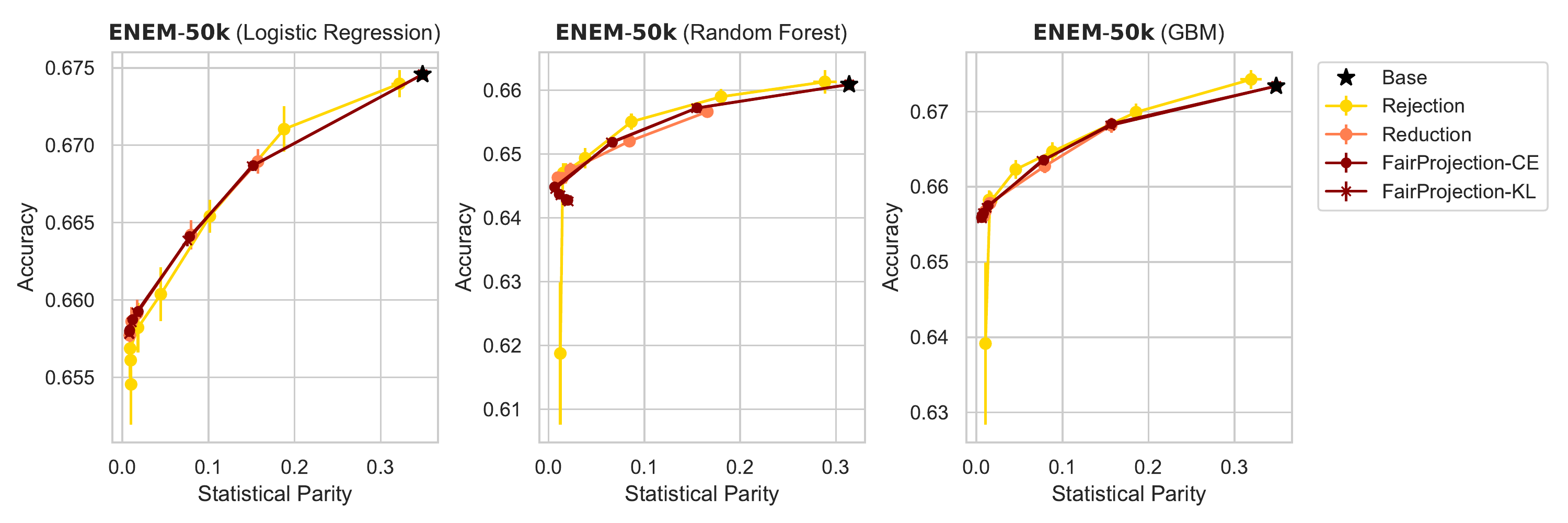}}\vspace{-.1in}
\caption{Accuracy-fairness curves of FairProjection and benchmark methods on the ENEM-50k-2C dataset with 3 different models (Logistic regression, Random forest, GBM). The fairness constraint is statistical parity. \vspace{-.4in}}
\label{fig:enem-sp-all}
\end{center}
\end{figure}

\begin{figure}[t!]
\begin{center}
\centerline{\includegraphics[width=\columnwidth]{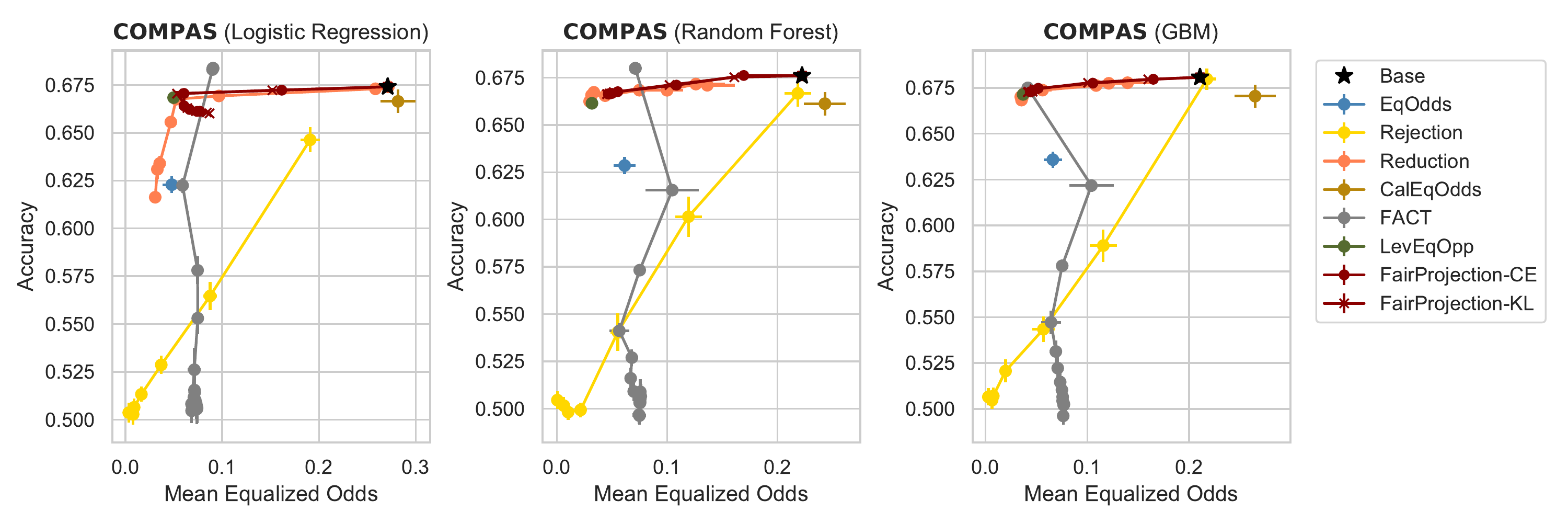}}\vspace{-.1in}
\caption{Accuracy-fairness curves of FairProjection and benchmark methods on COMPAS with 3 different models (Logistic regression, Random forest, GBM). The fairness constraint is MEO.\vspace{-.4in}}
\label{fig:compas-eo-all}
\end{center}
\end{figure}

\begin{figure}[t!]
\begin{center}
\centerline{\includegraphics[width=\columnwidth]{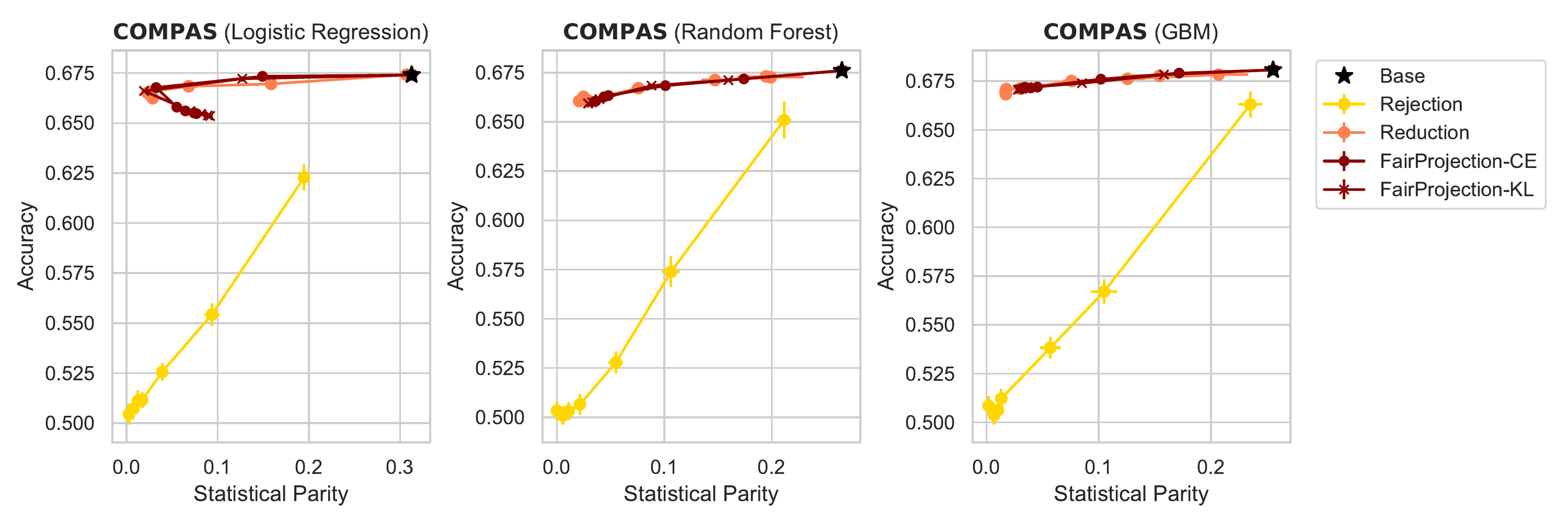}}\vspace{-.1in}
\caption{Accuracy-fairness curves of FairProjection and benchmark methods on COMPAS with 3 different models (Logistic regression, Random forest, GBM). The fairness constraint is statistical parity.}
\label{fig:compas-sp-all}
\end{center}
\end{figure}

\begin{figure}[t!]
\begin{center}
\centerline{\includegraphics[width=\columnwidth]{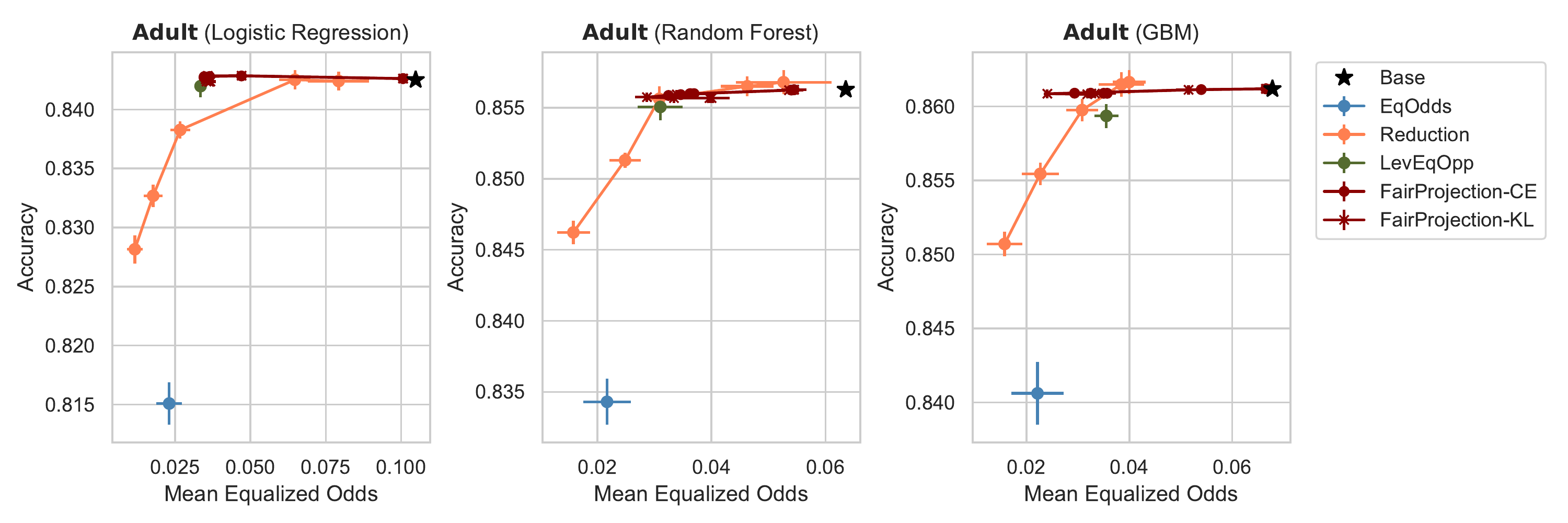}}
\caption{ Accuracy-fairness curves of FairProjection and benchmark methods on the Adult dataset with 3 different models (Logistic regression, Random forest, GBM). The fairness constraint is MEO.}
\label{fig:adult-eo-all}
\end{center}
\end{figure}

\begin{figure}[t!]
\begin{center}
\centerline{\includegraphics[width=\columnwidth]{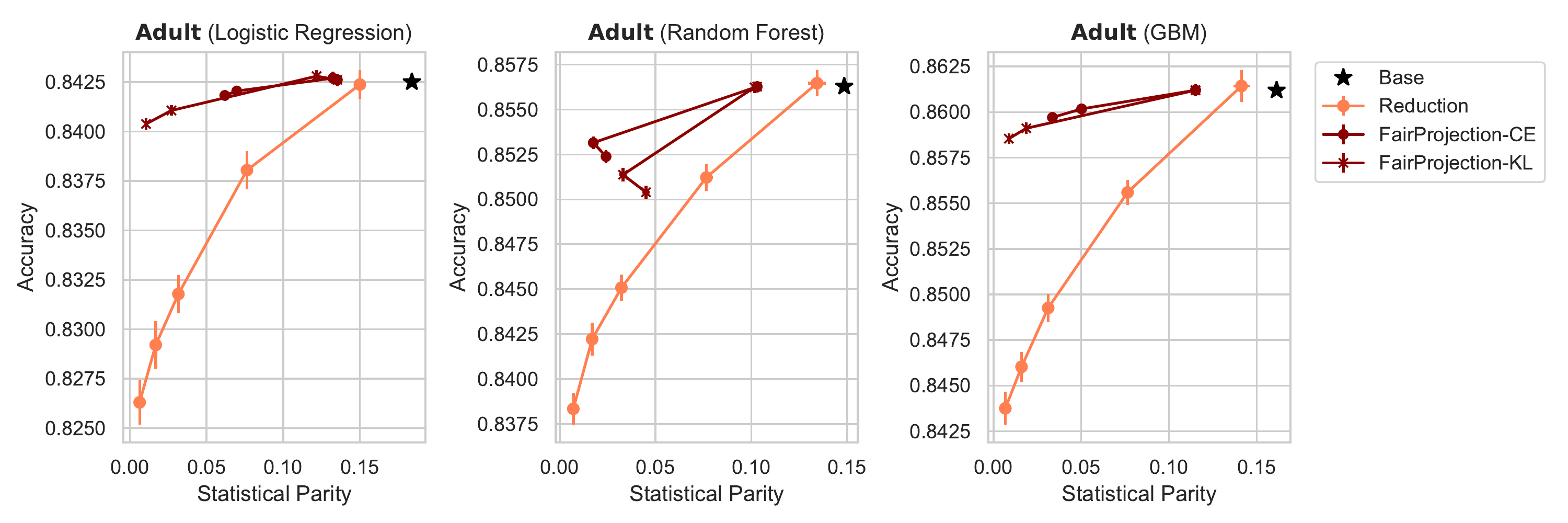}}
\caption{Accuracy-fairness curves of FairProjection and benchmark methods on the Adult dataset with 3 different models (Logistic regression, Random forest, GBM). The fairness constraint is statistical parity.}
\label{fig:adult-sp-all}
\end{center}
\end{figure}

\begin{figure}[t!]
\begin{center}
\centerline{\includegraphics[width=\columnwidth]{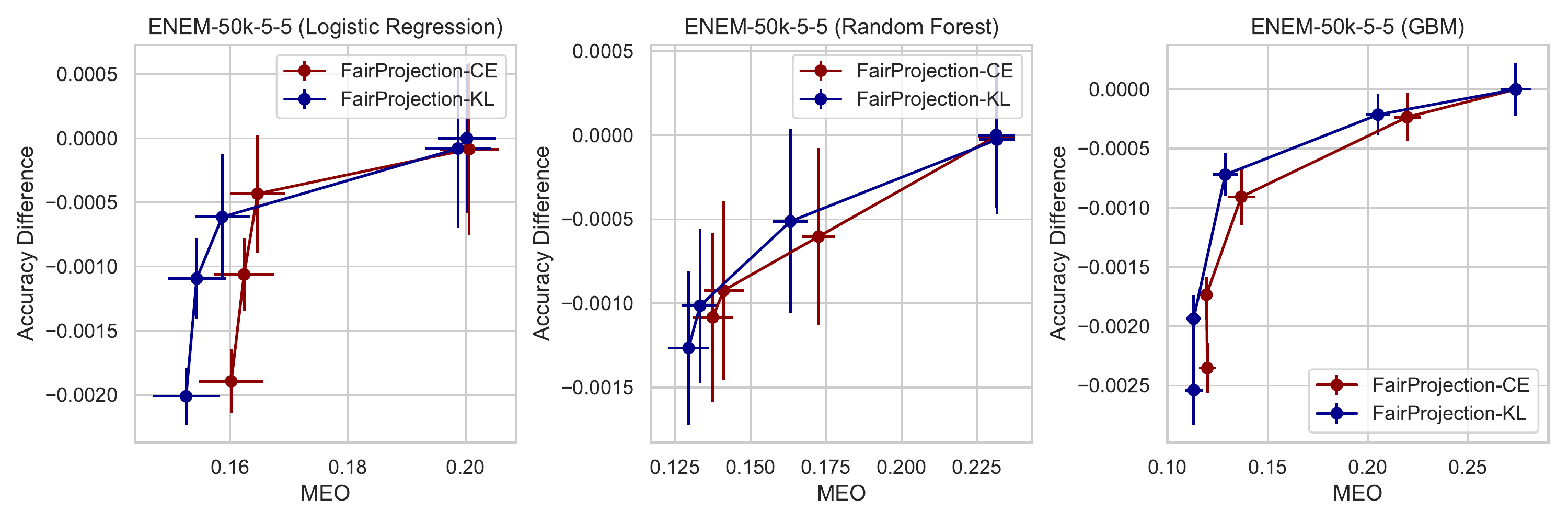}}
\caption{Accuracy-fairness curves of \texttt{FairProjection-CE} and \texttt{FairProjection-KL} on ENEM-50k with with 5 labels, 5 groups and different base classifiers base classifiers. The fairness constraint is MEO.}
\label{fig:enem-50k-5-5-meo}
\end{center}
\end{figure}

\begin{figure}[t!]
\begin{center}
\centerline{\includegraphics[width=\columnwidth]{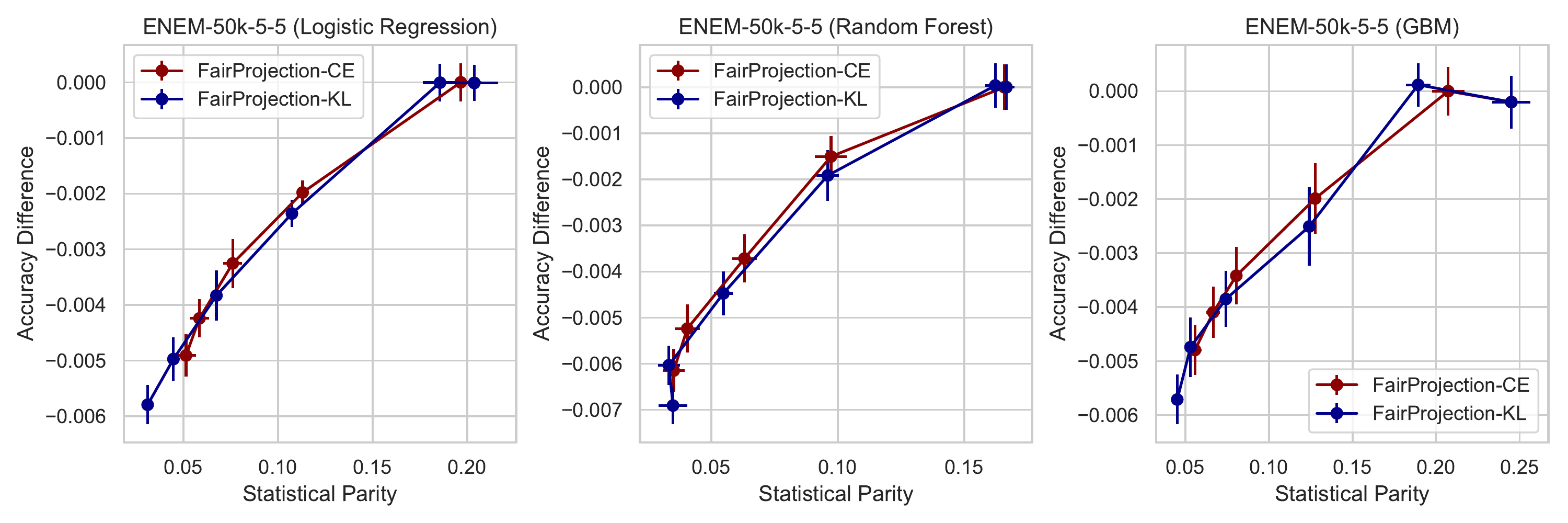}}
\caption{Accuracy-fairness curves of \texttt{FairProjection-CE} and \texttt{FairProjection-KL} on ENEM-50k with with 5 labels, 5 groups and different base classifiers base classifiers. The fairness constraint is SP.}
\label{fig:enem-50k-5-5-sp}
\end{center}
\end{figure}

\begin{figure}[t!]
\begin{center}
\centerline{\includegraphics[width=\columnwidth]{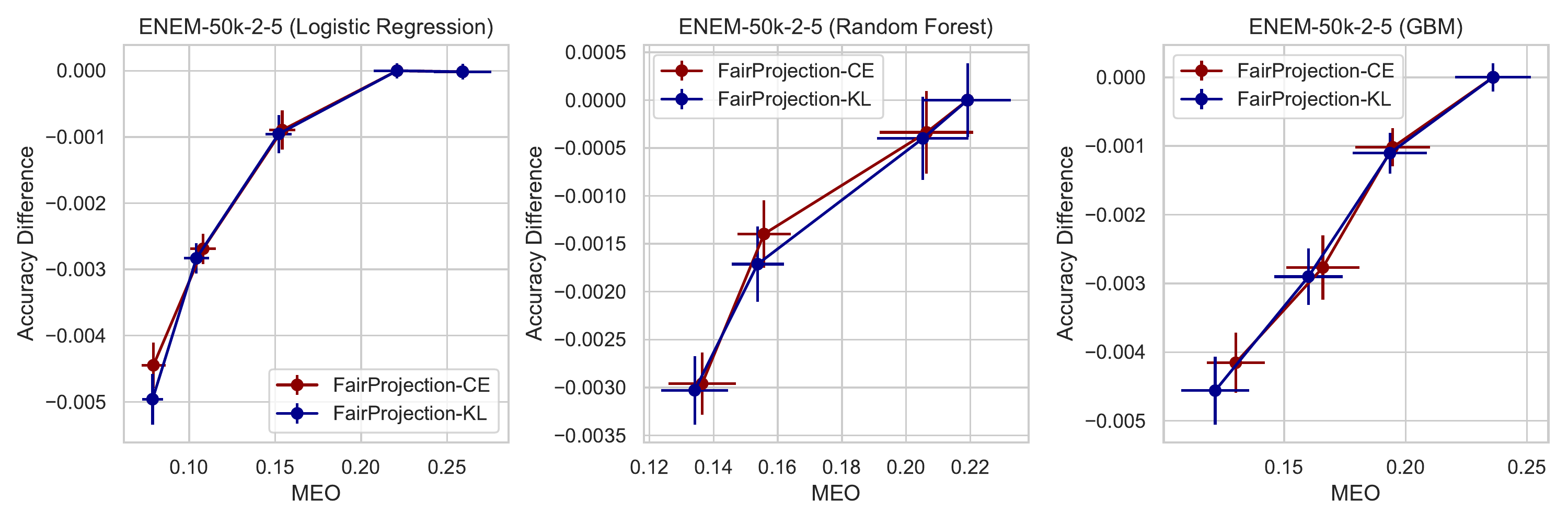}}
\caption{Accuracy-fairness curves of \texttt{FairProjection-CE} and \texttt{FairProjection-KL} on ENEM-50k with with 2 labels, 5 groups and different base classifiers base classifiers. The fairness constraint is MEO. \vspace{-.5in}}
\label{fig:enem-50k-2-5-meo}
\end{center}
\end{figure}

\begin{figure}[t!]
\begin{center}
\centerline{\includegraphics[width=\columnwidth]{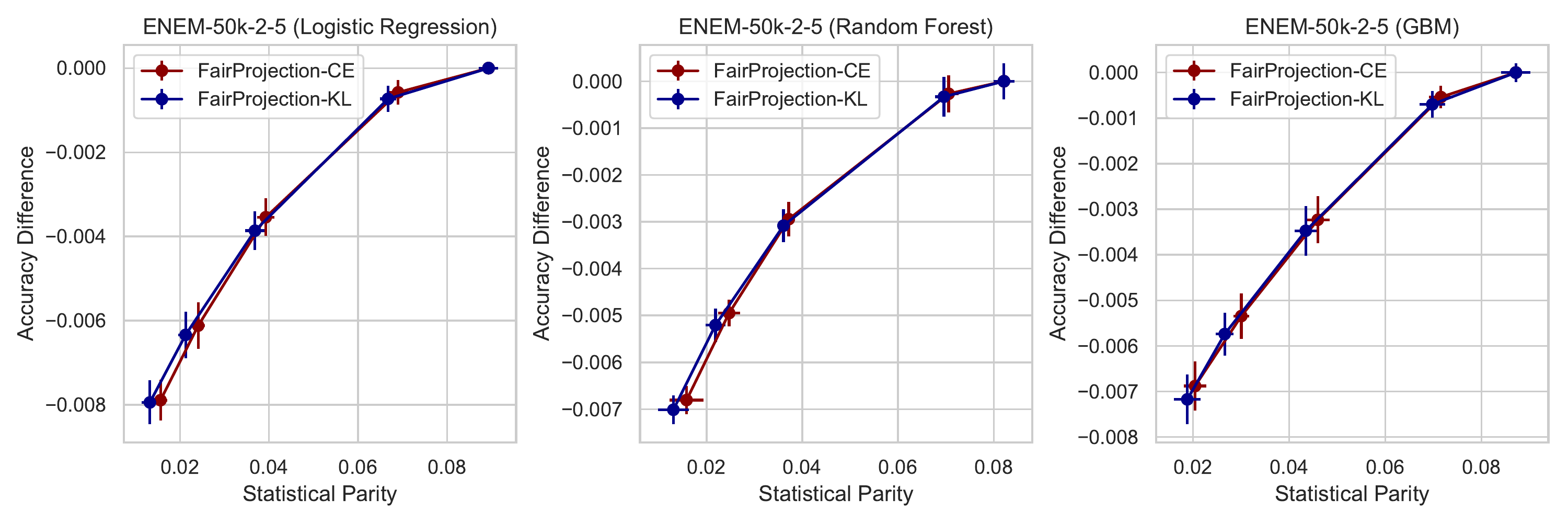}}
\caption{Accuracy-fairness curves of \texttt{FairProjection-CE} and \texttt{FairProjection-KL} on ENEM-50k with with 2 labels, 5 groups and different base classifiers base classifiers. The fairness constraint is SP.  \vspace{-.5in}}
\label{fig:enem-50k-2-5-sp}
\end{center}
\end{figure}

\begin{figure}[t!]
\begin{center}
\centerline{\includegraphics[width=0.4\columnwidth]{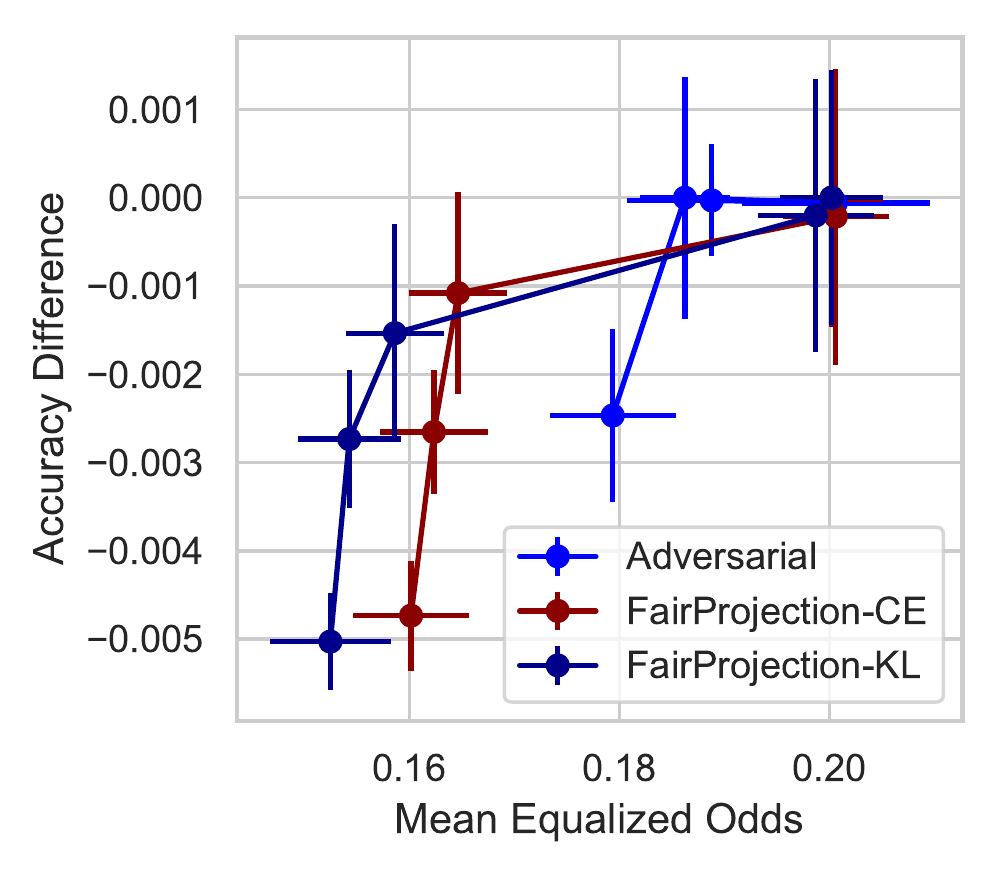}}
\caption{Comparison of \texttt{FairProjection-CE} and \texttt{FairProjection-KL} with \texttt{Adversarial} on ENEM-50k-5-2, meaning  5 labels, 2 groups. The reason for the difference comparing to Fig.~\ref{fig:multi} is that we resampled 50k data points from ENEM.}
\label{fig:enem-multi-adv}
\end{center}
\end{figure}

\clearpage 

\subsection{More on related work} \label{appendix:literature}

\begin{table*}[h]
\rowcolors{3}{MidnightBlue!10}{}
\centering
\resizebox{0.9\textwidth}{!}{\begin{tabular}{cccccccc}
\toprule
\multirow{2}{*}{\textbf{Method}} &  \multicolumn{6}{c}{\textbf{Feature}} \\
& Multiclass & Multigroup & Scores  &  Curve & Parallel  & Rate & Metric \\
\toprule
Reductions~\cite{agarwal2018reductions}        & \XSolidBold  & \CheckmarkBold  & \CheckmarkBold     & \CheckmarkBold  & \XSolidBold & \CheckmarkBold & SP, (M)EO   \\

Reject-option~\cite{kamiran2012reject}     & \XSolidBold  & \CheckmarkBold & \XSolidBold & \CheckmarkBold & \XSolidBold & \XSolidBold  & SP, (M)EO \\
EqOdds~\cite{hardt2016equality}             & \XSolidBold  & \CheckmarkBold  & \XSolidBold & \XSolidBold & \XSolidBold & \CheckmarkBold & EO  \\ 
LevEqOpp~\cite{chzhen2019leveraging}      & \XSolidBold  & \XSolidBold  & \XSolidBold    & \XSolidBold & \XSolidBold & \XSolidBold & FNR \\ 
CalEqOdds~\cite{pleiss2017fairness}      & \XSolidBold  & \XSolidBold  & \CheckmarkBold & \XSolidBold & \XSolidBold  & \CheckmarkBold   &  MEO  \\ 
FACT~\cite{kim2020fact}               &  \XSolidBold  &  \XSolidBold  & \XSolidBold   & \CheckmarkBold  &  \XSolidBold & \XSolidBold   &   SP, (M)EO \vspace{-0.5mm} \\ 
 Identifying~\cite{jiang2020identifying}       &  $\text{\CheckmarkBold}^{\text{\XSolidBold}}$ & \CheckmarkBold & \CheckmarkBold & \CheckmarkBold & \XSolidBold  & \XSolidBold & SP, (M)EO \vspace{0.8mm} \\

\FST~\cite{wei2019optimized,wei2021optimized}     & \XSolidBold   & \CheckmarkBold & \CheckmarkBold &\CheckmarkBold & \XSolidBold  &   \CheckmarkBold  & SP, (M)EO  \\ 

Overlapping~\cite{yang2020fairness}  & \CheckmarkBold  & \CheckmarkBold  & \CheckmarkBold & \CheckmarkBold &  \XSolidBold & \XSolidBold &  SP, (M)EO \\ 

Adversarial~\cite{zhang2018mitigating}  & \CheckmarkBold & \CheckmarkBold & N/A & \CheckmarkBold & \CheckmarkBold & \XSolidBold & SP, (M)EO\\ 
\midrule
\methodname~(ours)  & \CheckmarkBold  & \CheckmarkBold  & \CheckmarkBold  & \CheckmarkBold & \CheckmarkBold  & \CheckmarkBold & SP, (M)EO \\ 
\bottomrule
\end{tabular}}
\vspace{2mm}
\caption*{\textbf{Copy of Table~\ref{tab:benchmarks}.} Comparison between benchmark methods. \textbf{Multiclass/multigroup}: implementation takes datasets with multiclass/multigroup labels; \textbf{Scores}: processes raw outputs of probabilistic classifiers; \textbf{Curve}: outputs fairness-accuracy tradeoff curves (instead of a single point); \textbf{Parallel}: parallel implementation (e.g., on GPU) is available; \textbf{Rate}: convergence rate or sample complexity guarantee is proved.
 \textbf{Metric}: applicable fairness metric, with SP$\leftrightarrow$Statistical Parity, EO$\leftrightarrow$Equalized Odds, MEO$\leftrightarrow$Mean EO. 
}
\label{tab:benchmarks2}
\end{table*}

Our method is a model-agnostic post-processing method, so we focus our comparison on such post-processing fairness intervention methods. In the above table, the only exception is Adversarial~\cite{zhang2018mitigating}, which we use to benchmark multi-class prediction. Adversarial~\cite{zhang2018mitigating} is an in-processing method based on generative-adversarial network (GAN) where the adversary tries to guess the sensitive group attribute $S$ from $Y$ and $\widehat{Y}$. Even though this GAN-based approach is applicable to multi-class, multi-group prediction, it cannot be universally applied to any pre-trained classifier like our method.

EqOdds~\cite{hardt2016equality}, CalEqOdds~\cite{pleiss2017fairness} and LevEqOpp~\cite{chzhen2019leveraging} are post-processing methods designed for binary prediction with binary groups. 
They find different decision thresholds for each group that equalize FNR and FPR of two groups. CalEqOdds~\cite{pleiss2017fairness} has an additional constraint that the post-processed classifier must be well-calibrated, and we observe in our experiments that this stringent constraint leads to a low-accuracy classifier especially when there is a big gap in the base rate between the two groups. FACT~\cite{kim2020fact} follows a similar approach but generalizes this to an optimization framework that can have both equalized odds and statistical parity constraints and flexible accuracy-fairness trade-off. The optimization formulation finds a desired confusion matrix, and their proposed post-processing method flips the predictions to match the desired confusion matrix. Reject-option~\cite{kamiran2012reject} is similar in that it flips predictions near the decision threshold. In~\cite{kamiran2012reject}, instead of finding the optimal confusion matrix, it performs grid search to find the optimal margin around the decision threshold that  can minimize either equalized odds or statistical parity. For these methods that center around modifying decision thresholds, it is not straightforward to extend to multi-class and multi-group as one will have to consider
$\binom{|\calY|}{2}\cdot \binom{|\calS|}{2}$  boundaries.

FST~\cite{wei2019optimized,wei2021optimized} tackles fairness intervention via minimizing cross-entropy for binary classes. Their method is inherently tailored to binary classification \emph{and} only a cross-entropy objective function, and our \methodname-CE reduces to FST for the case of CE and binary classification tasks. Identifying~\cite{jiang2020identifying} is a method for minimizing KL-divergence for group-fairness intervention, which changes the label weights (via a convex combination) between unweighted and weighted samples, but it is not clear that this would navigate a good fairness-accuracy trade-off curve. Their method can be extended to non-binary prediction with non-binary groups by an appropriate choice of base classifier and fairness constraints, which is a non-trivial extension of the accompanying code, and we chose not to pursue this.
Note that \cite{jiang2020identifying} and \methodname~solve the KL-divergence minimization in very different ways. In particular, the runtime of \cite{wei2019optimized,wei2021optimized} on ~350k train is  longer than 30 minutes using logistic regression as a base classifier (in comparison, the runtime of \methodname~for a 500k dataset is less than 1 minute). This is because they require reweighing the data and retraining a large number of times. 
Hence, it is inherently non-parallelizable.

\subsubsection{Fairness in Multi-Class Prediction} 

Methods that are based on optimization with a fairness regularizer often can be easily extended to multi-class prediction as it only requires a small change in the regularizer. For example, instead of using $|\text{FNR}_0(x) - \text{FNR}_1(x)|$, one can replace this with 
\begin{equation} \label{eq:multi_loss}
\sum_{i \in \mathcal{Y}} \sum_{j \neq i \in \mathcal{Y}}|P(\widehat{Y} = j \mid Y =i, S =0) - P(\widehat{Y} = j \mid Y =i, S =1)|.
\end{equation}
FERM~\cite{donini2018empirical} mentions how their method can be extended to multi-class sensitive attribute. Similarly, we believe that their method can be used for multi-class labels as well. The reductions approach~\cite{agarwal2018reductions} assumes binary labels but is has natural extension to multi-class, which is explored in \cite{yang2020fairness}. In-processing methods proposed in \cite{cho2020fair} and \cite{zhang2018mitigating}  allow for both multi-class labels and multi-class group attributes. \cite{zhang2018mitigating} aims to achieve the independence between the sensitive attribute $S$ and $\widehat{Y}$ or  $\widehat{Y}$ given $Y$ by training an adversary who tries to figure out $\widehat{S}$. \cite{cho2020fair} directly estimates the fairness loss (e.g., \ref{eq:multi_loss}) using kernel density estimation. They also demonstrate the empirical performance in a three-class classification using synthetic data. Another in-processing method is \cite{aghaei2019learning} where the authors propose a way to incorporate multi-class fairness constraints into decision tree training. 
The preprocessing method suggested in \cite{celis2020data} is conceptually similar to our methods in that it aims to minimize the KL-divergence between the original distribution and preprcoessed distribution while satisfying fairness constraints. Their method, however, requires all feature vectors to be binary, and applies only to demographic parity or representation rate. 
There exist other notions of fairness, which is different from commonly-used group fairness metrics such as envy-freeness~\cite{balcan2019envy} or best-effort~\cite{krishnaswamy2021fair}, which can be applied to multi-class prediction tasks.

Finally, there are unpublished works~\cite{denis2021fairness,ye2020unbiased} that could handle multi-class classification. Specifically, \cite{denis2021fairness} presents a post-processing method that selects different thresholds for each group to achieve  demographic parity. \cite{ye2020unbiased} formulates SVM training as a mixed-integer program and integrates fairness regularizer in the objective, which can also deal with multi-class.

\section{Datasheet for ENEM 2020 dataset} \label{appendix:datasheet}

\subsubsection*{Questions}

The questions below are derived from \citep{gebru2021datasheets} and aim to provide context about the ENEM-2020 dataset. We highlight that we did not create the dataset nor collect the data included in it. Instead, we simply provide a link to the ENEM-2020 data at \citep{enem2017}. At the time of writing, the ENEM-2020 dataset is open and  made freely available by the Brazilian Government  at \citep{enem2017} under a Creative Commons Attribution-NoDerivs 3.0 Unported License \citep{CC-License}. We provide the datasheet below to clarify certain aspects of the dataset (e.g., motivation, composition, etc.) since the original information is available in Portuguese at \citep{enem2017}, thus limiting its  access to a broader audience. The website \citep{enem2017} contains a link to download a \texttt{.zip} file which contains the ENEM-2020 data in \texttt{.csv} format and extensive accompanying documentation. 

The datasheet below is \textbf{not} a substitute for the explanatory files that are downloaded together with the dataset at \citep{enem2017}, and we emphatically recommend the user to familiarize themselves with associated documentation prior to usage. We also strongly recommend the user to carefully read the ``Leia-Me'' (readme) file \texttt{Leia\_Me\_Enem\_2020.pdf} available in the same \texttt{.zip} folder that contains the dataset. The answers in the datasheet below are based on an English translation of information available at \citep{enem2017} and may be incomplete or inaccurate. The datasheet below is based on our own independent analysis and in no way represents or attempts to represent the opinion or official position of the Brazilian Government and its agencies. 

We also note that we do not distribute the ENEM-2020 dataset directly nor host the dataset ourselves. Instead, we provide a link to download the data from a public website hosted by the Brazilian Government. The dataset may become unavailable in case the link in \citep{enem2017} becomes inaccessible. 

\subsubsection*{Motivation}

\begin{itemize}
    \item \textbf{For what purpose was the dataset created?} According to the ``Leia-me'' (Read Me) file that accompanies the data, the dataset was made available to fufill the mission of the Instituto Nacional de Estudos e Pesquisas Educacionais An\'isio Teixeira (INEP) of developing and disseminating data about exams and evaluations of basic education in Brazil. 
    \item \textbf{Who created the dataset (e.g., which team, research group) and on behalf of which entity (e.g., company, institution, organization)?} The dataset was developed by INEP, which is a government agency connected to the Brazilian Ministry of Education.
    \item \textbf{Who funded the creation of the dataset?} The data is made freely available by the Brazilian Government.
\end{itemize}

\subsubsection*{Composition}
\begin{itemize}
    \item \textbf{What do the instances that comprise the dataset represent (e.g., documents, photos, people, countries)?} The instances of the dataset are information about individual students who took the Exame Nacional do Ensino M\'edio (ENEM). The ENEM is the capstone exam for Brazilian students who are graduating or have graduated high school. 
    \item \textbf{How many instances are there in total (of each type, if appropriate)?} The raw data provided in at \citep{enem2017} has approximately 5.78 million entries. The processed version we use in our experiments has approximately 1.4 million entries.
    
    \item \textbf{Does the dataset contain all possible instances or is it a sample (not necessarily random) of instances from a larger set?} The data provided is the lowest level of aggregation of data collected  from ENEM exam-takers made available by INEP.
    
    \item \textbf{What data does each instance consist of?} We provide a brief description of the features available in the raw public data provided at \citep{enem2017}. Upon downloading the data, a detailed description of features and their values are available (in Portuguese) in the file titled  \texttt{Dicion\'ario\_Mircrodados\_ENEM\_2020.xsls}. The features include:
    \begin{itemize}
        \item \textbf{Information about exam taker:} exam registration number (masked), year the exam was taken (2020), age range, sex, marriage status, race, nationality, status of high school graduation, year of high school graduation, type of high school (public, private, n/a), if they are a ``treineiro'' (i.e., taking the exam as practice).
        \item \textbf{School data}: city and state of participant's school, school administration type (private, city, state, or federal), location (urban or rural), and school operation status.
         \item \textbf{Location where exam was taken}: city and state.
         \item \textbf{Data on multiple-choice questions}: The exam is divided in 4 parts (translated from Portuguese): natural sciences, human sciences, languages and codes, and mathematics. For each part there is data if the participant attended the corresponding portion of the exam, the type of exam book they received, their overall grade, answers to exam questions, and the answer sheet for the exam. 
         \item \textbf{Data on essay question}: if participant took the exam, grade on different evaluation criteria, and overall grade.
         \item \textbf{Data on socio-economic questionnaire answers:} the data include answers to 25 socio-economic questions (e.g., number of people who live in your house, family average income, if the your house has a bathroom, etc.).
         
    \end{itemize}
    
    \item \textbf{Is there a label or target associated with each instance?} No, there is no explicit label. In our fairness benchmarks, we use grades in various components of the exam as a predicted label.
    
    \item \textbf{Is any information missing from individual instances?} Yes, certain instances have missing values.
    
    \item \textbf{Are relationships between individual instances made explicit (e.g., users' movie ratings, social network links)?} No explicit relationships identified.
    
    \item \textbf{Are there recommended data splits (e.g., training, development/validation, testing)?} No.
    
    \item \textbf{Are there any errors, sources of noise, or redundancies in the dataset?} The data contains missing values and, according to INEP, was collected from individual exam takers. The information is self-reported and collected at the time of the exam.
    
    \item \textbf{Is the dataset self-contained, or does it link to or otherwise rely on external resources (e.g., websites, tweets, other datasets)?} Self-contained.
    
    \item \textbf{Does the dataset contain data that might be considered confidential (e.g., data that is protected by legal privilege or by doctor–patient confidentiality, data that includes the content of individuals' non-public communications)?} According to the \emph{Leia-me} (readme) file (in Portuguese) that accompanies the dataset and our own inspection, the dataset does not contain any feature that allows direct identification of exam takers such as name, email, ID number, birth date, address, etc. The exam registration number has been substituted by a sequentially generated mask. INEP states that the released data is aligned with the Brazilian \emph{Lei Geral de Proteç\~ao dos Dados} (LGPD, General Law for Data Protection). We emphatically recommend the user to view the Readme file prior to usage.
    
    \item \textbf{Does the dataset contain data that, if viewed directly, might be offensive, insulting, threatening, or might otherwise cause anxiety?}  The official terminology used by the Brazilian Government to denote race can be viewed as offensive. Specifically, the term used to describe the race of exam takers of Asian heritage is ``Amarela,'' which is the Portuguese word for the color yellow. Moreover, the term ``Pardo,'' which roughly translates to brown, is used to denote individuals of multiple or mixed ethnicity. This outdated and inappropriate terminology is still in official use by the Brazilian Government, including in its population census. The dataset itself includes integers to denote race, which are mapped to specific categories through the variable dictionary. 
    
    \item \textbf{Does the dataset relate to people?} Yes.
    
    \item \textbf{Does the dataset identify any subpopulations (e.g., by age, gender)?} Yes. Information about age, sex, and race are included in the dataset.
    
    \item \textbf{Is it possible to identify individuals (i.e., one or more natural persons), either directly or indirectly (i.e., in combination with other data) from the dataset?} The \emph{Leia-me} (readme) file notes that the individual exam-takers cannot be directly identified from the data. However, in the same file, INEP recognizes that the Brazilian data protection law (LGPD) does not clearly define what constitutes a reasonable effort of de-identification. Thus, INEP adopted a cautious approach: this  dataset is a simplified/abbreviated version of the ENEM micro-data compared to prior releases and aims to remove any features that may allow identification of the exam-taker. 
    
    \item \textbf{Does the dataset contain data that might be considered sensitive in any way (e.g., data that reveals racial or ethnic origins, sexual orientations, religious beliefs, political opinions or union memberships, or locations; financial or health data; biometric or genetic data; forms of government identification, such as social security numbers; criminal history)?} The data includes race information and socio-economic questionnaire answers.
    
\end{itemize}

\subsubsection*{Collection Process}

Since we did not produce the data, we cannot speak directly about the collection process. Our understanding is that the data contains self-reported answers from exam-takers of the ENEM collected at the time of the exam. The exam was applied on 17 and 24 of January 2021 (delayed due to COVID). The data was aggregated and made publicly available by INEP at \citep{enem2017}. After consulting the IRB office at our institution, no specific IRB was required to use this data since it is anonymized and publicly available.

\subsubsection*{Preprocessing/cleaning/labeling}
\begin{itemize}
    \item \textbf{Was any preprocessing/cleaning/labeling of the data done (e.g., discretization or bucketing, tokenization, part-of-speech tagging, SIFT feature extraction, removal of instances, processing of missing values)?} Some mild pre-processing was done on the data to ensure anonymity, as indicated in the ``Leia-me'' file. This includes aggregating participant ages, masking exam registration numbers, and removing additional information that could allow de-anonymization. 
    
    \item \textbf{Was the ``raw'' data saved in addition to the preprocessed/cleaned/labeled data (e.g., to support unanticipated future uses)?} The raw data is not publicly available.
    
\end{itemize}

\subsection*{Uses}
\begin{itemize}
    \item \textbf{Has the dataset been used for any tasks already?}  We have used this dataset to benchmark fairness interventions in ML in the present paper. ENEM microdata has also been widely used in studies ranging from public policy in Brazil to item response theory in high school exams.
    
    \item \textbf{Are there tasks for which the dataset should not be used?} INEP does not clearly define tasks that should not be used on this dataset. However, no attempt should be made to de-anonymize the data.
    
\end{itemize}

\subsubsection*{Distribution and Maintenance}
The ENEM-2020 dataset is open and  made freely available by the Brazilian Government  at \citep{enem2017} under a Creative Commons Attribution-NoDerivs 3.0 Unported License \citep{CC-License} at the time of writing. The dataset may become unavailable in case the link in \citep{enem2017} becomes inaccessible.

\end{document}